\documentclass{article}

\usepackage[utf8]{inputenc}
\usepackage[T1]{fontenc}
\usepackage{graphicx}
\usepackage{amsmath,amsfonts,amssymb,amsthm}
\usepackage{mathtools}
\usepackage{mathrsfs}
\usepackage{booktabs}
\usepackage{bbold}
\usepackage{bm}
\usepackage[font=footnotesize]{caption}
\usepackage{subcaption}
\usepackage{placeins}
\usepackage{enumitem}
\usepackage{empheq}
\usepackage{framed}
\usepackage{comment}
\usepackage{soul}
\usepackage{dirtytalk}
\usepackage{dsfont}
\usepackage{bbm}
\usepackage{natbib}
\usepackage{fullpage}
\usepackage[dvipsnames]{xcolor}
\usepackage{xspace}
\usepackage{epsf}
\usepackage{graphics}
\usepackage{psfrag}
\usepackage{microtype}
\usepackage{algorithm}
\usepackage{algpseudocode}
\usepackage{xurl}
\usepackage[mathscr]{eucal}
\usepackage{scalerel}
\usepackage{mathbbol}
\usepackage[framemethod=TikZ]{mdframed}
\usepackage[colorlinks,linkcolor=magenta,citecolor=blue,urlcolor=blue,pagebackref=false,backref=false]{hyperref}
\usepackage{cleveref}

\newtheoremstyle{named}{}{}{\itshape}{}{\bfseries}{.}{.5em}{\thmnote{#3's }#1}
\theoremstyle{named}

\theoremstyle{plain}

\newtheorem{theorem}{Theorem}
\newtheorem{proposition}{Proposition}

\newtheorem{lemma}{Lemma}

\newtheorem{corollary}{Corollary}

\newlength{\widebarargwidth}
\newlength{\widebarargheight}
\newlength{\widebarargdepth}

\makeatletter
\long\def\@makecaption#1#2{
        \vskip 0.8ex
        \setbox\@tempboxa\hbox{\small {\bf #1:} #2}
        \parindent 1.5em  %% How can we use the global value of this???
        \dimen0=\hsize
        \advance\dimen0 by -3em
        \ifdim \wd\@tempboxa >\dimen0
                \hbox to \hsize{
                        \parindent 0em
                        \hfil
                        \parbox{\dimen0}{\def\baselinestretch{0.96}\small
                                {\bf #1.} #2
                                }
                        \hfil}
        \else \hbox to \hsize{\hfil \box\@tempboxa \hfil}
        \fi
        }
\makeatother

\long\def\comment#1{}
\definecolor{battleshipgrey}{rgb}{0.52, 0.52, 0.51}
\definecolor{darkgray}{rgb}{0.66, 0.66, 0.66}
\definecolor{darkgreen}{rgb}{0.0, 0.2, 0.13}
\definecolor{darkspringgreen}{rgb}{0.09, 0.45, 0.27}
\definecolor{dukeblue}{rgb}{0.0, 0.0, 0.61}
\definecolor{olivedrab7}{rgb}{0.24, 0.2, 0.12}
\definecolor{darkblue}{rgb}{0.0, 0.0, 0.55}
\definecolor{darkscarlet}{rgb}{0.34, 0.01, 0.1}
\definecolor{candyapplered}{rgb}{1.0, 0.03, 0.0}
\definecolor{ao(english)}{rgb}{0.0, 0.5, 0.0}
\definecolor{applegreen}{rgb}{0.55, 0.71, 0.0}

\definecolor{myred}{HTML}{880000}
\definecolor{mygreen}{HTML}{008800}
\definecolor{myblue}{HTML}{000088}
\definecolor{linkblue}{HTML}{0000BB}
\definecolor{myred}{HTML}{880000}
\definecolor{mygreen}{HTML}{008800}
\definecolor{myblue}{HTML}{000088}
\definecolor{linkblue}{HTML}{0000BB}

\renewcommand{\leq}{\leqslant}

\renewcommand{\le}{\leqslant}
\renewcommand{\ge}{\geqslant}

\newcommand{\sigmamax}{\sigma_{\max}}

\newtheorem{assumption}{Assumption}
\theoremstyle{remark}
\newtheorem{remark}{Remark}

\DeclareFontFamily{U}{mathx}{}
\DeclareFontShape{U}{mathx}{m}{n}{<-> mathx10}{}
\DeclareSymbolFont{mathx}{U}{mathx}{m}{n}

\DeclareMathAccent{\widecheck}{0}{mathx}{"71}

\long\def\comment#1{}

\newcommand{\rewardmodel}{R}
\newcommand{\refpolicy}{\pi_{\mathrm{ref}}}

\newcommand{\curpolicy}{\pi_{\theta}}
\newcommand{\rlhfobj}{J_{\mathrm{RN}}}
\newcommand{\TTSobj}{J_{\mathrm{TT}}}
\newcommand{\hatTTSobj}{\widehat{J}_{\mathrm{TT}}}

\newcommand{\Ind}[1]{\mathbb{I}\left\{ #1 \right\}} % Indicator function

\newcommand{\modelLlamaOneB}{\texttt{meta-\allowbreak llama/\allowbreak Llama-\allowbreak 3.2-\allowbreak 1B-\allowbreak Instruct}}
\newcommand{\modelLlamaEightBMeta}{\texttt{meta-\allowbreak llama/\allowbreak Meta-\allowbreak Llama-\allowbreak 3.1-\allowbreak 8B-\allowbreak Instruct}}
\newcommand{\modelLlamaEightB}{\texttt{meta-\allowbreak llama/\allowbreak Llama-\allowbreak 3.1-\allowbreak 8B-\allowbreak Instruct}}
\newcommand{\modelQwenFourB}{\texttt{Qwen/\allowbreak Qwen3-\allowbreak 4B-\allowbreak Instruct-\allowbreak 2507}}
\newcommand{\rewardSkyworkLlama}{\texttt{Skywork/\allowbreak Skywork-\allowbreak Reward-\allowbreak V2-\allowbreak Llama-\allowbreak 3.1-\allowbreak 8B}}
\newcommand{\rewardSkyworkQwen}{\texttt{Skywork/\allowbreak Skywork-\allowbreak Reward-\allowbreak V2-\allowbreak Qwen3-\allowbreak 8B}}
\newcommand{\datasetHH}{\texttt{Anthropic/\allowbreak hh-rlhf}}

\begin{document}
\begin{center}
{\bf{\LARGE{What should post-training optimize? A test-time scaling law perspective}}}

\vspace*{.2in}
{\large{
 \begin{tabular}{ccc}
  Muheng Li$^{ \dagger}$ &  Jian Qian$^{\diamond}$ &
  Wenlong Mou$^{\dagger}$ 
 \end{tabular}

}

\vspace*{.2in}

 \begin{tabular}{c}
 Department of Statistical Sciences, University of Toronto$^{\dagger}$\\
 Department of AI and Data Science, The University of Hong Kong $^{\diamond}$
 \end{tabular}

}

\end{center}

\begin{abstract}
Large language models are increasingly deployed with test-time strategies: sample \(N\) responses, score them with a reward model or verifier, and return the best. This deployment rule exposes a mismatch in post-training: standard objectives optimize the mean reward of a single response, whereas best-of-\(N\) performance is governed by the upper tail of the reward distribution. Recent test-time-aware objectives partly address this mismatch, but typically assume that training can use the same per-prompt rollout budget as deployment, which is impractical when post-training must cover many prompts while deployment can allocate much larger per-prompt test-time compute. We study this budget-mismatch regime, where only \(m\ll N\) per-prompt rollouts are available during training but the target objective is best-of-\(N\) deployment. Under structural assumptions on the reward tails, we show that the policy gradient of the best-of-\(N\) objective can be approximated from a much smaller rollout group by extrapolating upper-tail statistics. This yields a family of Tail-Extrapolated estimators for best-of-\(N\)-oriented post-training: a simple direct estimator, Tail-Extrapolated Advantage (TEA), and a fixed-order debiased Prefix-TEA estimator based on moment cancellation. Experiments on instruction-following tasks show that TEA and Prefix-TEA improve best-of-\(N\) performance across different language models, reward models and datasets under various training and test-time budget settings.
\end{abstract}

\section{Introduction}
The modern workflow for large language models (LLM) typically consists of three stages: pre-training, post-training, and test-time adaptation. Post-training is a critical stage in the development of large language models (LLMs), where a pre-trained model is further fine-tuned using reinforcement learning to optimize a task-specific objective or align with human preferences. At deployment time, practitioners often use additional test-time compute to improve performance. A widely used strategy is best-of-$N$ sampling: for each prompt, one draws $N$ independent responses from the model, scores them with a reward model, and returns the one with the largest score. This test-time strategy can substantially improve performance, but it also creates a mismatch with the standard post-training objective, which optimizes the expected reward of a single response rather than the best response among $N$ samples. This mismatch is not only conceptually unsatisfying, but also practically sub-optimal. Empirically, reward-guided fine-tuning has been shown to reduce output diversity across different tasks, which removes precisely the stochasticity that the best-of-\(N\) (BoN) sampling strategy relies on~\citep{kirk2023understanding,omahony2024attributing}. Direct test-time-aware training studies give complementary evidence: BoN-aware, pass@$k$/majority-vote-aware, and max@$k$ objectives improve best-of-$N$ or set-level performance on various benchmarks over standard single-sample objectives~\citep{chow2024inference,tang2025optimizing,bagirov2025best,ousherovitch2026compute}.

These test-time-aware post-training methods typically involve computing the test-time objective directly during training, which requires a large rollout budget per prompt. For example, if the test-time strategy is best-of-$N$ sampling, then for each prompt during training, one needs to generate at least $N$ responses to compute the objective and the policy gradient. This matched-budget requirement creates tension between post-training and test-time deployment.
On the one hand, post-training must cover a large number of prompts, so the
rollout budget available for each individual prompt is often limited \citep{zhang2025grpo,sun2025improving}.
On the other hand, test-time compute is increasingly treated as a separate
scaling dimension, and recent work has begun to study deployment regimes with
substantially larger search budgets \citep{snell2024scaling, chow2024inference, yuksekgonul2026learning, novikov2025alphaevolve}. This motivates the key question:
\begin{quote}
\textit{Can we optimize best-of-$N$ performance when only $m \ll N$ rollouts per prompt are available during training?}
\end{quote}
In this paper, we provide a solution to this question through the lens of \emph{test-time scaling laws}, which characterize how best-of-$N$ performance scales with $N$ under a given policy. Recently, Li et al.,~\cite{li2026predicting} characterize the BoN scaling law by modeling the upper tail of the reward distribution, and show that the BoN scaling law can be predicted using $m \ll N$ samples by extrapolating upper-tail statistics. Drawing on this perspective, we develop a family of algorithms that use tail extrapolation to optimize best-of-$N$ performance under a small rollout budget. Our main contributions are as follows:
\begin{itemize}
\item We derive a class of Tail-Extrapolated Advantage (TEA) estimators for the advantage function and policy gradient of the best-of-$N$ objective, which can be estimated from $m \ll N$ rollouts per prompt. We further construct a debiased estimator, Prefix-TEA, which further removes the bias in stochastic gradient estimation.
\item We prove non-asymptotic gradient estimation error bounds for TEA and Prefix-TEA, characterizing the bias--variance tradeoff and the dependence on the training budget $m$ and the test-time budget $N$. Building upon these guarantees, we establish convergence guarantees for post-training with TEA and Prefix-TEA, with respect to the best-of-$N$ objective.
\item We conduct extensive experiments on instruction-following tasks, showing that TEA and Prefix-TEA consistently improve best-of-$N$ performance over GRPO and recent test-time-aware baselines across datasets and reward models under various training and test-time budget settings.
\end{itemize}

We provide a detailed discussion of related work in \Cref{app:additional_related_work}.

\section{Problem setup and test-time-aware objective}

\subsection{Text generation and reward distribution}
% \mwlcomment{We don't even need to describe auto-regressive generation -- our framework has nothing to do with autoregression}
Let $x \sim P_X$ be a prompt. We model the language model as a stochastic
policy $\curpolicy$, which induces a distribution over responses
\[
y \sim \curpolicy(\cdot \mid x).
\]
We assume a fixed reward model $\rewardmodel(x,y)\in\mathbb{R}$ that assigns
a scalar score to each prompt-response pair $(x,y)$. For a given prompt $x$ and policy parameter $\theta$, the induced reward is
the random variable
\[
R_\theta(x)
:=
\rewardmodel(x,y),
\qquad
y \sim \curpolicy(\cdot \mid x).
\]
For a fixed prompt $x$, let
\begin{equation}
\label{eq:reward_distribution}
F_{\theta,x}(t)
:=
\mathbb{P}\bigl(R_\theta(x)\le t \mid x\bigr)
\end{equation}
denote the conditional distribution function of $R_\theta(x)$ given $x$.
We refer to $F_{\theta,x}$ as the reward distribution induced by prompt $x$
and policy $\curpolicy$.

\subsection{From standard post-training to test-time-aware training}

In standard RLHF-style post-training, the policy is optimized to improve the
expected reward while staying close to a reference policy. Concretely, the
training objective takes the form
\begin{equation}
\label{eq:risk_neutral_objective}
\rlhfobj(\theta)
:=
\mathbb{E}_{x \sim P_X}
\Bigl[
\mathbb{E}_{y \sim \curpolicy(\cdot \mid x)}
\bigl[\rewardmodel(x,y)\bigr]
-
\beta\,\mathrm{KL}\bigl(
\curpolicy(\cdot \mid x)\,\|\,\refpolicy(\cdot \mid x)
\bigr)
\Bigr].
\end{equation}
This objective is driven by the \emph{mean} reward under the current policy.

At deployment time, however, one often uses additional test-time compute,
for example through best-of-$N$ sampling: for each prompt $x$, one draws
$N$ independent responses from $\curpolicy(\cdot \mid x)$, scores them with
the reward model, and returns the one with the largest score. This induces
the promptwise value
\begin{equation}
\label{eq:promptwise_bon_value}
V_N(\theta;x)
:=
\mathbb{E}
\Bigl[
\max_{1 \le i \le N} \rewardmodel(x,y_i)
\,\Bigm|\, x,\;
y_1,\dots,y_N \stackrel{\mathrm{i.i.d.}}{\sim} \curpolicy(\cdot \mid x)
\Bigr].
\end{equation}
Therefore, if the goal of post-training is to improve performance under some additional test-time budget, a natural objective is
\begin{equation}
\label{eq:exact_tts_objective}
\TTSobj(\theta)
:=
\mathbb{E}_{x \sim P_X}
\Bigl[
V_N(\theta;x)
-
\beta\,\mathrm{KL}\bigl(
\curpolicy(\cdot \mid x)\,\|\,\refpolicy(\cdot \mid x)
\bigr)
\Bigr].
\end{equation}

The difference between \eqref{eq:risk_neutral_objective} and
\eqref{eq:exact_tts_objective} is the part of the reward distribution being
optimized. Mean-reward post-training improves the average sample, whereas
best-of-\(N\) deployment is controlled by the upper tail. Directly optimizing
\eqref{eq:exact_tts_objective} with a matched budget would require rollout
groups comparable to the deployment budget \(N\). In modern post-training, this
is often impractical: the training rollout budget per prompt is limited, while
deployment may use a much larger search budget. We therefore focus on the
finite-budget regime
\[
m\ll N,
\]
and ask how to construct a test-time-aware training signal from only \(m\)
rollouts per prompt.

% Recent work has begun to optimize post-training for the eventual
% inference-time procedure, bringing objectives of the form
% \eqref{eq:exact_tts_objective} into focus. For example,
% \citet{chow2024inference} study Best-of-$N$-aware fine-tuning, while
% \citet{tang2025optimizing} optimize fixed-$k$ inference-time objectives such
% as pass@$k$ and majority voting. However, these works focus on a \emph{matched-budget} regime, where the same
% sample budget is used in both training and deployment. In particular, if the
% test-time procedure uses $N$ samples per prompt, then training also assumes
% access to $N$ samples per prompt.

% \mwlcomment{Some of the texts here already appeared in intro/related work so we don't need to repeat.}

% In contrast, this matched-budget setting can be impractical in modern LLM post-training.
% Training must cover a large number of prompts, so the
% rollout budget available for each individual prompt is often limited \citep{zhang2025grpo,sun2025improving}.
% At the same time, test-time compute is increasingly treated as a separate
% scaling dimension, and recent work has begun to study deployment regimes with
% substantially larger search budgets \citep{snell2024scaling, chow2024inference, yuksekgonul2026learning, novikov2025alphaevolve}.
% This creates a practically important setting in which the deployment budget
% $N$ can be much larger than the training-time budget $m$.

% Our goal is therefore to optimize \eqref{eq:exact_tts_objective} when only
% $m \ll N$ rollouts per prompt are available during training.

\subsection{A tail-based surrogate objective}

To make the deployment-aware objective \eqref{eq:exact_tts_objective}
tractable under a small rollout budget, we impose a structural assumption on
the upper tail of the promptwise reward distribution.

\begin{assumption}
\label{assum:gaussian-tail-model}
Fix $\alpha\in(0,1/2)$. Let $F_{\theta,x}$ be the reward distribution defined in
\eqref{eq:reward_distribution}, and let $p_{\theta,x}$ denote its density.
Define
\(
r_{\theta,2\alpha}(x):=F_{\theta,x}^{-1}(1-2\alpha).
\)
We assume that, for every prompt $x$ and parameter $\theta$, the upper
$2\alpha$ tail of the reward distribution is Gaussian, namely
\begin{equation}
p_{\theta,x}(r)=\phi\bigl(r;\mu_\theta(x),\sigma_\theta^2(x)\bigr),
\qquad r\ge r_{\theta,2\alpha}(x),
\label{eq:gaussian_tail_density}
\end{equation}
where $\phi(\cdot;\mu,\sigma^2)$ denotes the Gaussian density with mean $\mu$
and variance $\sigma^2$. In addition, there exists $M_R>0$ such that
\(
\mathbb{E}_{y\sim\curpolicy(\cdot\mid x)}\big[\rewardmodel(x,y)^4\big]\le M_R
\)
for all prompts $x$ and all policy parameters $\theta$.
\end{assumption}

This assumption is motivated by prior work
\citep{li2026predicting} on tail-guided prediction, where Gaussian upper-tail behavior is empirically
supported across multiple model and reward-model pairs. Although we focus on the Gaussian tail model for concreteness, our approach can be extended to other tail models as well, as long as the tail behavior can be parametrized by a small number of parameters.

Under \Cref{assum:gaussian-tail-model}, $V_N(\theta;x)$ as defined in
\eqref{eq:promptwise_bon_value} is governed by the extreme upper tail of the
reward distribution. It is therefore natural to predict $V_N(\theta;x)$ using
the conditional reward on a smaller upper-tail region that remains inside the
modeled Gaussian regime. Specifically, define the upper $\alpha$-quantile of
the reward distribution as
\(
r_{\theta,\alpha}(x):=F_{\theta,x}^{-1}(1-\alpha),
\)
and let \[\mu_{\theta,\alpha}(x)
=
\mathbb{E}\!\left[
R_\theta(x)\mid R_\theta(x)\ge r_{\theta,\alpha}(x),\,x
\right], \qquad \sigma_{\theta,\alpha}^2(x) = \mathrm{Var}\!\left[R_\theta(x)\mid R_\theta(x)\ge r_{\theta,\alpha}(x),\,x
\right].\]
We collect these population tail statistics into the vector
\(
\eta_{\theta,\alpha}(x)
:=
\bigl(
r_{\theta,\alpha}(x),
\mu_{\theta,\alpha}(x),
\sigma_{\theta,\alpha}(x)
\bigr),\)
which will be the promptwise population object estimated from the
$m$ training rollouts in the sequel.

The next lemma shows that we can approximate 
$V_N(\theta;x)$ well by these upper-tail statistics.

\begin{lemma}
\label{lem:VN_tail_surrogate}
Under \Cref{assum:gaussian-tail-model}, there exists a constant
\(
\widetilde c_N
\)
depending only on $(N,\alpha)$, such that for every prompt $x$ and parameter
$\theta$,
\begin{equation}
\label{eq:VN_tail_surrogate_bound}
\left|
V_N(\theta;x)
-
\Bigl(
\mu_{\theta,\alpha}(x)+\widetilde c_N\,\sigma_{\theta,\alpha}(x)
\Bigr)
\right|
\le
C(1-2\alpha)^N,
\end{equation}
where $C>0$ depends only on $\alpha$ and the moment bound in
\cref{assum:gaussian-tail-model}.

Moreover, writing $c_N$ as the expected maximum of $N$ i.i.d.\ standard Gaussian random variables,
we have
\(
\widetilde c_N=\frac{c_N-\lambda_\alpha}{\sqrt{\delta_\alpha}}\), where 
\(z_\alpha:=\Phi^{-1}(1-\alpha),
\lambda_\alpha:=\frac{\phi(z_\alpha)}{1-\Phi(z_\alpha)},
\delta_\alpha:=1+z_\alpha\lambda_\alpha-\lambda_\alpha^2.
\)
\end{lemma}
The proof of \Cref{lem:VN_tail_surrogate} is deferred to Appendix~\ref{app:proof-tail-based-surrogate}. \Cref{lem:VN_tail_surrogate} motivates the surrogate objective
\begin{equation}
\label{eq:surrogate_tts_objective}
\hatTTSobj(\theta)
:=
\mathbb{E}_{x \sim P_X}
\Bigl[
\mu_{\theta,\alpha}(x)
+
\widetilde c_N\,\sigma_{\theta,\alpha}(x)
-
\beta\,\mathrm{KL}\bigl(
\curpolicy(\cdot \mid x)\,\|\,\refpolicy(\cdot \mid x)
\bigr)
\Bigr].
\end{equation}
Taking expectation in \eqref{eq:VN_tail_surrogate_bound} over prompts $x$ yields that the gap between
\eqref{eq:exact_tts_objective} and \eqref{eq:surrogate_tts_objective} is also
exponentially small in $N$. Thus, in the rest of the paper we focus on
optimizing the surrogate objective \eqref{eq:surrogate_tts_objective}.

To use \eqref{eq:surrogate_tts_objective} for post-training, we first derive a
standard policy-gradient representation of its population gradient in
\Cref{sec:tail_based_policy_gradient}. Then, in
\Cref{sec:finite_sample_gradient_estimators}, we construct finite-sample
estimators of this oracle gradient from \(m\) rollouts, analyze their
bias--variance tradeoffs, and connect these estimator-level guarantees to
stochastic post-training dynamics.

\subsection{Policy gradient of the tail-based surrogate objective}
\label{sec:tail_based_policy_gradient}

We now turn the surrogate objective
\eqref{eq:surrogate_tts_objective} into a training signal. To obtain the gradient of \eqref{eq:surrogate_tts_objective}, it suffices to identify the promptwise gradient $\nabla_\theta (\mu_{\theta,\alpha}(x)+\widetilde c_N\,\sigma_{\theta,\alpha}(x))$ for each prompt $x$.

\begin{lemma}
\label{lem:tail_based_policy_gradient}
Fix a prompt $x$. For any tail vector $\eta=(r,\mu,\sigma)$ with
$\sigma>0$ and any scalar reward value $u\in\mathbb R$, define the
tail-shaped reward
\begin{equation}
\label{eq:tail_shaped_reward}
\widetilde R_{\eta}(u)
:=
(u-r)
+
\tfrac{\widetilde c_N}{2\sigma}
\big(
(u-\mu)^2-(r-\mu)^2
\big).
\end{equation}
Denoting the policy score function by $S_\theta(x,y):=\nabla_\theta \log \curpolicy(y\mid x)$, we have
\[
g_\theta(x) := \nabla_\theta \bigl(\mu_{\theta,\alpha}(x)+\widetilde c_N\,\sigma_{\theta,\alpha}(x)\bigr)
=
\frac{1}{\alpha}
\mathbb{E}_{y\sim\curpolicy(\cdot\mid x)}
\Bigl[
\Ind{R(x,y)\ge r_{\theta,\alpha}(x)}
\,
\widetilde R_{\eta_{\theta,\alpha}(x)}(R(x,y))
\,
S_\theta(x,y)
\Bigr].
\]
\end{lemma}
The proof is deferred to Appendix~\ref{app:proof-tail-based-policy-gradient}.
\Cref{lem:tail_based_policy_gradient} shows that the promptwise gradient of
\eqref{eq:surrogate_tts_objective} has the standard policy-gradient form
\(
\mathbb E\!\left[A_\theta(x,y)S_\theta(x,y)\right],
\)
but with a tail-dependent weight \(A_\theta(x,y)\). The indicator
\(\Ind{R(x,y)\ge r_{\theta,\alpha}(x)}\) focuses the update on upper-tail
responses; within this tail, the linear term rewards excess above the threshold,
while the quadratic term supplies the tail-scale correction relevant to
best-of-\(N\) performance.

Combining \Cref{lem:tail_based_policy_gradient} with the KL regularizer gives
the population gradient of the surrogate objective:
\begin{equation}
\label{eq:full_surrogate_gradient}
\nabla_\theta \hatTTSobj(\theta)
=
\mathbb E_{x\sim P_X}
\Bigl[
g_\theta(x)
-
\beta\,
\nabla_\theta
\mathrm{KL}\bigl(
\curpolicy(\cdot\mid x)\,\|\,\refpolicy(\cdot\mid x)
\bigr)
\Bigr].
\end{equation}

 In practice, however, the population
tail vector is unknown and must be estimated from only $m$ rollouts per
prompt. The next section provides a family of finite-sample estimators for $g_\theta(x)$.

\section{Finite-sample gradient estimators}
\label{sec:finite_sample_gradient_estimators}

The oracle gradient \(g_\theta(x)\) in \Cref{lem:tail_based_policy_gradient}
depends on the population tail vector \(\eta_{\theta,\alpha}(x)\). We construct
finite-rollout estimators of \(g_\theta(x)\) and show how their bias--variance
guarantees enter post-training.

\subsection{Promptwise gradient estimators}
\label{subsec:promptwise_gradient_estimators}

All results in this subsection are promptwise for a fixed \(x\), with constants
uniform over \(x\) and \(\theta\).
\paragraph{Direct plug-in estimator.}
Draw \(y_1,\dots,y_m\stackrel{\mathrm{i.i.d.}}{\sim}\curpolicy(\cdot\mid x)\)
and set \(R_i:=R(x,y_i)\). Let \(q_m:=\lceil\alpha m\rceil\), let
\(\hat r_m\) be the \(q_m\)-th largest reward, and let \(\mathcal I_m\) be the
indices of the top \(q_m\) rewards, with ties broken randomly. Define
\begin{equation}
\label{eq:empirical_tail_vector}
\hat\mu_m
:=
\frac1{q_m}\sum_{i\in\mathcal I_m}R_i,
\qquad
\hat\sigma_m^2
:=
\max\Big\{ 
\frac1{q_m}\sum_{i\in\mathcal I_m}(R_i-\hat\mu_m)^2,
\varepsilon_\sigma^2
 \Big\},
\qquad
\hat\eta_m:=(\hat r_m,\hat\mu_m,\hat\sigma_m).
\end{equation}

The direct plug-in estimator is obtained by replacing the population tail
vector in the policy gradient formula in
\Cref{lem:tail_based_policy_gradient} by \(\hat\eta_m\):
\begin{equation}
\label{eq:direct_gradient_estimator}
\widehat g_m^{\mathrm{dir}}(\theta;x)
:=
\frac{1}{\alpha m}
\sum_{i=1}^{m}
\Ind{R_i\ge \hat r_m}
\,
\widetilde R_{\hat\eta_m}(R_i)
\,
S_\theta(x,y_i).
\end{equation}
Here \(\widetilde R_{\hat\eta_m}\) is the tail-shaped reward defined in
\eqref{eq:tail_shaped_reward}.

To analyze the finite guarantees of \(\widehat g_m^{\mathrm{dir}}(\theta;x)\), we state two standard regularity assumptions.

\begin{assumption}
\label{assum:bounded_score}
There exists a constant \(G<\infty\) such that
\(
\|S_\theta(x,y)\|\le G
\)
for all prompts \(x\), responses \(y\), and policy parameters \(\theta\).
Moreover, writing
\(
K_\theta(x)
:=
\mathrm{KL}\bigl(
\curpolicy(\cdot\mid x)\|\refpolicy(\cdot\mid x)
\bigr),
\) 
then the first and second \(\theta\)-derivatives of
\(\mu_{\theta}(x)\), \(\sigma_{\theta}(x)\), and \(K_\theta(x)\) are uniformly
bounded by \(G\) over \(x\) and \(\theta\). \looseness=-1
\end{assumption}
In practice, we usually clip the policy score functions. So this assumption is not restrictive. 
\begin{assumption}
\label{assum:nondegenerate_scale}
There exists \(\sigma_{\min}>0\) such that
\(
\sigma_\theta(x)\ge \sigma_{\min}
\)
uniformly over \(x\) and \(\theta\).
\end{assumption}
For the theoretical guarantees, we take
\(
\varepsilon_\sigma
:=
\frac12\sqrt{\delta_\alpha}\,\sigma_{\min},
\)
where \(\delta_\alpha\) is defined in \Cref{lem:VN_tail_surrogate}.  We then have the following bias variance bounds for $\widehat g_m^{\mathrm{dir}}(\theta;x)$.

\begin{theorem}
\label{thm:direct_plugin_estimator}
Suppose Assumptions~\ref{assum:gaussian-tail-model},
\ref{assum:bounded_score}, and \ref{assum:nondegenerate_scale} hold. Let
\(\widehat g_m^{\mathrm{dir}}(\theta;x)\) be defined by
\eqref{eq:direct_gradient_estimator}, with the empirical tail vector
\(\hat\eta_m\) in \eqref{eq:empirical_tail_vector}. Then there
exists a constant \(C>0\) such that, for all sufficiently large \(m\),
\begin{equation}
\label{eq:direct_bias_bound}
\bigl\|
\mathbb E[\widehat g_m^{\mathrm{dir}}(\theta;x)]
-
g_\theta(x)
\bigr\|
\le
C\frac{\sqrt{\log N}}{m}, \qquad
\mathbb E\Big[
\bigl\|
\widehat g_m^{\mathrm{dir}}(\theta;x)
-
\mathbb E[\widehat g_m^{\mathrm{dir}}(\theta;x)]
\bigr\|^2
\Big]
\le
C\frac{\log N}{m}.
\end{equation}
The constant \(C\) depends only on \((\alpha,G,\sigma_{\min},M_R)\), uniformly
over \(x\) and \(\theta\).
\end{theorem}

The proof is deferred to Appendix~\ref{app:gradient_estimator_proofs}.
\Cref{thm:direct_plugin_estimator} shows that the direct plug-in estimator has
\(1/m\) variance rate, but also a systematic \(1/m\) bias. In stochastic first-order methods, it is well-known that the bias can prevent convergence, and that bias reduction is desirable. We now show how to construct a family of debiased estimators that can cancel the bias to any fixed order in \(1/m\).

\paragraph{Fixed-order debiasing.}
In Appendix~\ref{app:fixed_order_aux}, we construct a cross-fitted plug-in estimator that
admits a finite expansion in inverse powers of the rollout budget. For any fixed
order \(k\), one can cancel the first \(k-1\) terms of this expansion by
combining cross-fitted estimates with compatible prefix sizes and
moment-canceling weights. This can be viewed as a generalized jackknife
correction applied to the finite-rollout bias expansion
\citep{schucany1971bias}. We defer the construction to
Appendix~\ref{app:fixed_order_aux}, and state here only the resulting
finite-sample guarantee.

\begin{theorem}
\label{thm:fixed_order_debiased_gradient}
Suppose Assumptions~\ref{assum:gaussian-tail-model},
\ref{assum:bounded_score}, and \ref{assum:nondegenerate_scale} hold. Suppose
\(\alpha\in(0,1/2)\cap\mathbb Q\). For any fixed integers \(1\le k\le J\), there exists an estimator
\(
\widehat g_{m,k,J}^{\mathrm{fo}}(\theta;x)
\) constructed in Appendix~\ref{app:fixed_order_aux} using $m$ rollouts, such that for all sufficiently large \(m\),
\[
\left\|
\mathbb E[\widehat g_{m,k,J}^{\mathrm{fo}}(\theta;x)]
-
g_\theta(x)
\right\|
\le
C_{k,J}\frac{\sqrt{\log N}}{m^k},
\quad
\mathbb E\!\left[
\left\|
\widehat g_{m,k,J}^{\mathrm{fo}}(\theta;x)
-
\mathbb E[\widehat g_{m,k,J}^{\mathrm{fo}}(\theta;x)]
\right\|^2
\right]
\le
C_{k,J}\frac{\log N}{m}.
\]
The constant \(C_{k,J}\) depends only on
\((\alpha,G,\sigma_{\min},M_R,k,J)\), uniformly over \(x,\theta,m\).
\end{theorem}

Here, \(k\) sets the bias-cancellation order and higher \(J\) reduces the constants in the bias and variance bounds.
Thus fixed-order debiasing improves the direct estimator's bias from \(m^{-1}\)
to \(m^{-k}\), while keeping variance at order \(m^{-1}\), up to
\(k,J\)-dependent constants.

\subsection{Finite-rollout post-training implication}
\label{subsec:finite_rollout_post_training}

The preceding bounds are promptwise finite-sample guarantees for estimating the
gradient of the Gaussian-tail surrogate \(\hatTTSobj\). We now explain how they
translate into a post-training guarantee.

At training step \(t\), sample a prompt batch \(\{x_{t,p}\}_{p=1}^P\) from the prompt distribution \(P_X\).
For each prompt, form either the direct estimator
\(\widehat g_m^{\mathrm{dir}}(\theta_t;x_{t,p})\) or the fixed-order estimator
\(\widehat g_{m,k,J}^{\mathrm{fo}}(\theta_t;x_{t,p})\), both using total
rollout budget \(m\). Define
\begin{equation}
\label{eq:prompt_batch_gradient_estimator}
\widehat G_t
:=
\frac1P
\sum_{p=1}^P
\left[
\widehat g(\theta_t;x_{t,p})
-
\beta\nabla_\theta
\mathrm{KL}\bigl(
\curpolicy(\cdot\mid x_{t,p})
\|\refpolicy(\cdot\mid x_{t,p})
\bigr)
\right],
\end{equation}
where \(\widehat g\) denotes the chosen finite-rollout estimator. We analyze the
first-order ascent update \(\theta_{t+1}=\theta_t+\gamma\widehat G_t\).

\begin{proposition}
\label{prop:finite_rollout_post_training}
Suppose Assumptions~\ref{assum:gaussian-tail-model},
\ref{assum:bounded_score}, and \ref{assum:nondegenerate_scale} hold. Consider
the prompt-batch ascent update \(\theta_{t+1}=\theta_t+\gamma\widehat G_t\), and let
\(
\hatTTSobj^\star:=\sup_\theta \hatTTSobj(\theta).
\)
Then there exist constants \(C,c>0\), depending only on
\((\alpha,G,\sigma_{\min},M_R)\), and in the fixed-order case also on the fixed
pair \((k,J)\), such that for every
\(
0<\gamma\le \frac{c}{1+\beta+\sqrt{\log N}},
\)
and all sufficiently large \(m\),
\begin{equation}
\label{eq:finite_rollout_post_training_stationarity}
\frac1T
\sum_{t=0}^{T-1}
\mathbb E
\Big[
\left\|
\nabla_\theta \hatTTSobj(\theta_t)
\right\|^2
\Big]
\le
C
\Big[
\frac{\hatTTSobj^\star-\hatTTSobj(\theta_0)}{\gamma T}
+
\mathsf B_m
+
\gamma\frac{\log N+\beta^2}{P}
\Big],
\end{equation}
where \(\mathsf B_m=\log N/m^2\) for the direct plug-in estimator
\eqref{eq:direct_gradient_estimator}, and
\(\mathsf B_m=\log N/m^{2k}\) for the fixed-order estimator in
\Cref{thm:fixed_order_debiased_gradient}.
\end{proposition}

The proof is deferred to Appendix~\ref{app:proof_finite_rollout_post_training}.
Under an additional gradient-domination condition on \(\hatTTSobj\), the same
argument yields a value guarantee for the exact best-of-\(N\) objective
\(\TTSobj\), up to the Gaussian-tail approximation error; see
\Cref{cor:value_transfer_to_best_of_n} in the appendix.

\begin{remark}
Fixed-order debiasing reduces the systematic finite-rollout bias, but its
constant depends on \((k,J)\). Thus fixed-order debiasing is most useful when
the rollout budget \(m\) and prompt batch \(P\) are large enough for the smaller
bias floor to outweigh the additional constant costs.
Appendix~\ref{app:synthetic_bias_variance_frontier} isolates this frontier in a
controlled Gaussian-tail simulation.
\end{remark}

In the next section, we implement our finite-rollout gradient estimators in a
realistic post-training pipeline and test whether they improve best-of-\(N\)
performance with a small per-prompt rollout budget. \looseness=-1

\section{Experiments}
\label{sec:experiments}
\subsection{Setup and estimators}
\label{subsec:experiments_setup_estimators}

\paragraph{Training interface.}
The estimators in Section~\ref{sec:finite_sample_gradient_estimators} define promptwise updates of the form \looseness=-1
\[
\widehat g_x(A)
=
\frac{1}{m}\sum_{i=1}^m
A_i \nabla_\theta \log \pi_\theta(y_i\mid x).
\]
In experiments, we use this form as the common training interface. For
each prompt \(x\), the current policy samples \(m\) responses, the reward
model assigns rewards \(R_i=R(x,y_i)\), and each method computes a scalar
advantage \(A_i\) from this reward group. We then use \(A_i\) as
stop-gradient advantage weights in the same grouped on-policy training
pipeline with KL regularization. Thus all methods share the rollout
budget, reward model, log-probability computation, KL penalty, optimizer,
and checkpoint selection rule; they differ only in how they construct
\(A_i\).

\paragraph{TEA estimators.}
Section~\ref{sec:finite_sample_gradient_estimators} defines two
finite-rollout gradient estimators. We call the advantage induced by the direct plug-in estimator
\emph{Tail-Extrapolated Advantage} (TEA). The direct plug-in estimator in
Eq.~\eqref{eq:direct_gradient_estimator} already has the score-function
form
\[
\widehat g_m^{\mathrm{dir}}(\theta;x)
=
\frac1m\sum_{i=1}^m
A_i^{\mathrm{raw}}
S_\theta(x,y_i),
\qquad
A_i^{\mathrm{raw}}
=
\frac{1}{\alpha}
\Ind{R_i\ge \hat r_m}
\widetilde R_{\hat\eta_m}(R_i),
\]
where \(\hat\eta_m\) is computed from the same \(m\) rollouts as in
Eq.~\eqref{eq:empirical_tail_vector}. In training, we use a practical stabilized version of the raw estimator:
we take its positive part as a tail score and then center it within the
prompt group
\[
A_i^{\mathrm{TEA}}
=
\bigl(A_i^{\mathrm{raw}}\bigr)_+
-
\frac1m
\sum_{j=1}^m
\bigl(A_j^{\mathrm{raw}}\bigr)_+ .
\]

\emph{Prefix-TEA} denotes the advantage obtained from the fixed-order
estimator in \Cref{thm:fixed_order_debiased_gradient}. In the main
experiments, we use the \((k,J)=(2,4)\) instance and apply the same
positive-part and group-centering steps as in TEA. The implementation details are given in
Appendix~\ref{app:estimator_baseline_definitions}. Unless otherwise stated, both TEA and Prefix-TEA use \(\alpha=0.25\) and
\(N_{\mathrm{target}}=128\), which controls the tail-extrapolation
curvature in the finite-rollout estimator.

\paragraph{Baselines.}
We compare TEA and Prefix-TEA with \textbf{\emph{GRPO}} and its normalized version
\textbf{\emph{GRPO-Z}} \citep{shao2024deepseekmath}, as well as test-time-aware training baselines:
\textbf{\emph{BoN-max mean}}, \textbf{\emph{BoN-max second}}, and
\textbf{\emph{BoN mean}} in~\citep{bagirov2025best},
\textbf{\emph{Chow BoN-RL}} in~\citep{chow2024inference}, and
\textbf{\emph{CAT-BoN}}, our implementation of the Best-of-\(N\)
rank-weighted scaling in~\citep{ousherovitch2026compute}.
All baselines are implemented as alternative advantage rules in the same
grouped training interface. Our main comparisons use the matched per-prompt rollout
budget \(m=64\); for test-time-aware baselines that naturally benefit from
larger candidate sets, we also report \(m=128\) compute-unmatched
tests. Full definitions, formulas, and compute-budget
conventions are given in Appendix~\ref{app:estimator_baseline_definitions}.

\paragraph{Main experimental setting.}
Unless otherwise stated, our default setting trains
\modelLlamaOneB on an
UltraFeedback-derived instruction-following split~\citep{cui2023ultrafeedback},
using \rewardSkyworkLlama
as both the training reward model and the primary reward evaluator. We
construct disjoint training, validation, and held-out evaluation prompt
sets. For the comprehensive baseline comparison, we use \emph{core250}, a
pre-specified 250-prompt held-out set on which we evaluate the full
baseline suite. We additionally report results on \emph{core500}, a
larger 500-prompt held-out set containing core250, for larger-scale
confirmation of the key methods.

The default training configuration uses rollout group size \(m=64\),
KL coefficient \(\beta=0.1\), learning rate \(10^{-5}\), and 450 training
steps. Checkpoints are selected using the same validation rule for all
methods. Held-out evaluation samples \(256\) completions per prompt and
reports grouped best-of-\(N\) reward for
\(N\in\{1,2,4,8,16,32,64,128,256\}\). For scaling experiments, we sample \(512\) completions per prompt to evaluate up to \(N=512\). We report confidence intervals by
paired prompt bootstrap. Full split construction, hyperparameters,
validation protocol, and evaluation details are given in
Appendix~\ref{app:training_hyperparameters}. Additional experiments vary the dataset, reward model, and policy
backbone; each setting is introduced in the corresponding robustness
subsection.

\subsection{Main results on UltraFeedback}
\label{subsec:experiments_main_results}

We first evaluate whether TEA improves test-time best-of-\(N\)
performance after training. Figure~\ref{fig:main_ultrafeedback_performance}(a)
shows the grouped best-of-\(N\) frontier on UltraFeedback core250 when training with \(m=64\) rollouts per prompt.
TEA improves over GRPO across \(N=1,\ldots,256\), and remains above the
strongest tested test-time-aware baseline throughout the frontier.
Prefix-TEA shows the same trend.

\begin{figure*}[!htbp]
\centering
\begin{subfigure}[t]{0.45\linewidth}
  \centering
  \includegraphics[width=\linewidth]{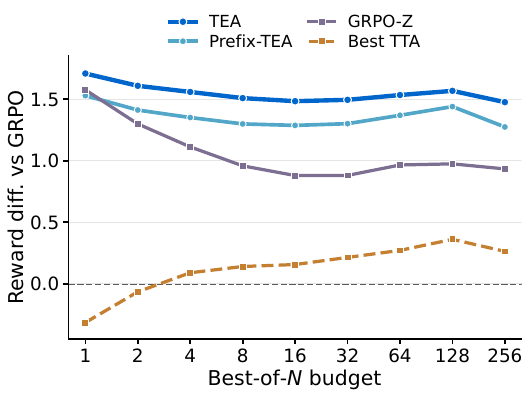}
  \caption{Best-of-\(N\) frontier on UltraFeedback core250.}
  \label{fig:core250_bestofk_curve}
\end{subfigure}
\hfill
\begin{subfigure}[t]{0.45\linewidth}
  \centering
  \includegraphics[width=\linewidth]{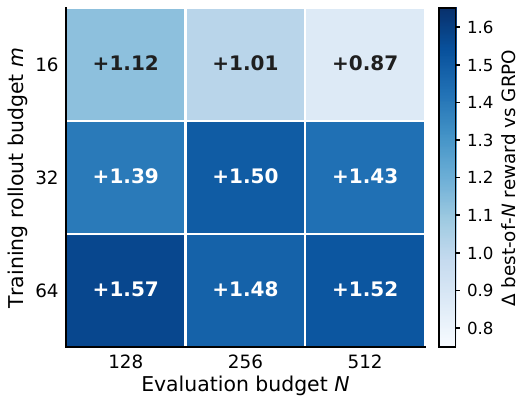}
  \caption{\(m\)-to-\(N\) scaling.}
  \label{fig:m_to_n_scaling}
\end{subfigure}
\caption{
\textbf{UltraFeedback performance and scaling.}
(a) TEA improves the best-of-\(N\) frontier over GRPO and all test-time-aware baselines on UltraFeedback core250.
(b) TEA also improves over matched GRPO across training rollout budgets
\(m\in\{16,32,64\}\) and evaluation budgets
\(N\in\{128,256,512\}\). \looseness=-1
}
\label{fig:main_ultrafeedback_performance}
\end{figure*}

Table~\ref{tab:main_core250_endpoint_summary} quantifies the high-\(N\)
region. TEA and Prefix-TEA substantially outperform GRPO, GRPO-Z, and
the strongest test-time-aware baselines; their paired-bootstrap
confidence intervals against GRPO are positive at all reported
endpoints. TEA also remains well above a compute-unmatched BoN mean
baseline with \(2\times\) per-prompt rollouts during training.

\begin{table}[!htbp]
\centering
\small
\caption{
\textbf{High-\(N\) endpoint results on UltraFeedback core250.}
Values are reward-score improvements over GRPO with paired bootstrap
95\% confidence intervals. Bold point estimates mark the best method in
each column. Best TTA reports the strongest cited test-time-aware
baseline for each column; it is CAT-BoN at bo64 and BoN-max second at
bo128/bo256. BoN mean \(2\times\) uses a compute-unmatched \(m=128\)
per-prompt rollout budget, while all other methods use \(m=64\).
}
\label{tab:main_core250_endpoint_summary}
\begin{tabular}{@{}lccc@{}}
\toprule
Method  & \(\Delta\)bo64 & \(\Delta\)bo128 & \(\Delta\)bo256 \\
\midrule
GRPO-Z  &
\(+0.966\,[+0.757,+1.190]\) &
\(+0.975\,[+0.723,+1.255]\) &
\(+0.933\,[+0.612,+1.297]\) \\
Best TTA &
\(+0.362\,[+0.171,+0.565]\) &
\(+0.361\,[+0.125,+0.607]\) &
\(+0.263\,[-0.053,+0.583]\) \\
BoN mean \(2\times\)  &
\(-0.024\,[-0.209,+0.163]\) &
\(+0.043\,[-0.200,+0.262]\) &
\(+0.016\,[-0.324,+0.319]\) \\
\midrule
\textbf{Prefix-TEA (ours)}  &
\(+1.370\,[+1.147,+1.604]\) &
\(+1.441\,[+1.182,+1.704]\) &
\(+1.275\,[+0.979,+1.598]\) \\
\textbf{TEA (ours)}  &
\(\mathbf{+1.535}\,[+1.309,+1.769]\) &
\(\mathbf{+1.569}\,[+1.305,+1.819]\) &
\(\mathbf{+1.477}\,[+1.177,+1.767]\) \\
\bottomrule
\end{tabular}

\end{table}

\paragraph{Full core250 frontier values.}
Table~\ref{tab:core250_full_baseline_results} reports the full
UltraFeedback core250 baseline comparison underlying
Figure~\ref{fig:core250_bestofk_curve} and
Table~\ref{tab:main_core250_endpoint_summary}. The table includes the
validation-selected checkpoint for each method and the full best-of-\(N\)
frontier from bo1 to bo256.

\begin{table}[!htbp]
\centering
\scriptsize
\caption{
\textbf{Full UltraFeedback core250 baseline comparison.}
The reported metric is the reward-model score. Step represents the chosen
checkpoint step for each method by the validation-based selection rule.
All methods are evaluated with \(N=256\) samples per prompt. BoN mean
\(2\times\) uses a compute-unmatched \(m=128\) rollout budget.
}
\label{tab:core250_full_baseline_results}
\resizebox{\linewidth}{!}{
\begin{tabular}{llrrrrrrrrrr}
\toprule
Method & Budget & Step & bo1 & bo2 & bo4 & bo8 & bo16 & bo32 & bo64 & bo128 & bo256 \\
\midrule
GRPO & \(m=64\) & 100 & -1.129 & 2.226 & 4.719 & 6.731 & 8.398 & 9.803 & 11.033 & 12.149 & 13.273 \\
GRPO-Z & \(m=64\) & 450 & 0.448 & 3.526 & 5.832 & 7.690 & 9.278 & 10.684 & 11.999 & 13.124 & 14.207 \\
TEA & \(m=64\) & 425 & 0.581 & 3.836 & 6.279 & 8.241 & 9.883 & 11.299 & 12.568 & 13.718 & 14.750 \\
Prefix-TEA & \(m=64\) & 450 & 0.401 & 3.639 & 6.071 & 8.031 & 9.686 & 11.105 & 12.403 & 13.590 & 14.548 \\
BoN-max mean & \(m=64\) & 125 & -1.603 & 1.930 & 4.552 & 6.610 & 8.310 & 9.743 & 11.052 & 12.207 & 13.291 \\
BoN-max second & \(m=64\) & 250 & -1.445 & 2.160 & 4.808 & 6.871 & 8.554 & 10.017 & 11.304 & 12.510 & 13.536 \\
BoN mean & \(m=64,k=32\) & 250 & -1.633 & 1.901 & 4.534 & 6.609 & 8.295 & 9.717 & 11.047 & 12.193 & 13.233 \\
BoN mean \(2\times\) & \(m=128,k=64\) & 275 & -1.431 & 2.046 & 4.584 & 6.579 & 8.241 & 9.672 & 11.009 & 12.191 & 13.289 \\
Chow BoN-RL & \(m=64\) & 200 & -1.716 & 1.843 & 4.450 & 6.496 & 8.183 & 9.637 & 10.903 & 12.013 & 13.190 \\
CAT-BoN & \(m=64\) & 425 & -0.800 & 2.593 & 5.095 & 7.083 & 8.757 & 10.158 & 11.395 & 12.482 & 13.484 \\
\bottomrule
\end{tabular}
}
\end{table}

\paragraph{Additional main-result checks.}
We next check whether the headline gains are broad across prompts and
robust to common evaluation artifacts. Table~\ref{tab:main_additional_checks}
summarizes four such checks. First, the gains are not driven by a small
set of prompts: at bo128, TEA beats GRPO on 78.4\% of core250 prompts
and beats the strongest test-time-aware baseline on 70.4\% of prompts.
Second, the reward gains are not explained by longer completions: TEA is
shorter on average than GRPO on the same evaluation records. Third, the
key-method ordering is unchanged at high-\(N\) endpoints on the larger
core500 held-out set. Finally,
using fixed step-450 checkpoints instead of validation-selected
checkpoints still leaves TEA and Prefix-TEA above GRPO and GRPO-Z,
indicating that the main result is not a checkpoint-selection artifact.

\begin{table*}[!htbp]
\centering
\small
\caption{
\textbf{Additional checks for the main UltraFeedback result.}
TEA's gains are broad across prompts, not explained by completion length,
replicate on the larger core500 held-out set, and persist under a fixed
checkpoint comparison.
}
\label{tab:main_additional_checks}
\begin{subtable}[t]{0.62\linewidth}
\centering
\caption{Core250 prompt-level win/tie/loss rates.}
\label{tab:core250_win_tie_loss}
\resizebox{\linewidth}{!}{
\begin{tabular}{lccc}
\toprule
Comparison & bo64 & bo128 & bo256 \\
\midrule
TEA vs GRPO & 80.0/0.8/19.2 & 78.4/0.8/20.8 & 75.2/1.6/23.2 \\
Prefix-TEA vs GRPO & 74.8/0.8/24.4 & 76.4/0.8/22.8 & 69.2/2.0/28.8 \\
TEA vs Best TTA & 75.6/0.8/23.6 & 70.4/0.8/28.8 & 67.2/1.6/31.2 \\
TEA vs GRPO-Z & 63.2/0.0/36.8 & 62.8/1.2/36.0 & 58.4/1.2/40.4 \\
\bottomrule
\end{tabular}
}
\end{subtable}
\hfill
\begin{subtable}[t]{0.32\linewidth}
\centering
\caption{Mean completion length.}
\label{tab:core250_length_check}
\resizebox{\linewidth}{!}{
\begin{tabular}{lrr}
\toprule
Method & Mean length & Source samples \\
\midrule
GRPO & 321.9 & 64000 \\
GRPO-Z & 321.6 & 64000 \\
TEA & 311.4 & 64000 \\
Prefix-TEA & 321.7 & 64000 \\
\bottomrule
\end{tabular}
}
\end{subtable}

\vspace{0.85em}
\begin{subtable}[t]{0.48\linewidth}
\centering
\caption{Core500 confirmation.}
\label{tab:core500_confirmation}
\resizebox{\linewidth}{!}{
\begin{tabular}{lrrrrrrr}
\toprule
Method & bo1 & bo32 & bo64 & bo128 & bo256 & \(\Delta\)bo128 & \(\Delta\)bo256 \\
\midrule
GRPO & -1.172 & 9.850 & 11.179 & 12.334 & 13.400 & 0.000 & 0.000 \\
GRPO-Z & 0.314 & 10.608 & 11.889 & 13.034 & 14.060 & +0.701 & +0.660 \\
Prefix-TEA & 0.298 & 11.107 & 12.409 & 13.575 & 14.618 & +1.242 & +1.218 \\
TEA & 0.467 & 11.193 & 12.504 & 13.675 & 14.708 & +1.342 & +1.307 \\
\bottomrule
\end{tabular}
}
\end{subtable}
\hfill
\begin{subtable}[t]{0.48\linewidth}
\centering
\caption{Fixed step-450 core500 sanity check.}
\label{tab:step450_sanity}
\resizebox{\linewidth}{!}{
\begin{tabular}{lrrrrrrr}
\toprule
Method & bo32 & bo64 & bo128 & bo256 & \(\Delta\)bo128 & \(\Delta\)bo256 & Length \\
\midrule
GRPO & 10.135 & 11.402 & 12.616 & 13.650 & 0.000 & 0.000 & 327.6 \\
GRPO-Z & 10.608 & 11.889 & 13.034 & 14.060 & +0.418 & +0.410 & 325.7 \\
Prefix-TEA & 11.124 & 12.416 & 13.563 & 14.644 & +0.947 & +0.994 & 325.4 \\
TEA & 11.172 & 12.453 & 13.604 & 14.623 & +0.988 & +0.973 & 317.8 \\
\bottomrule
\end{tabular}
}
\end{subtable}
\end{table*}

\FloatBarrier
\subsection{Scaling with the train-time rollout budget}
\label{subsec:experiments_m_to_n_scaling}
We next ask whether TEA remains effective when the train-time rollout
budget changes. Figure~\ref{fig:main_ultrafeedback_performance}(b)
compares TEA with matched GRPO for
\(m\in\{16,32,64\}\), evaluated at
\(N\in\{128,256,512\}\). TEA improves over GRPO in every
\((m,N)\) cell, showing that the gains are not specific to the default
\(m=64\) setting. The improvements are already positive at \(m=16\), become larger for
\(m=32\) and \(m=64\), and persist across all evaluation budgets
\(N\in\{128,256,512\}\).

This pattern is consistent with the intended use case of TEA: even with a
small per-prompt rollout budget during training, modeling the upper tail
can produce gains that transfer to larger inference-time sampling
budgets. The full numerical table makes the budget-mismatch pattern
explicit. As shown in Table~\ref{tab:m_to_n_scaling_full}, TEA improves
matched GRPO in every cell, including the low-budget regime
\(m=16\ll N=512\), where it still improves bo512 by \(+0.867\).
Prompt-level win/tie/loss rates at bo512 also favor TEA for all three
training rollout budgets.

\begin{table*}[!htbp]
\centering
\small
\caption{
\textbf{Full \(m\)-to-\(N\) scaling results on UltraFeedback core250.}
Each row trains matched GRPO and TEA at rollout budget \(m\), then
evaluates best-of-\(N\) endpoints up to \(N=512\). Deltas are TEA minus
GRPO with paired bootstrap 95\% confidence intervals. Reported metrics
are reward-model scores.
}
\label{tab:m_to_n_scaling_full}
\resizebox{\linewidth}{!}{
\begin{tabular}{c rr c rr c rr c c}
\toprule
\(m\) &
GRPO bo128 & TEA bo128 & \(\Delta\)bo128 [95\% CI] &
GRPO bo256 & TEA bo256 & \(\Delta\)bo256 [95\% CI] &
GRPO bo512 & TEA bo512 & \(\Delta\)bo512 [95\% CI] &
Length \(\Delta\) \\
\midrule
16 &
12.129 & 13.254 & \(+1.125\,[+0.896,+1.360]\) &
13.241 & 14.253 & \(+1.012\,[+0.740,+1.315]\) &
14.318 & 15.185 & \(+0.867\,[+0.520,+1.212]\) &
\(-22.9\) \\
32 &
12.466 & 13.857 & \(+1.391\,[+1.025,+1.741]\) &
13.505 & 15.001 & \(+1.496\,[+1.103,+1.885]\) &
14.526 & 15.952 & \(+1.427\,[+0.988,+1.853]\) &
\(-34.4\) \\
64 &
12.149 & 13.718 & \(+1.569\,[+1.307,+1.828]\) &
13.273 & 14.750 & \(+1.477\,[+1.177,+1.768]\) &
14.196 & 15.714 & \(+1.519\,[+1.149,+1.865]\) &
\(-10.5\) \\
\bottomrule
\end{tabular}
}
\end{table*}

\begin{table}[!htbp]
\centering
\small
\caption{
Prompt-level win/tie/loss rates at bo512 for the \(m\)-to-\(N\) scaling
runs. Values are percentages.
}
\label{tab:m_to_n_bo512_wtl}
\begin{tabular}{cc}
\toprule
Train rollout budget \(m\) & bo512 W/T/L \\
\midrule
16 & 59.6/2.4/38.0 \\
32 & 62.0/2.4/35.6 \\
64 & 69.6/1.6/28.8 \\
\bottomrule
\end{tabular}
\end{table}

We also ablate \(N_{\mathrm{target}}\), the parameter controlling the
tail-extrapolation curvature in the finite-rollout estimator.
Table~\ref{tab:target_n_ablation} shows that
\(N_{\mathrm{target}}\) behaves as a tail-emphasis hyperparameter rather
than a literal deployment budget selector. The default
\(N_{\mathrm{target}}=128\) is strongest in this finite-\(m\) setting;
larger values remain positive but are weaker, consistent with a
finite-sample bias--variance tradeoff.

\begin{table*}[!htbp]
\centering
\small
\caption{
\textbf{\(N_{\mathrm{target}}\) ablation at fixed \(m=64\) on UltraFeedback
core250.}
Deltas are relative to the matched GRPO \(m=64\) baseline, with paired
bootstrap 95\% confidence intervals.
}
\label{tab:target_n_ablation}
\resizebox{\linewidth}{!}{
\begin{tabular}{c cccc c}
\toprule
\(N_{\mathrm{target}}\) &
\(\Delta\)bo32 [95\% CI] &
\(\Delta\)bo128 [95\% CI] &
\(\Delta\)bo256 [95\% CI] &
\(\Delta\)bo512 [95\% CI] &
Length \(\Delta\) \\
\midrule
128 &
\(+1.496\,[+1.279,+1.711]\) &
\(+1.569\,[+1.307,+1.828]\) &
\(+1.477\,[+1.177,+1.768]\) &
\(+1.519\,[+1.149,+1.865]\) &
\(-10.5\) \\
256 &
\(+0.723\,[+0.554,+0.883]\) &
\(+0.828\,[+0.588,+1.052]\) &
\(+0.779\,[+0.470,+1.075]\) &
\(+0.840\,[+0.468,+1.213]\) &
\(-11.7\) \\
512 &
\(+0.176\,[+0.030,+0.323]\) &
\(+0.308\,[+0.102,+0.513]\) &
\(+0.286\,[+0.030,+0.546]\) &
\(+0.462\,[+0.103,+0.826]\) &
\(+0.7\) \\
\bottomrule
\end{tabular}
}
\end{table*}

\FloatBarrier
\subsection{Reward-model-independent judge checks}
\label{subsec:experiments_judge_checks}

Because all primary metrics use learned reward models, we additionally
evaluate selected outputs with a position-randomized pairwise LLM judge.
For each comparison, we select one response per method using the same
best-of-\(N\) selection rule as the corresponding reward-model
evaluation, randomly order the two responses, and ask DeepSeek-v4 Pro to
choose the better response while prioritizing correctness, instruction
following, completeness, clarity, and safety rather than length. The
responses are mapped back to method identity after parsing the judge's
JSON object with a winner, confidence, and concise reason.
Table~\ref{tab:deepseek_v4pro_checks} reports the full set of headline
checks. The judge favors TEA in every group: the main UltraFeedback
comparison, all \(m\)-to-\(N\) scaling cells, both policy-backbone
robustness settings, and the dataset/reward-model robustness settings.

\begin{table*}[!htbp]
\centering
\small
\setlength{\tabcolsep}{5pt}
\renewcommand{\arraystretch}{0.95}
\caption{
\textbf{DeepSeek-v4 Pro pairwise judge sanity checks.}
Each row reports randomized A/B comparisons between TEA and the indicated
baseline using selected outputs from the corresponding best-of-\(N\)
evaluation. W--L denotes TEA wins versus baseline wins. Across all
headline checks, the judge favors TEA.
}
\label{tab:deepseek_v4pro_checks}
\begin{tabular}{@{}lllcc@{}}
\toprule
Check & Sel. & Baseline & W--L & Win \\
\midrule
\multicolumn{5}{@{}l}{\emph{Main UltraFeedback}} \\
UF core250 & bo128 & GRPO & 152--98 & 60.8\% \\
UF core250 & bo128 & GRPO-Z & 136--114 & 54.4\% \\
UF core250 & bo128 & Best TTA & 145--105 & 58.0\% \\
\midrule
\multicolumn{5}{@{}l}{\emph{\(m\)-to-\(N\) scaling}} \\
\(m=16\) & bo128 & GRPO & 159--91 & 63.6\% \\
\(m=16\) & bo256 & GRPO & 141--109 & 56.4\% \\
\(m=16\) & bo512 & GRPO & 146--104 & 58.4\% \\
\(m=32\) & bo128 & GRPO & 141--109 & 56.4\% \\
\(m=32\) & bo256 & GRPO & 148--102 & 59.2\% \\
\(m=32\) & bo512 & GRPO & 161--89 & 64.4\% \\
\(m=64\) & bo128 & GRPO & 152--98 & 60.8\% \\
\(m=64\) & bo256 & GRPO & 141--109 & 56.4\% \\
\(m=64\) & bo512 & GRPO & 145--105 & 58.0\% \\
\midrule
\multicolumn{5}{@{}l}{\emph{Policy-backbone robustness}} \\
Llama-3.1-8B LoRA & bo128 & GRPO & 137--113 & 54.8\% \\
Qwen3-4B & bo128 & GRPO & 144--106 & 57.6\% \\
\midrule
\multicolumn{5}{@{}l}{\emph{Dataset / reward-model robustness}} \\
HH-helpful & bo128 & GRPO & 149--101 & 59.6\% \\
Qwen3-RM & bo128 & GRPO & 158--92 & 63.2\% \\
\bottomrule
\end{tabular}
\end{table*}

\FloatBarrier
\subsection{Generalization and robustness}
\label{subsec:experiments_robustness}

\begin{figure*}[!htbp]
\centering
\includegraphics[width=\linewidth]{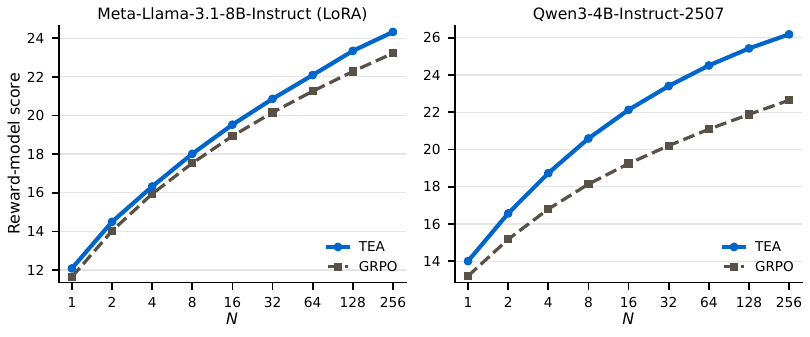}
\caption{
\textbf{Policy-backbone robustness.}
We compare TEA with matched GRPO on two non-default policy backbones.
Left: \modelLlamaEightBMeta, a larger Llama
backbone trained with LoRA and evaluated after merging adapters.
Right: \modelQwenFourB, a different policy family.  TEA
improves the full best-of-\(N\) frontier in both settings.
}
\label{fig:policy_backbone_robustness}
\end{figure*}

We next test whether TEA's improvements transfer across changes to the
dataset, reward model, and policy backbone. For each of these dimensions, we compare TEA with GRPO using the same training interface, while changing only the relevant component.

Table~\ref{tab:dataset_rm_robustness} summarizes dataset and reward-model
robustness. On an HH-helpful split constructed from
\datasetHH~\citep{bai2022training},  TEA gives large reward-score gains over GRPO
across various best-of-\(N\) endpoints. When replacing the training and evaluation reward model with \rewardSkyworkQwen, TEA still improves over GRPO in both average best-of-\(N\) reward and prompt-level win/tie/loss rates.

\begin{table*}[!htbp]
\centering
\small
\caption{
\textbf{Dataset and reward-model robustness.}
Left: TEA improves matched GRPO on an HH-helpful prompt split in
reward-model score. Right: when replacing the training and evaluation
reward model with \rewardSkyworkQwen, TEA still improves over GRPO in both average best-of-\(N\) reward and prompt-level win/tie/loss rates under the new reward model. \looseness=-1
}
\label{tab:dataset_rm_robustness}
\begin{subtable}[t]{0.54\linewidth}
\centering
\caption{HH-helpful dataset robustness.}
\label{tab:hh_helpful_robustness}
\begin{tabular}{rrrr}
\toprule
\(N\) & GRPO & TEA & TEA improvement [95\% CI] \\
\midrule
8   & 10.758 & 12.550 & \(+1.792\,[+1.482,+2.123]\) \\
16  & 12.644 & 14.555 & \(+1.911\,[+1.583,+2.272]\) \\
32  & 14.244 & 16.224 & \(+1.980\,[+1.605,+2.366]\) \\
64  & 15.734 & 17.674 & \(+1.940\,[+1.502,+2.356]\) \\
128 & 17.113 & 18.930 & \(+1.817\,[+1.336,+2.269]\) \\
256 & 18.354 & 20.155 & \(+1.800\,[+1.264,+2.317]\) \\
\bottomrule
\end{tabular}
\end{subtable}
\hfill
\begin{subtable}[t]{0.42\linewidth}
\centering
\caption{Qwen3-8B reward model robustness.}
\label{tab:qwen3_rm_robustness}
\begin{tabular}{rrc}
\toprule
\(N\) & \(\Delta\) RM & TEA vs GRPO W/T/L \\
\midrule
8   & \(+0.196\) & 69.6/0.0/30.4 \\
16  & \(+0.212\) & 67.6/0.0/32.4 \\
32  & \(+0.249\) & 63.6/0.8/35.6 \\
64  & \(+0.292\) & 60.8/0.8/38.4 \\
128 & \(+0.248\) & 58.8/1.6/39.6 \\
256 & \(+0.209\) & 56.0/3.2/40.8 \\
\bottomrule
\end{tabular}
\end{subtable}
\end{table*}

Figure~\ref{fig:policy_backbone_robustness} reports policy-backbone
robustness. We test both a different policy family,
\modelQwenFourB, and a larger Llama backbone,
\modelLlamaEightB. In both settings, we keep
the grouped training interface fixed and compare TEA to matched GRPO. TEA improves the full best-of-\(N\) frontier in both panels, showing that
the gain is not specific to the default Llama-1B backbone. 

The robustness comparisons are also checked with the DeepSeek-v4 Pro
pairwise judge in Table~\ref{tab:deepseek_v4pro_checks}.

\paragraph{HH-helpful dataset robustness.}
For dataset robustness, we train and evaluate matched GRPO and TEA runs on
a helpful-dialogue prompt split constructed from
\datasetHH~\citep{bai2022training}. Specifically, the split has 121{,}079
training prompts, 512 validation prompts, and a deterministic held-out
core250 evaluation set with 250 prompts. Train, validation, and final
evaluation prompt hashes are disjoint. We use the same language model,
reward model and evaluation protocol as in the main UltraFeedback
experiments. Table~\ref{tab:hh_helpful_full_robustness} reports the full
best-of-\(N\) frontier. TEA improves GRPO at every endpoint, and the
confidence intervals exclude zero throughout.

\begin{table}[!htbp]
\centering
\small
\caption{
\textbf{Full HH-helpful dataset robustness results.}
Values are reward-model best-of-\(N\) scores; deltas are TEA minus GRPO
with paired bootstrap 95\% confidence intervals. W/T/L is prompt-level
win/tie/loss percentage.
}
\label{tab:hh_helpful_full_robustness}
\begin{tabular}{rrrrc}
\toprule
\(N\) & GRPO & TEA & \(\Delta\) [95\% CI] & W/T/L \\
\midrule
1   & 2.039  & 3.017  & \(+0.978\,[+0.683,+1.254]\) & 71.6/2.0/26.4 \\
2   & 5.707  & 7.081  & \(+1.375\,[+1.071,+1.687]\) & 76.8/2.0/21.2 \\
4   & 8.485  & 10.120 & \(+1.634\,[+1.321,+1.956]\) & 78.0/2.0/20.0 \\
8   & 10.758 & 12.550 & \(+1.792\,[+1.482,+2.123]\) & 78.4/2.0/19.6 \\
16  & 12.644 & 14.555 & \(+1.911\,[+1.583,+2.272]\) & 77.6/2.0/20.4 \\
32  & 14.244 & 16.224 & \(+1.980\,[+1.605,+2.366]\) & 75.2/2.4/22.4 \\
64  & 15.734 & 17.674 & \(+1.940\,[+1.502,+2.356]\) & 72.8/2.0/25.2 \\
128 & 17.113 & 18.930 & \(+1.817\,[+1.336,+2.269]\) & 67.2/2.4/30.4 \\
256 & 18.354 & 20.155 & \(+1.800\,[+1.264,+2.317]\) & 64.8/3.6/31.6 \\
\bottomrule
\end{tabular}
\end{table}

\paragraph{Qwen3-8B reward-model robustness.}
For reward-model robustness, we replace both the training reward model and
the primary reward evaluator with
\rewardSkyworkQwen, while keeping the same policy backbone and training
protocol as in the main UltraFeedback experiments. Table~\ref{tab:qwen3_rm_full_robustness}
reports the full result. The raw RM-score deltas are smaller because the
new reward model has a different scale, but TEA wins more prompts than
GRPO across all reported endpoints.

\begin{table}[!htbp]
\centering
\small
\caption{
\textbf{Full Qwen3-8B reward-model robustness results.}
TEA and GRPO are trained and evaluated with \rewardSkyworkQwen. Deltas
are TEA minus GRPO reward-model scores with paired bootstrap 95\%
confidence intervals. W/T/L is prompt-level win/tie/loss percentage.
}
\label{tab:qwen3_rm_full_robustness}
\begin{tabular}{rrrrc}
\toprule
\(N\) & GRPO RM & TEA RM & \(\Delta\) [95\% CI] & W/T/L \\
\midrule
1   & -1.013 & -0.735 & \(+0.279\,[+0.146,+0.418]\) & 65.6/0.0/34.4 \\
2   & 1.100  & 1.340  & \(+0.240\,[+0.118,+0.367]\) & 70.0/0.0/30.0 \\
4   & 2.619  & 2.831  & \(+0.212\,[+0.100,+0.325]\) & 70.8/0.0/29.2 \\
8   & 3.764  & 3.960  & \(+0.196\,[+0.097,+0.302]\) & 69.6/0.0/30.4 \\
16  & 4.646  & 4.858  & \(+0.212\,[+0.100,+0.317]\) & 67.6/0.0/32.4 \\
32  & 5.364  & 5.614  & \(+0.249\,[+0.130,+0.362]\) & 63.6/0.8/35.6 \\
64  & 5.992  & 6.284  & \(+0.292\,[+0.164,+0.423]\) & 60.8/0.8/38.4 \\
128 & 6.553  & 6.801  & \(+0.248\,[+0.109,+0.382]\) & 58.8/1.6/39.6 \\
256 & 7.093  & 7.303  & \(+0.209\,[-0.005,+0.405]\) & 56.0/3.2/40.8 \\
\bottomrule
\end{tabular}
\end{table}

\paragraph{Policy-backbone robustness details.}
We evaluate matched GRPO and TEA on two non-default policy backbones:
\modelQwenFourB for model-family transfer and
\modelLlamaEightBMeta for scale transfer within the Llama family. Both
use the UltraFeedback training split, the default \rewardSkyworkLlama
reward model, rollout budget \(m=64\), and TEA hyperparameters
\(\alpha=0.25\), \(N_{\mathrm{target}}=128\). Checkpoints are selected
on the same validation protocol using 128 validation prompts and 32
samples per prompt, and held-out evaluation uses UltraFeedback core250
with \(N=256\) completions per prompt. The Llama-3.1-8B runs use LoRA
training with adapters merged before evaluation. We report the full
Qwen3-4B policy-backbone result in
Table~\ref{tab:qwen4b_policy_robustness} and the Llama-3.1-8B result in
Table~\ref{tab:llama8b_policy_robustness}.

\begin{table}[!htbp]
\centering
\small
\caption{
\textbf{Qwen3-4B policy-backbone robustness.}
Matched GRPO and TEA are trained on UltraFeedback with
\modelQwenFourB and evaluated on core250 with
\(N=256\) completions per prompt. Deltas are TEA minus GRPO reward-model
scores with paired bootstrap 95\% confidence intervals. W/T/L is
prompt-level win/tie/loss percentage.
}
\label{tab:qwen4b_policy_robustness}
\begin{tabular}{rrrrc}
\toprule
\(N\) & GRPO & TEA & \(\Delta\) [95\% CI] & W/T/L \\
\midrule
1   & 13.202 & 14.001 & \(+0.800\,[+0.180,+1.426]\) & 53.6/0.0/46.4 \\
2   & 15.173 & 16.564 & \(+1.392\,[+0.742,+2.066]\) & 59.2/0.0/40.8 \\
4   & 16.797 & 18.727 & \(+1.929\,[+1.275,+2.627]\) & 63.2/0.0/36.8 \\
8   & 18.135 & 20.588 & \(+2.453\,[+1.776,+3.131]\) & 66.8/0.0/33.2 \\
16  & 19.239 & 22.127 & \(+2.888\,[+2.239,+3.579]\) & 71.6/0.0/28.4 \\
32  & 20.202 & 23.406 & \(+3.204\,[+2.537,+3.905]\) & 73.2/0.0/26.8 \\
64  & 21.085 & 24.509 & \(+3.424\,[+2.727,+4.124]\) & 73.6/0.0/26.4 \\
128 & 21.882 & 25.425 & \(+3.543\,[+2.845,+4.265]\) & 72.8/1.2/26.0 \\
256 & 22.642 & 26.184 & \(+3.542\,[+2.813,+4.227]\) & 73.6/0.0/26.4 \\
\bottomrule
\end{tabular}
\end{table}

\begin{table*}[!htbp]
\centering
\small
\caption{
\textbf{Llama-3.1-8B policy-backbone robustness.}
Matched GRPO and TEA are trained with LoRA on UltraFeedback using
\modelLlamaEightBMeta. Left: best-of-\(N\) reward-model scores on
core250. Right: paired bootstrap confidence intervals for TEA minus GRPO
at high-\(N\) endpoints.
}
\label{tab:llama8b_policy_robustness}
\begin{subtable}[t]{0.55\linewidth}
\centering
\caption{Best-of-\(N\) frontier.}
\label{tab:llama8b_policy_robustness_frontier}
\begin{tabular}{rrrr}
\toprule
\(N\) & GRPO & TEA & \(\Delta\) \\
\midrule
1   & 11.640 & 12.081 & +0.441 \\
2   & 14.024 & 14.493 & +0.469 \\
4   & 15.939 & 16.322 & +0.383 \\
8   & 17.529 & 18.014 & +0.485 \\
16  & 18.933 & 19.522 & +0.589 \\
32  & 20.157 & 20.858 & +0.701 \\
64  & 21.255 & 22.090 & +0.835 \\
128 & 22.275 & 23.342 & +1.067 \\
256 & 23.221 & 24.323 & +1.102 \\
\bottomrule
\end{tabular}
\end{subtable}
\hfill
\begin{subtable}[t]{0.40\linewidth}
\centering
\caption{High-\(N\) confidence intervals.}
\label{tab:llama8b_policy_robustness_ci}
\begin{tabular}{cc}
\toprule
Endpoint & TEA - GRPO [95\% CI] \\
\midrule
bo32  & \(+0.701\,[+0.565,+0.832]\) \\
bo64  & \(+0.835\,[+0.659,+1.001]\) \\
bo128 & \(+1.067\,[+0.839,+1.300]\) \\
bo256 & \(+1.103\,[+0.821,+1.384]\) \\
\bottomrule
\end{tabular}
\end{subtable}
\end{table*}

\FloatBarrier
\subsection{Mechanism diagnostics}
\label{subsec:experiments_mechanism}

Finally, we examine why TEA improves best-of-\(N\) performance.
Figure~\ref{fig:mechanism_diagnostics}(a) compares the actual gradients
induced by different training methods with an empirical expected-best-of-128
oracle gradient. TEA is more aligned than GRPO and the strongest
test-time-aware baseline across rollout budgets, supporting the claim
that its finite-rollout advantage better approximates the desired
best-of-\(N\) update. Figure~\ref{fig:mechanism_diagnostics}(b) shows the
corresponding trained-policy effect on the reward distribution: relative
to both the base model and GRPO, TEA shifts probability mass toward
higher-reward completions, with the inset highlighting the high-reward
tail. This suggests that the best-of-\(N\) gains come from a broader
rightward shift of the reward distribution, rather than only from isolated
maximum-reward outliers.

\begin{figure*}[!htbp]
\centering
\begin{subfigure}[t]{0.48\linewidth}
  \centering
  \includegraphics[width=\linewidth]{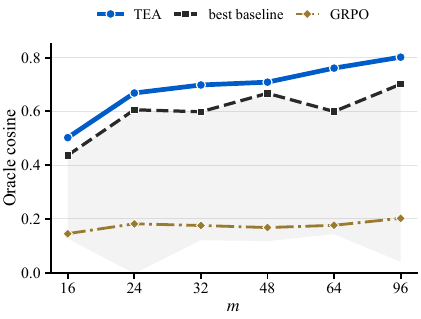}
  \caption{Gradient alignment.}
  \label{fig:mechanism_gradient_alignment}
\end{subfigure}
\hfill
\begin{subfigure}[t]{0.48\linewidth}
  \centering
  \includegraphics[width=\linewidth]{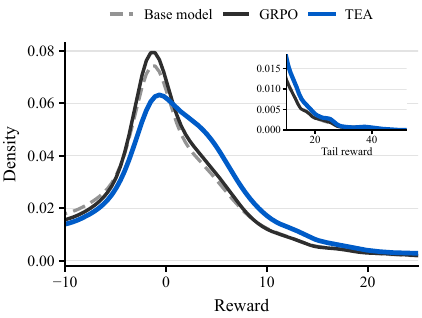}
  \caption{Reward distribution shift.}
  \label{fig:mechanism_reward_distribution}
\end{subfigure}
\caption{
\textbf{Mechanistic diagnostics.}
(a) TEA induces update directions more aligned with an empirical
expected-best-of-128 oracle gradient as the rollout budget \(m\) varies.
(b) After training, TEA shifts the reward distribution to the right
relative to the base model and GRPO; the inset zooms into the high-reward
tail.
}
\label{fig:mechanism_diagnostics}
\end{figure*}

\paragraph{Oracle-gradient alignment diagnostic.}
This diagnostic evaluates whether the advantages produced by different
training rules induce policy-gradient directions aligned with an
expected best-of-\(N\) objective. We use a fixed base-policy sample pool
with 128 prompts and 512 stored completions per prompt, generated by
\modelLlamaOneB and scored by
\rewardSkyworkLlama.

For each prompt, we treat the 512 reward scores as an empirical reward
distribution. For a sample with reward \(r_i\), the expected-best-of-\(N\)
oracle advantage is
\[
A_N^\star(r_i)
=
\mathbb{E}\!\left[\max(r_i,M_{N-1})\right]
-
\mathbb{E}\!\left[M_N\right],
\]
where \(M_N\) denotes the maximum reward among \(N\) independent samples
from the empirical distribution. The corresponding oracle update is
\[
g_N^\star
=
\sum_i
A_N^\star(r_i)
\nabla_\theta \log \pi_\theta(y_i\mid x).
\]
For each training rule, we similarly form
\[
g(A)
=
\sum_i
A_i
\nabla_\theta \log \pi_\theta(y_i\mid x),
\]
where \(A_i\) is the method's training advantage for that rollout
group. We report cosine similarity between \(g(A)\) and \(g_N^\star\).

All gradients are computed on the final transformer block, layer 15,
containing 60,821,504 parameters. The KL/reference term is disabled and
no optimizer step is taken. All results are averaged over 32 prompts and
5 deterministic seeds. Table~\ref{tab:gradient_alignment_mcurve_n128}
gives the full numerical results for the fixed \(N=128\) curve used in
Figure~\ref{fig:mechanism_gradient_alignment}.

\begin{table}[!htbp]
\centering
\caption{
\textbf{Actual-gradient cosine similarity to the empirical expected-best-of-128
oracle gradient as the rollout budget \(m\) varies.}
Values are averaged over 32 prompts and 5 deterministic seeds.
}
\label{tab:gradient_alignment_mcurve_n128}
\begin{tabular}{lrrrrrr}
\toprule
Method & \(m=16\) & \(m=24\) & \(m=32\) & \(m=48\) & \(m=64\) & \(m=96\) \\
\midrule
GRPO & 0.1452 & 0.1818 & 0.1755 & 0.1679 & 0.1763 & 0.2023 \\
TEA & 0.5023 & 0.6683 & 0.6981 & 0.7086 & 0.7608 & 0.8015 \\
BoN-max mean & 0.2874 & 0.4793 & 0.5176 & 0.5366 & 0.5366 & 0.6868 \\
BoN-max second & 0.4366 & 0.6059 & 0.5985 & 0.6671 & 0.5995 & 0.7010 \\
BoN mean & 0.3376 & 0.3842 & 0.4366 & 0.4725 & 0.4930 & 0.6465 \\
Chow BoN-RL & 0.1265 & -0.0032 & 0.1210 & 0.1174 & 0.1431 & 0.0405 \\
CAT-BoN & 0.2980 & 0.4704 & 0.4815 & 0.5437 & 0.5188 & 0.6924 \\
\bottomrule
\end{tabular}
\end{table}

For BoN mean, we use \(k=\lfloor m/2\rfloor\). For Chow BoN-RL, we use a
matched split with \(N_{\mathrm{sel}}=\lfloor m/2\rfloor\) and
\(M_{\mathrm{corr}}=m-N_{\mathrm{sel}}\). CAT-BoN uses the rank-weighted
Best-of-\(N\) scaling rule with \(N_{\mathrm{target}}=128\).

\paragraph{Reward-distribution diagnostic.}
For the reward-distribution panel in
Figure~\ref{fig:mechanism_reward_distribution}, we use existing
UltraFeedback core500 held-out reward records with \(N=256\) completions
per prompt. We compare the base policy, GRPO, and TEA under the same
Skywork-Llama reward evaluator. For each method, we pool all sampled
completion-level reward scores across prompts to estimate the marginal
held-out reward distribution. The main curve is a density-normalized
histogram smoothed with the same one-dimensional Gaussian kernel for all
methods, using a common bin grid and reward range.

Finally, since the estimator is derived from a Gaussian upper-tail
approximation, we check whether this structure is visible in the sampled
reward tails. For each prompt and method, we fit empirical upper-tail
reward quantiles to an affine function of the standard-normal quantile
\[
Q_x(q) = a_x + b_x \Phi^{-1}(q)
\]
over \(q\in[0.80,0.99]\). Table~\ref{tab:gaussian_tail_check_80_99}
shows strong prompt-level QQ linearity: all methods have median
\(R^2\) around \(0.98\), mean \(R^2\) between \(0.962\) and \(0.966\),
and at least \(82\%\) of prompts have \(R^2\ge .95\). This supports the
Gaussian-tail model as a reasonable local approximation to the
high-reward tail used by TEA.

\begin{table}[!htbp]
\centering
\small
\caption{
\textbf{Gaussian QQ upper-tail fit on UltraFeedback core500.}
Fits use \(q\in[0.80,0.99]\) and are computed per prompt; entries
summarize prompt-level \(R^2\) values.
}
\label{tab:gaussian_tail_check_80_99}
\begin{tabular}{lrrrr}
\toprule
Method & Median \(R^2\) & Mean \(R^2\) & p10 \(R^2\) & \(\Pr(R^2\ge .95)\) \\
\midrule
Base model & 0.980 & 0.962 & 0.931 & 82.0\% \\
GRPO & 0.980 & 0.963 & 0.937 & 85.4\% \\
Prefix-TEA & 0.982 & 0.966 & 0.936 & 85.8\% \\
TEA & 0.980 & 0.966 & 0.939 & 86.0\% \\
\bottomrule
\end{tabular}
\end{table}

\FloatBarrier
\section{Discussion}
We studied post-training for language models deployed with test-time strategies such as best-of-\(N\), where performance is governed by the upper tail of the reward distribution rather than the mean reward of a single response. To address the practical budget-mismatch regime \(m\ll N\), we used test-time scaling-law structure to approximate the best-of-\(N\) policy gradient from a small rollout group, leading to Tail-Extrapolated Advantage (TEA) and its fixed-order debiased variant, Prefix-TEA. We provided finite-sample error and post-training guarantees, and showed empirically that these estimators improve best-of-\(N\) performance over standard and test-time-aware baselines across datasets, reward models, and budget settings. 

This work also opens up a number of directions for future research.
\begin{itemize}
  \item We use a Gaussian upper-tail approximation to derive TEA and Prefix-TEA, but other tail structures could be used to construct different extrapolated estimators. In particular, it is interesting to explore methods that adapt to any tail structure, rather than relying on a fixed parametric form. Moreover, it is an important future direction to extend the idea to non-smooth reward distributions, such as zero-one rewards in the pass@\(k\) setting.
  \item It is useful to extend the methods to other test-time strategies beyond best-of-\(N\), such as tree search methods, self-refine procedures, or interaction with external tools. In general, the tail-extrapolation principle applies to any test-time strategy that admits a scaling-law structure, but the specific form of the estimator and the training procedure may need to be adapted to the strategy and its scaling law.
\end{itemize}

\section*{Acknowledgments}
This work was partially supported by NSERC grant RGPIN-2024-05092 and a Connaught New Researcher Award to WM.

\newpage
\bibliographystyle{alpha}
\bibliography{mybib,Formatting_Instructions_For_NeurIPS_2026/neurips_extra}

\begin{thebibliography}{26}
\providecommand{\natexlab}[1]{#1}
\providecommand{\url}[1]{\texttt{#1}}
\expandafter\ifx\csname urlstyle\endcsname\relax
  \providecommand{\doi}[1]{doi: #1}\else
  \providecommand{\doi}{doi: \begingroup \urlstyle{rm}\Url}\fi

\bibitem[Bagirov et~al.(2025)Bagirov, Arkhipov, Sycheva, Glukhov, and
  Bogomolov]{bagirov2025best}
Farid Bagirov, Mikhail Arkhipov, Ksenia Sycheva, Evgeniy Glukhov, and Egor
  Bogomolov.
\newblock The best of n worlds: Aligning reinforcement learning with best-of-n
  sampling via max@ k optimisation.
\newblock \emph{arXiv preprint arXiv:2510.23393}, 2025.

\bibitem[Bai et~al.(2022)Bai, Jones, Ndousse, Askell, Chen, DasSarma, Drain,
  Fort, Ganguli, Henighan, et~al.]{bai2022training}
Yuntao Bai, Andy Jones, Kamal Ndousse, Amanda Askell, Anna Chen, Nova DasSarma,
  Dawn Drain, Stanislav Fort, Deep Ganguli, Tom Henighan, et~al.
\newblock Training a helpful and harmless assistant with reinforcement learning
  from human feedback.
\newblock \emph{arXiv preprint arXiv:2204.05862}, 2022.

\bibitem[Balashankar et~al.(2024)Balashankar, Sun, Berant, Eisenstein, Collins,
  Hutter, Lee, Nagpal, Prost, Sinha, et~al.]{balashankar2024infalign}
Ananth Balashankar, Ziteng Sun, Jonathan Berant, Jacob Eisenstein, Michael
  Collins, Adrian Hutter, Jong Lee, Chirag Nagpal, Flavien Prost, Aradhana
  Sinha, et~al.
\newblock Infalign: Inference-aware language model alignment.
\newblock \emph{arXiv preprint arXiv:2412.19792}, 2024.

\bibitem[Brown et~al.(2024)Brown, Juravsky, Ehrlich, Clark, Le, R{\'e}, and
  Mirhoseini]{brown2024large}
Bradley Brown, Jordan Juravsky, Ryan Ehrlich, Ronald Clark, Quoc~V Le,
  Christopher R{\'e}, and Azalia Mirhoseini.
\newblock Large language monkeys: Scaling inference compute with repeated
  sampling.
\newblock \emph{arXiv preprint arXiv:2407.21787}, 2024.

\bibitem[Chow et~al.(2024)Chow, Tennenholtz, Gur, Zhuang, Dai, Thiagarajan,
  Boutilier, Agarwal, Kumar, and Faust]{chow2024inference}
Yinlam Chow, Guy Tennenholtz, Izzeddin Gur, Vincent Zhuang, Bo~Dai, Sridhar
  Thiagarajan, Craig Boutilier, Rishabh Agarwal, Aviral Kumar, and Aleksandra
  Faust.
\newblock Inference-aware fine-tuning for best-of-n sampling in large language
  models.
\newblock \emph{arXiv preprint arXiv:2412.15287}, 2024.

\bibitem[Cobbe et~al.(2021)Cobbe, Kosaraju, Bavarian, Chen, Jun, Kaiser,
  Plappert, Tworek, Hilton, Nakano, et~al.]{cobbe2021training}
Karl Cobbe, Vineet Kosaraju, Mohammad Bavarian, Mark Chen, Heewoo Jun, Lukasz
  Kaiser, Matthias Plappert, Jerry Tworek, Jacob Hilton, Reiichiro Nakano,
  et~al.
\newblock Training verifiers to solve math word problems.
\newblock \emph{arXiv preprint arXiv:2110.14168}, 2021.

\bibitem[Cui et~al.(2023)Cui, Yuan, Ding, Yao, He, Zhu, Ni, Xie, Xie, Lin,
  et~al.]{cui2023ultrafeedback}
Ganqu Cui, Lifan Yuan, Ning Ding, Guanming Yao, Bingxiang He, Wei Zhu, Yuan Ni,
  Guotong Xie, Ruobing Xie, Yankai Lin, et~al.
\newblock Ultrafeedback: Boosting language models with scaled ai feedback.
\newblock \emph{arXiv preprint arXiv:2310.01377}, 2023.

\bibitem[Gui et~al.(2024)Gui, G{\^a}rbacea, and Veitch]{gui2024bonbon}
Lin Gui, Cristina G{\^a}rbacea, and Victor Veitch.
\newblock Bonbon alignment for large language models and the sweetness of
  best-of-n sampling.
\newblock volume~37, pages 2851--2885, 2024.

\bibitem[Kirk et~al.(2023)Kirk, Mediratta, Nalmpantis, Luketina, Hambro,
  Grefenstette, and Raileanu]{kirk2023understanding}
Robert Kirk, Ishita Mediratta, Christoforos Nalmpantis, Jelena Luketina, Eric
  Hambro, Edward Grefenstette, and Roberta Raileanu.
\newblock Understanding the effects of rlhf on llm generalisation and
  diversity.
\newblock \emph{arXiv preprint arXiv:2310.06452}, 2023.

\bibitem[Li et~al.(2026)Li, Qian, and Mou]{li2026predicting}
Muheng Li, Jian Qian, and Wenlong Mou.
\newblock Predicting and improving test-time scaling laws via reward
  tail-guided search.
\newblock \emph{arXiv preprint arXiv:2602.01485}, 2026.

\bibitem[Novikov et~al.(2025)Novikov, V{\~u}, Eisenberger, Dupont, Huang,
  Wagner, Shirobokov, Kozlovskii, Ruiz, Mehrabian,
  et~al.]{novikov2025alphaevolve}
Alexander Novikov, Ng{\^a}n V{\~u}, Marvin Eisenberger, Emilien Dupont, Po-Sen
  Huang, Adam~Zsolt Wagner, Sergey Shirobokov, Borislav Kozlovskii,
  Francisco~JR Ruiz, Abbas Mehrabian, et~al.
\newblock Alphaevolve: A coding agent for scientific and algorithmic discovery.
\newblock \emph{arXiv preprint arXiv:2506.13131}, 2025.

\bibitem[Ousherovitch and Tewari(2026)]{ousherovitch2026compute}
Adam Ousherovitch and Ambuj Tewari.
\newblock Compute aligned training: Optimizing for test time inference.
\newblock \emph{arXiv preprint arXiv:2604.24957}, 2026.

\bibitem[Ouyang et~al.(2022)Ouyang, Wu, Jiang, Almeida, Wainwright, Mishkin,
  Zhang, Agarwal, Slama, Ray, et~al.]{ouyang2022training}
Long Ouyang, Jeffrey Wu, Xu~Jiang, Diogo Almeida, Carroll Wainwright, Pamela
  Mishkin, Chong Zhang, Sandhini Agarwal, Katarina Slama, Alex Ray, et~al.
\newblock Training language models to follow instructions with human feedback.
\newblock \emph{Advances in neural information processing systems},
  35:\penalty0 27730--27744, 2022.

\bibitem[O’Mahony et~al.(2024)O’Mahony, Grinsztajn, Schoelkopf, and
  Biderman]{omahony2024attributing}
Laura O’Mahony, Leo Grinsztajn, Hailey Schoelkopf, and Stella Biderman.
\newblock Attributing mode collapse in the fine-tuning of large language
  models.
\newblock In \emph{ICLR 2024 Workshop on Mathematical and Empirical
  Understanding of Foundation Models}, volume~2, page~2, 2024.

\bibitem[Rafailov et~al.(2023)Rafailov, Sharma, Mitchell, Manning, Ermon, and
  Finn]{rafailov2023direct}
Rafael Rafailov, Archit Sharma, Eric Mitchell, Christopher~D Manning, Stefano
  Ermon, and Chelsea Finn.
\newblock Direct preference optimization: Your language model is secretly a
  reward model.
\newblock \emph{Advances in neural information processing systems},
  36:\penalty0 53728--53741, 2023.

\bibitem[Razin et~al.(2025)Razin, Wang, Strauss, Wei, Lee, and
  Arora]{razin2025makes}
Noam Razin, Zixuan Wang, Hubert Strauss, Stanley Wei, Jason~D Lee, and Sanjeev
  Arora.
\newblock What makes a reward model a good teacher? an optimization
  perspective.
\newblock \emph{arXiv preprint arXiv:2503.15477}, 2025.

\bibitem[Rosenthal(1970)]{rosenthal1970subspaces}
Haskell~P Rosenthal.
\newblock On the subspaces of l p (p> 2) spanned by sequences of independent
  random variables.
\newblock \emph{Israel Journal of Mathematics}, 8\penalty0 (3):\penalty0
  273--303, 1970.

\bibitem[Schucany et~al.(1971)Schucany, Gray, and Owen]{schucany1971bias}
WR~Schucany, HL~Gray, and DB~Owen.
\newblock On bias reduction in estimation.
\newblock \emph{Journal of the American Statistical Association}, 66\penalty0
  (335):\penalty0 524--533, 1971.

\bibitem[Shao et~al.(2024)Shao, Wang, Zhu, Xu, Song, Bi, Zhang, Zhang, Li, Wu,
  et~al.]{shao2024deepseekmath}
Zhihong Shao, Peiyi Wang, Qihao Zhu, Runxin Xu, Junxiao Song, Xiao Bi, Haowei
  Zhang, Mingchuan Zhang, YK~Li, Yang Wu, et~al.
\newblock Deepseekmath: Pushing the limits of mathematical reasoning in open
  language models.
\newblock \emph{arXiv preprint arXiv:2402.03300}, 2024.

\bibitem[Snell et~al.(2024)Snell, Lee, Xu, and Kumar]{snell2024scaling}
Charlie Snell, Jaehoon Lee, Kelvin Xu, and Aviral Kumar.
\newblock Scaling llm test-time compute optimally can be more effective than
  scaling model parameters.
\newblock \emph{arXiv preprint arXiv:2408.03314}, 2024.

\bibitem[Sun et~al.(2025)Sun, Shen, Wang, Chen, Wang, Zhou, and
  Zhang]{sun2025improving}
Yifan Sun, Jingyan Shen, Yibin Wang, Tianyu Chen, Zhendong Wang, Mingyuan Zhou,
  and Huan Zhang.
\newblock Improving data efficiency for llm reinforcement fine-tuning through
  difficulty-targeted online data selection and rollout replay.
\newblock \emph{arXiv preprint arXiv:2506.05316}, 2025.

\bibitem[Tang et~al.(2025)Tang, Zheng, Synnaeve, and Munos]{tang2025optimizing}
Yunhao Tang, Kunhao Zheng, Gabriel Synnaeve, and R{\'e}mi Munos.
\newblock Optimizing language models for inference time objectives using
  reinforcement learning.
\newblock \emph{arXiv preprint arXiv:2503.19595}, 2025.

\bibitem[Walder and Karkhanis(2025)]{walder2025pass}
Christian Walder and Deep Karkhanis.
\newblock Pass@ k policy optimization: Solving harder reinforcement learning
  problems.
\newblock \emph{arXiv preprint arXiv:2505.15201}, 2025.

\bibitem[Wang et~al.(2022)Wang, Wei, Schuurmans, Le, Chi, Narang, Chowdhery,
  and Zhou]{wang2022self}
Xuezhi Wang, Jason Wei, Dale Schuurmans, Quoc Le, Ed~Chi, Sharan Narang,
  Aakanksha Chowdhery, and Denny Zhou.
\newblock Self-consistency improves chain of thought reasoning in language
  models.
\newblock \emph{arXiv preprint arXiv:2203.11171}, 2022.

\bibitem[Yuksekgonul et~al.(2026)Yuksekgonul, Koceja, Li, Bianchi, McCaleb,
  Wang, Kautz, Choi, Zou, Guestrin, et~al.]{yuksekgonul2026learning}
Mert Yuksekgonul, Daniel Koceja, Xinhao Li, Federico Bianchi, Jed McCaleb,
  Xiaolong Wang, Jan Kautz, Yejin Choi, James Zou, Carlos Guestrin, et~al.
\newblock Learning to discover at test time.
\newblock \emph{arXiv preprint arXiv:2601.16175}, 2026.

\bibitem[Zhang and Zuo(2025)]{zhang2025grpo}
Jixiao Zhang and Chunsheng Zuo.
\newblock Grpo-lead: A difficulty-aware reinforcement learning approach for
  concise mathematical reasoning in language models.
\newblock In \emph{Proceedings of the 2025 Conference on Empirical Methods in
  Natural Language Processing}, pages 5642--5665, 2025.

\end{thebibliography}

\newpage

\appendix

\section{Additional related work}
\label{app:additional_related_work}

\paragraph{Single-response post-training.}
Most LLM post-training methods are designed to improve the quality of a single
sampled response. Reinforcement learning with human feedback (RLHF) fine-tunes a policy against a learned reward model,
typically with a KL penalty to keep the policy close to a reference model
\citep{ouyang2022training}. Preference-optimization methods such as DPO
optimize an equivalent preference objective without explicitly running online RL
\citep{rafailov2023direct}, while grouped policy-gradient methods such as GRPO
improve the efficiency of RL-style post-training by normalizing rewards within a
rollout group \citep{shao2024deepseekmath}. These methods are natural when the
deployed model returns one response, but they do not directly optimize the
multi-candidate behavior used by test-time strategies such as best-of-\(N\).

This distinction is practically important because post-training changes not only
the average quality of responses, but also the distribution of candidate responses at
test time. Empirical studies show that RLHF, instruction tuning, and
reward-based fine-tuning can reduce the diversity of generated candidates
\citep{kirk2023understanding,omahony2024attributing}. Complementary theory
shows that reward variation also matters for optimization: if a reward model
induces too little variance under the current policy, the RLHF objective can
become flat and slow to optimize \citep{razin2025makes}. These works motivate a
distributional view of post-training, but they do not specify how to optimize a
target best-of-\(N\) deployment objective.

\paragraph{Test-time strategies and test-time scaling.}
A separate line of work improves language model performance by spending more compute at
deployment. Verifier-based methods generate many candidate solutions and return
the one ranked highest by a verifier \citep{cobbe2021training}. Self-consistency
samples multiple reasoning paths and aggregates their final answers
\citep{wang2022self}. Repeated-sampling studies show that coverage can
increase predictably with the number of attempts, especially when candidate
solutions can be automatically checked \citep{brown2024large}. More broadly,
test-time compute has become an independent scaling axis: it can be allocated to
search, verifier scoring, adaptive procedures, or even problem-specific
test-time training, and has been used to surpass prior human-engineered or
best-known solutions in discovery-style tasks
\citep{snell2024scaling,yuksekgonul2026learning,novikov2025alphaevolve}.
These results make the budget mismatch especially important. At deployment, one
may spend a large amount of compute on a single hard prompt to find one
exceptional response; during post-training, the same per-prompt budget is
usually infeasible because optimization must cover many prompts.

Best-of-\(N\) is one of the simplest and most widely used instances of this
principle. BoNBoN studies the best-of-\(N\) distribution itself, relates it to
KL-regularized alignment, and trains a model to mimic this distribution so that
the sampling cost can be amortized \citep{gui2024bonbon}. Closest to our
test-time scaling perspective, prior work shows that best-of-\(N\) scaling laws
can be predicted from \(m\ll N\) samples by extrapolating the upper tail of the
reward distribution \citep{li2026predicting}. Our work uses the same high-level
idea for a different purpose: rather than predicting the value of best-of-\(N\)
inference or amortizing its output distribution, we use tail extrapolation to
construct a policy-gradient estimator for improving best-of-\(N\) performance.

\paragraph{Test-time-aware post-training.}
Recent work directly adapts post-training to the deployment-time procedure.
Best-of-\(N\)-aware imitation-learning and RL objectives have been introduced
\citep{chow2024inference}. Generic \(k\)-sample test-time objectives,
including pass@\(k\) and majority vote, have also been optimized
\citep{tang2025optimizing}. Low-variance unbiased estimators have been derived
for pass@\(k\) and its gradient \citep{walder2025pass}, and RL has been aligned
with best-of-\(N\) sampling through max@\(k\) optimization
\citep{bagirov2025best}. More general frameworks also show that alignment
should depend on the test-time rule: InfAlign derives
procedure-dependent reward transformations \citep{balashankar2024infalign}, and
Compute Aligned Training views test-time strategies as operators and derives
corresponding training losses \citep{ousherovitch2026compute}.

These works establish that post-training should account for how the model will
be used at deployment. Our focus is the budget-mismatch regime left open by this
line of work. Existing test-time-aware objectives typically construct their
training signal by assuming access to the same per-prompt rollout budget as the deployment procedure. In contrast,
we study the case where training observes only \(m\ll N\) rollouts per prompt,
so the target best-of-\(N\) objective cannot be directly formed during
post-training. TEA uses the small training group only to estimate local
upper-tail statistics, then extrapolates the corresponding gradient signal
toward the larger best-of-\(N\) deployment objective.

\section{Additional Experimental Details}
\label{app:additional_experimental_details}

\subsection{Estimator and Baseline Definitions}
\label{app:estimator_baseline_definitions}

We define all implemented training methods as alternative advantage rules \(A_i\) for
the same prompt-level rollout group \(\{(y_i,R_i)\}_{i=1}^m\) as in \Cref{sec:experiments}. All
advantages are treated as stop-gradient weights in the grouped training
interface.

\paragraph{TEA (ours).}
For completeness, we spell out the advantage used by our
finite-rollout estimators. For a prompt \(x\), let
\(y_1,\ldots,y_m\sim \pi_\theta(\cdot\mid x)\) be the sampled responses
and \(R_i=R(x,y_i)\) their reward-model scores. As in
\eqref{eq:empirical_tail_vector}, let
\(\hat\eta_m=(\hat r_m,\hat\mu_m,\hat\sigma_m)\) denote the empirical
top-\(\alpha\) tail vector computed from these \(m\) rewards. The direct
plug-in estimator in \eqref{eq:direct_gradient_estimator} can be written
in score-function form with raw advantage
\begin{equation}
\label{eq:app_tea_raw_advantage}
A_i^{\mathrm{raw}}
=
\frac{1}{\alpha}
\Ind{R_i\ge \hat r_m}
\widetilde R_{\hat\eta_m}(R_i).
\end{equation}
In experiments, TEA uses the positive part of this raw advantage and then
centers it within the prompt group:
\begin{equation}
\label{eq:app_tea_advantage}
A_i^{\mathrm{TEA}}
=
\bigl(A_i^{\mathrm{raw}}\bigr)_+
-
\frac1m\sum_{j=1}^m
\bigl(A_j^{\mathrm{raw}}\bigr)_+,
\qquad
(z)_+ := \max\{z,0\}.
\end{equation}

\paragraph{Prefix-TEA (ours).}
Prefix-TEA is the implementation of the fixed-order estimator in
\Cref{thm:fixed_order_debiased_gradient}, with the specific form defined
in \eqref{eq:app_total_budget_fixed_order_estimator}. Let
\((m_1,\ldots,m_J)\) be the prescribed prefix sizes and
\((w_1,\ldots,w_J)\) the corresponding cancellation weights. For each
prefix \(j\), we compute the empirical tail vector
\(\hat\eta_{m_j}^{(j)}\) using only the first \(m_j\) rewards in the
rollout group, and define the prefix raw advantage
\begin{equation}
\label{eq:app_prefix_tea_raw_advantage}
A_{i,j}^{\mathrm{raw}}
=
\Ind{i\le m_j}\,
\frac{1}{\alpha}
\Ind{R_i\ge \hat r_{m_j}^{(j)}}
\widetilde R_{\hat\eta_{m_j}^{(j)}}(R_i).
\end{equation}
To express the weighted prefix gradients in the common full-group
training interface, define \(\rho_j:=m/m_j\). We apply the same
positive-part operation to each prefix raw advantage, combine the prefix
scores with the fixed-order weights and the prefix-to-full-group scaling,
and then center over the full rollout group:
\begin{equation}
\label{eq:app_prefix_tea_advantage}
C_i^{\mathrm{Prefix}}
=
\sum_{j=1}^J
w_j\,\rho_j\,
\bigl(A_{i,j}^{\mathrm{raw}}\bigr)_+,
\qquad
A_i^{\mathrm{Prefix\text{-}TEA}}
=
C_i^{\mathrm{Prefix}}
-
\frac1m\sum_{\ell=1}^m
C_\ell^{\mathrm{Prefix}} .
\end{equation}

In the main experiments we use the Prefix-TEA \((k,J)=(2,4)\) instance.
For the default rollout group size \(m=64\), this corresponds to prefix
sizes
\[
(m_1,m_2,m_3,m_4)=(40,48,56,64)
\]
and cancellation weights
\[
(w_1,w_2,w_3,w_4)
=
(-1.82946,\,-0.15392,\,1.04289,\,1.94050).
\]
The corresponding prefix-to-full-group scaling factors are
\[
(\rho_1,\rho_2,\rho_3,\rho_4)
=
(1.600,\,1.333,\,1.143,\,1.000).
\]
The theoretical fixed-order estimator is stated in a cross-fitted form
for analysis. Our main implementation uses fixed prefixes of the same
rollout group, without a separate fit/evaluation split.

\paragraph{GRPO and GRPO-Z \citep{shao2024deepseekmath}.}
GRPO uses the mean-centered reward advantage
\begin{equation}
\label{eq:app_grpo_advantage}
A_i^{\mathrm{GRPO}}
=
R_i-\bar R,
\qquad
\bar R=\frac1m\sum_{j=1}^m R_j .
\end{equation}
GRPO-Z is the group-normalized variant,
\begin{equation}
\label{eq:app_grpoz_advantage}
A_i^{\mathrm{GRPO\text{-}Z}}
=
\frac{R_i-\bar R}
{\operatorname{std}(R_1,\ldots,R_m)+\varepsilon}.
\end{equation}
We use GRPO-Z as an advantage-scale control.

\paragraph{BoN-max mean and BoN-max second \citep{bagirov2025best}.}
These are selected-output BoN baselines following the BoN-max baselines
listed in \citep{bagirov2025best}. Let
\[
i^\star=\arg\max_i R_i,
\qquad
R^\star=R_{i^\star},
\]
with deterministic tie-breaking, and let \(R^{(2)}\) denote the
second-largest reward in the group. BoN-max mean gives nonzero advantage
only to the reward-selected response:
\begin{equation}
\label{eq:app_bonmax_mean_advantage}
A_i^{\mathrm{BoN\text{-}max\ mean}}
=
\begin{cases}
R^\star-\bar R, & i=i^\star,\\
0, & \text{otherwise}.
\end{cases}
\end{equation}
BoN-max second instead uses the runner-up reward as the baseline:
\begin{equation}
\label{eq:app_bonmax_second_advantage}
A_i^{\mathrm{BoN\text{-}max\ second}}
=
\begin{cases}
R^\star-R^{(2)}, & i=i^\star,\\
0, & \text{otherwise}.
\end{cases}
\end{equation}

\paragraph{BoN mean \citep{bagirov2025best}.}
We use the name BoN mean for the transformed-reward baseline derived
from the max@\(k\) objective in \citep{bagirov2025best}. Let the rewards be sorted in ascending order,
\[
r_{(1)}\le r_{(2)}\le \cdots \le r_{(m)}.
\]
For a target subset size \(k<m\), define the transformed reward at sorted
position \(i\) as
\begin{equation}
\label{eq:app_maxatk_transformed_reward}
B_{(i)}
=
r_{(i)}
\frac{\binom{i-1}{k-1}}{\binom{m}{k}}
+
\sum_{j>i}
r_{(j)}
\frac{\binom{j-2}{k-2}}{\binom{m}{k}} .
\end{equation}
We unsort \(B_{(i)}\) back to the original sample order and use the
group-normalized advantage
\begin{equation}
\label{eq:app_maxatk_advantage}
A_i^{\mathrm{Max@K}}
=
\frac{B_i-\bar B}
{\operatorname{std}(B_1,\ldots,B_m)+\varepsilon}.
\end{equation}
The exact Max@\(k\) transformation requires \(k<m\). In the matched
\(m=64\) setting we use \(k=32\); in the \(m=128\) stress tests we use
\(k=64\).

\paragraph{Chow BoN-RL\citep{chow2024inference}.}
We implement a Monte Carlo estimator of the Lemma~2 BoN-RL gradient in \citep{chow2024inference},
  under the scalar reward-model specialization
\(R=r\) and prompt baseline \(b=0\). For a group of size \(m\), we
randomly split the group into a BoN selection set \(\mathcal S\) of size
\(N_{\mathrm{sel}}\) and an independent correction set \(\mathcal C\) of
size \(M_{\mathrm{corr}}\), with
\(N_{\mathrm{sel}}+M_{\mathrm{corr}}=m\). Let
\[
i^\star=\arg\max_{i\in\mathcal S} R_i,
\qquad
R^\star=R_{i^\star}.
\]
The selected BoN response receives the selected-response term, and
correction samples receive the Lemma~2 correction term:
\begin{equation}
\label{eq:app_chow_advantage}
A_i^{\mathrm{Chow}}
=
m\left[
\Ind{i=i^\star}R^\star
-
\Ind{i\in\mathcal C}\,
\frac{\lambda_{N_{\mathrm{sel}}}}{M_{\mathrm{corr}}}
\Ind{R_i>R^\star}\,R^\star
\right].
\end{equation}
Here \(\lambda_{N_{\mathrm{sel}}}\) is the coefficient appearing in
Chow et al.'s Lemma~2 for a BoN selection set of size
\(N_{\mathrm{sel}}\). The prefactor \(m\) accounts for the fact that our
trainer averages the loss over the \(m\) sampled responses. In the
matched \(m=64\) setting we use \(N_{\mathrm{sel}}=32\) and
\(M_{\mathrm{corr}}=32\); in the \(m=128\) stress tests we use a \(64/64\)
split. 

We should note that in \citep{chow2024inference}, the authors target verifier-selected BoN inference and report
pass@\(K\) evaluations, whereas our setting uses a learned reward
model for instruction-following. We therefore compare the Lemma~2
estimator under a matched rollout budget rather than reproducing their
full experimental setup.

\paragraph{CAT-BoN\citep{ousherovitch2026compute}.}
CAT-BoN follows the Compute Aligned Training principle in
\citep{ousherovitch2026compute}, which reweights gradients by the
marginal utility of a sample under a test-time strategy. For BoN, we use
a rank-based scaling. Let \(\hat F_{<}(R_i)\) be the empirical fraction
of group rewards strictly below \(R_i\), with the same tie-handling rule
used in the implementation. We define
\begin{equation}
\label{eq:app_cat_weight}
w_i
=
N_{\mathrm{target}}\,
\hat F_{<}(R_i)^{N_{\mathrm{target}}-1},
\qquad
\tilde w_i=\frac{w_i}{\frac1m\sum_{j=1}^m w_j+\varepsilon}.
\end{equation}
CAT-BoN applies this normalized rank weight to the group-normalized
reward advantage:
\begin{equation}
\label{eq:app_cat_advantage}
A_i^{\mathrm{CAT}}
=
\tilde w_i\,
\frac{R_i-\bar R}
{\operatorname{std}(R_1,\ldots,R_m)+\varepsilon}.
\end{equation}
Unless otherwise stated, we use \(N_{\mathrm{target}}=128\) for CAT-BoN.

\subsection{Training and Evaluation Details}
\label{app:training_hyperparameters}

\paragraph{UltraFeedback splits.}
We construct fixed, disjoint training, validation, and held-out evaluation
prompt sets from UltraFeedback~\citep{cui2023ultrafeedback}. After
deduplication and removing prompts overlapping the source test split, the
eligible pool contains 37,756 prompts. We use 34,994 prompts for policy
training, 512 prompts as a validation pool, and 2,000 prompts as a
held-out final pool. The main comprehensive baseline comparison uses
\emph{core250}, a deterministic 250-prompt held-out subset. We also use
\emph{core500}, a larger 500-prompt held-out set containing core250, to
confirm the main method ordering on a larger evaluation set. Training,
validation, and held-out evaluation prompts are disjoint.

\paragraph{Checkpoint selection.}
All methods use the same checkpoint-selection rule. Each $25$ training steps,
we evaluate the current checkpoint on the validation pool. Specifically, for each checkpoint,
we sample 32 completions for each of 128 fixed validation prompts and
score them with the reward model. We compute the mean reward of the top
10 completions for each prompt, average this value over prompts, and
select the checkpoint with the highest average. Validation prompts are
used only for checkpoint selection and are never used for held-out
reporting.

\paragraph{Reported best-of-\(N\) reward scores.}
Unless otherwise stated, all reported bo\(N\) values are reward-model
scores under the grouped best-of-\(N\) evaluation protocol. For each
held-out prompt \(x\), we sample \(M\) completions and score them as
\(R_i=R(x,y_i)\). For a reported budget \(N\) dividing \(M\), we partition
the \(M\) completions into \(M/N\) consecutive groups of size \(N\), take
the maximum reward in each group, average these maxima within the prompt,
and then average over prompts. For the main held-out evaluation,
\(M=256\) and \(N\in\{1,2,4,8,16,32,64,128,256\}\). The
\(m\)-to-\(N\) scaling experiments use \(M=512\) and report bo128,
bo256, and bo512.

\paragraph{Uncertainty estimates.}
Confidence intervals are computed using paired prompt bootstrap with
1,000 resamples. Prompt-level win/tie/loss is computed from the same
prompt-level grouped best-of-\(k\) values: a method wins a prompt if its
value is larger than the comparison method, loses if it is smaller, and
ties if the absolute difference is at most $10^{-9}$.

What's more, all reported experiments were run on NVIDIA L40S GPU workers with 48GB GPU
memory per device.

\section{Proofs for problem setup and population gradients}
\label{app:problem-setup-gradient-proofs}

Throughout this section, constants denoted by \(C_{\alpha,M_R}\) may change from
line to line, but depend only on \((\alpha,M_R)\).

\begin{lemma}
\label{lem:uniform_sigma_upper_bound}
Under Assumption~\ref{assum:gaussian-tail-model}, there exists a constant
\(\sigmamax=\sigmamax(\alpha,M_R)<\infty\) such that
\[
\sigma_\theta(x)\le \sigmamax
\]
for all prompts \(x\) and all policy parameters \(\theta\).
\end{lemma}

\begin{proof}
Fix a prompt \(x\) and parameter \(\theta\). Throughout the proof, probabilities
and expectations are taken with respect to the conditional reward law induced
by this fixed prompt and policy. Write
\[
R:=R_\theta(x),\qquad
p:=p_{\theta,x},\qquad
\mu:=\mu_\theta(x),\qquad
\sigma:=\sigma_\theta(x),\qquad
q:=r_{\theta,2\alpha}(x).
\]
By Assumption~\ref{assum:gaussian-tail-model},
\[
p(r)=\frac{1}{\sigma}\phi\!\left(\frac{r-\mu}{\sigma}\right),
\qquad r\ge q.
\]
Since \(\mathbb P(R\ge q)=2\alpha\), integrating over \([q,\infty)\) gives
\[
q=\mu+\sigma z_{2\alpha},
\qquad
z_{2\alpha}:=\Phi^{-1}(1-2\alpha).
\]

Let
\[
\kappa_\alpha
:=
\int_{z_{2\alpha}}^\infty (u-z_{2\alpha})^4\phi(u)\,du.
\]
Then \(\kappa_\alpha>0\) depends only on \(\alpha\), and the change of variables
\(r=\mu+\sigma u\) gives
\begin{equation}
\label{eq:app_sigma_upper_identity}
\kappa_\alpha\sigma^4
=
\int_q^\infty (r-q)^4p(r)\,dr.
\end{equation}
Using \((a-b)^4\le 8(a^4+b^4)\),
\begin{align}
\int_q^\infty (r-q)^4p(r)\,dr
&\le
8\int_q^\infty r^4p(r)\,dr
+
8q^4\int_q^\infty p(r)\,dr \notag\\
&\le
8\mathbb E[R^4]+16\alpha q^4.
\label{eq:app_sigma_upper_rhs}
\end{align}

It remains to bound \(q\). If \(q>0\), Markov's inequality gives
\[
2\alpha=\mathbb P(R\ge q)
\le
\frac{\mathbb E[R^4]}{q^4}
\le
\frac{M_R}{q^4}.
\]
If \(q<0\), then
\[
1-2\alpha
=
\mathbb P(R\le q)
\le
\mathbb P(|R|\ge |q|)
\le
\frac{\mathbb E[R^4]}{|q|^4}
\le
\frac{M_R}{|q|^4}.
\]
The case \(q=0\) is trivial. Hence
\[
q^4
\le
\frac{M_R}{\min\{2\alpha,1-2\alpha\}}.
\]
Substituting this bound into \eqref{eq:app_sigma_upper_rhs} and using
\(\mathbb E[R^4]\le M_R\), we obtain
\[
\int_q^\infty (r-q)^4p(r)\,dr
\le C_{\alpha,M_R}.
\]
Combining with \eqref{eq:app_sigma_upper_identity},
\[
\sigma^4\le \frac{C_{\alpha,M_R}}{\kappa_\alpha}.
\]
Thus
\[
\sigma\le
\left(\frac{C_{\alpha,M_R}}{\kappa_\alpha}\right)^{1/4}
=:\sigmamax.
\]
Since \(\kappa_\alpha\) depends only on \(\alpha\), the constant \(\sigmamax\)
depends only on \((\alpha,M_R)\). The claim follows.
\end{proof}

\begin{corollary}
\label{cor:uniform_tail_bounds}
Suppose Assumptions~\ref{assum:gaussian-tail-model} and
\ref{assum:nondegenerate_scale} hold. Then there exists a constant
\(C_{\alpha,M_R,\sigma_{\min}}>0\), depending only on
\((\alpha,M_R,\sigma_{\min})\), such that uniformly over all prompts \(x\) and
policy parameters \(\theta\),
\[
|r_{\theta,2\alpha}(x)|
+
|\mu_\theta(x)|
+
\sigma_\theta(x)
+
|r_{\theta,\alpha}(x)|
+
|\mu_{\theta,\alpha}(x)|
+
\sigma_{\theta,\alpha}(x)
\le
C_{\alpha,M_R,\sigma_{\min}}.
\]
Moreover,
\[
\sigma_\theta(x)\ge \sigma_{\min},
\qquad
\sigma_{\theta,\alpha}(x)\ge \rho_\alpha\sigma_{\min},
\]
where
\[
z_\alpha:=\Phi^{-1}(1-\alpha),
\qquad
\lambda_\alpha:=\frac{\phi(z_\alpha)}{1-\Phi(z_\alpha)}
=
\frac{\phi(z_\alpha)}{\alpha},
\qquad
\rho_\alpha
:=
\sqrt{1+z_\alpha\lambda_\alpha-\lambda_\alpha^2}.
\]
Finally, with
\[
z_{2\alpha}:=\Phi^{-1}(1-2\alpha),
\]
we have
\[
r_{\theta,\alpha}(x)-r_{\theta,2\alpha}(x)
\ge
\sigma_{\min}(z_\alpha-z_{2\alpha})>0.
\]
\end{corollary}

\begin{proof}
Fix \(x\) and \(\theta\), and write
\[
q_{2\alpha}:=r_{\theta,2\alpha}(x),
\qquad
z_{2\alpha}:=\Phi^{-1}(1-2\alpha).
\]
By the Gaussian-tail calibration in
Assumption~\ref{assum:gaussian-tail-model},
\[
q_{2\alpha}
=
\mu_\theta(x)+\sigma_\theta(x)z_{2\alpha}.
\]
The same Markov argument used in the proof of
\Cref{lem:uniform_sigma_upper_bound} gives
\[
|q_{2\alpha}|
\le
\left(
\frac{M_R}{\min\{2\alpha,1-2\alpha\}}
\right)^{1/4}.
\]
By \Cref{lem:uniform_sigma_upper_bound},
\[
\sigma_\theta(x)\le \sigmamax(\alpha,M_R).
\]
Since
\[
\mu_\theta(x)
=
q_{2\alpha}-\sigma_\theta(x)z_{2\alpha},
\]
we obtain
\[
|q_{2\alpha}|+|\mu_\theta(x)|+\sigma_\theta(x)
\le
C_{\alpha,M_R}.
\]
The lower bound
\[
\sigma_\theta(x)\ge\sigma_{\min}
\]
is exactly Assumption~\ref{assum:nondegenerate_scale}.

The remaining bounds follow from the truncated-normal identities. With
\[
z_\alpha:=\Phi^{-1}(1-\alpha),
\qquad
\lambda_\alpha:=\frac{\phi(z_\alpha)}{1-\Phi(z_\alpha)},
\qquad
\rho_\alpha
:=
\sqrt{1+z_\alpha\lambda_\alpha-\lambda_\alpha^2},
\]
the top-\(\alpha\) threshold, mean, and scale satisfy
\[
r_{\theta,\alpha}(x)
=
\mu_\theta(x)+\sigma_\theta(x)z_\alpha,
\]
\[
\mu_{\theta,\alpha}(x)
=
\mu_\theta(x)+\sigma_\theta(x)\lambda_\alpha,
\]
and
\[
\sigma_{\theta,\alpha}(x)
=
\rho_\alpha\sigma_\theta(x).
\]
Using the already established bounds on
\(|\mu_\theta(x)|\) and \(\sigma_\theta(x)\), we get
\[
|r_{\theta,\alpha}(x)|
+
|\mu_{\theta,\alpha}(x)|
+
\sigma_{\theta,\alpha}(x)
\le
C_{\alpha,M_R,\sigma_{\min}}.
\]
Also,
\[
\sigma_{\theta,\alpha}(x)
=
\rho_\alpha\sigma_\theta(x)
\ge
\rho_\alpha\sigma_{\min}.
\]
Finally,
\[
r_{\theta,\alpha}(x)-r_{\theta,2\alpha}(x)
=
\sigma_\theta(x)(z_\alpha-z_{2\alpha})
\ge
\sigma_{\min}(z_\alpha-z_{2\alpha}).
\]
Since \(\alpha<1/2\), we have
\[
1-\alpha>1-2\alpha,
\qquad
z_\alpha>z_{2\alpha},
\]
and hence
\[
\sigma_{\min}(z_\alpha-z_{2\alpha})>0.
\]
This proves the corollary.
\end{proof}

\subsection{Proof of \Cref{lem:VN_tail_surrogate}}
\label{app:proof-tail-based-surrogate}

\begin{proof}[Proof of \Cref{lem:VN_tail_surrogate}]
Fix a prompt \(x\) and parameter \(\theta\). Throughout the proof, probabilities
and expectations are taken with respect to the conditional reward law induced
by this fixed prompt and policy. Write
\[
R:=R_\theta(x),\quad
F:=F_{\theta,x},\quad
p:=p_{\theta,x},\quad
\mu:=\mu_\theta(x),\quad
\sigma:=\sigma_\theta(x),
\]
and
\[
q_{2\alpha}:=r_{\theta,2\alpha}(x),\qquad
q_\alpha:=r_{\theta,\alpha}(x),\qquad
\mu_\alpha:=\mu_{\theta,\alpha}(x),\qquad
\sigma_\alpha:=\sigma_{\theta,\alpha}(x).
\]
Let
\[
G(r):=\Phi\!\left(\frac{r-\mu}{\sigma}\right)
\]
be the CDF of \(\mathcal N(\mu,\sigma^2)\).

\paragraph{Tail moment identity.}
By Assumption~\ref{assum:gaussian-tail-model}, the density of \(R\) agrees
with the Gaussian density on \([q_{2\alpha},\infty)\). Since
\(\mathbb P(R\ge q_{2\alpha})=2\alpha\), we have
\[
q_{2\alpha}=\mu+\sigma z_{2\alpha},
\qquad
z_{2\alpha}:=\Phi^{-1}(1-2\alpha).
\]
Similarly, since \(q_\alpha\ge q_{2\alpha}\) and
\(\mathbb P(R\ge q_\alpha)=\alpha\),
\[
q_\alpha=\mu+\sigma z_\alpha,
\qquad
z_\alpha:=\Phi^{-1}(1-\alpha).
\]
Therefore, the conditional law of \(R\) given \(R\ge q_\alpha\) is the law of
\(\mu+\sigma Z\) given \(Z\ge z_\alpha\), where \(Z\sim\mathcal N(0,1)\).

Define
\[
\lambda_\alpha:=\frac{\phi(z_\alpha)}{1-\Phi(z_\alpha)},
\qquad
\delta_\alpha:=1+z_\alpha\lambda_\alpha-\lambda_\alpha^2,
\]
where \(\phi\) and \(\Phi\) are the standard normal density and CDF,
respectively. The standard truncated-normal identities give
\[
\mu_\alpha=\mu+\lambda_\alpha\sigma,
\qquad
\sigma_\alpha=\sqrt{\delta_\alpha}\,\sigma.
\]
Let
\[
c_N:=\mathbb E\!\left[\max_{1\le i\le N} Z_i\right],
\qquad
Z_1,\dots,Z_N\stackrel{\mathrm{i.i.d.}}{\sim}\mathcal N(0,1),
\]
and
\[
\widetilde c_N:=\frac{c_N-\lambda_\alpha}{\sqrt{\delta_\alpha}}.
\]
Then
\begin{equation}
\label{eq:app_tail_identity}
\mu_\alpha+\widetilde c_N\sigma_\alpha
=
\mu+c_N\sigma.
\end{equation}

\paragraph{Comparison with the Gaussian maximum.}
Let
\[
Y_1,\dots,Y_N\stackrel{\mathrm{i.i.d.}}{\sim}\mathcal N(\mu,\sigma^2).
\]
Then
\[
\mathbb E\!\left[\max_{1\le i\le N}Y_i\right]
=
\mu+c_N\sigma.
\]
Let \(R_1,\dots,R_N\) be i.i.d. copies of \(R\), and define
\[
U:=\max_{1\le i\le N}R_i,
\qquad
V:=\max_{1\le i\le N}Y_i.
\]
By the integrated-CDF identity for integrable random variables,
\[
\mathbb E[U]-\mathbb E[V]
=
\int_{-\infty}^{\infty}
\bigl(F_V(r)-F_U(r)\bigr)\,dr.
\]
Since the CDFs of \(U\) and \(V\) are \(F(r)^N\) and \(G(r)^N\), respectively,
and by the definition of \(V_N(\theta;x)\),
\[
V_N(\theta;x)-(\mu+c_N\sigma)
=
\int_{-\infty}^{\infty}
\bigl(G(r)^N-F(r)^N\bigr)\,dr.
\]

On the upper \(2\alpha\) tail, the two distributions agree. Indeed, for every
\(r\ge q_{2\alpha}\), the tail densities are identical and
\[
1-F(q_{2\alpha})=2\alpha=1-G(q_{2\alpha}).
\]
Therefore,
\[
F(r)=G(r),
\qquad r\ge q_{2\alpha},
\]
and hence
\[
V_N(\theta;x)-(\mu+c_N\sigma)
=
\int_{-\infty}^{q_{2\alpha}}
\bigl(G(r)^N-F(r)^N\bigr)\,dr.
\]
For \(r\le q_{2\alpha}\), both \(F(r)\) and \(G(r)\) are at most \(1-2\alpha\).
Thus
\begin{align}
\left|V_N(\theta;x)-(\mu+c_N\sigma)\right|
&\le
\int_{-\infty}^{q_{2\alpha}}
\bigl(G(r)^N+F(r)^N\bigr)\,dr \notag\\
&\le
(1-2\alpha)^{N-1}
\int_{-\infty}^{q_{2\alpha}}
\bigl(G(r)+F(r)\bigr)\,dr.
\label{eq:app_lower_region_bound}
\end{align}

It remains to bound the last integral. First, the Markov argument in
\Cref{lem:uniform_sigma_upper_bound} gives
\[
|q_{2\alpha}|
\le
\left(
\frac{M_R}{\min\{2\alpha,1-2\alpha\}}
\right)^{1/4}.
\]
Hence
\[
\int_{-\infty}^{q_{2\alpha}}F(r)\,dr
=
\mathbb E[(q_{2\alpha}-R)_+]
\le
|q_{2\alpha}|+\mathbb E|R|
\le
C_{\alpha,M_R}.
\]
For the Gaussian term, using \(q_{2\alpha}=\mu+\sigma z_{2\alpha}\) and the
change of variables \(u=(r-\mu)/\sigma\),
\[
\int_{-\infty}^{q_{2\alpha}}G(r)\,dr
=
\sigma\int_{-\infty}^{z_{2\alpha}}\Phi(u)\,du.
\]
The last integral is finite and depends only on \(\alpha\):
\[
\int_{-\infty}^{z_{2\alpha}}\Phi(u)\,du
=
z_{2\alpha}\Phi(z_{2\alpha})+\phi(z_{2\alpha}).
\]
Together with \Cref{lem:uniform_sigma_upper_bound}, this gives
\[
\int_{-\infty}^{q_{2\alpha}}G(r)\,dr
\le
C_{\alpha,M_R}.
\]
Plugging these bounds into \eqref{eq:app_lower_region_bound} yields
\[
\left|V_N(\theta;x)-(\mu+c_N\sigma)\right|
\le
C_{\alpha,M_R}(1-2\alpha)^{N-1}
\le
C(1-2\alpha)^N,
\]
for some constant \(C>0\) depending only on \((\alpha,M_R)\). Combining this
with \eqref{eq:app_tail_identity} gives
\[
\left|
V_N(\theta;x)
-
\Bigl(
\mu_{\theta,\alpha}(x)+\widetilde c_N\sigma_{\theta,\alpha}(x)
\Bigr)
\right|
\le
C(1-2\alpha)^N.
\]
This proves the claim.
\end{proof}

\subsection{Proof of \Cref{lem:tail_based_policy_gradient}}
\label{app:proof-tail-based-policy-gradient}

\begin{proof}[Proof of \Cref{lem:tail_based_policy_gradient}]
Fix a prompt \(x\). Let \(\mathbb E_\theta\) denote expectation with respect to
\(y\sim\curpolicy(\cdot\mid x)\), and write
\[
R(y):=\rewardmodel(x,y),
\qquad
S_\theta(x,y):=\nabla_\theta\log\curpolicy(y\mid x).
\]
We keep the tail quantities
\(r_{\theta,\alpha}(x)\), \(\mu_{\theta,\alpha}(x)\), and
\(\sigma_{\theta,\alpha}(x)\) explicit throughout the proof.

We first calculate the derivative of the moving upper-tail threshold. Since
\(r_{\theta,\alpha}(x)\) is the upper-\(\alpha\) quantile,
\begin{equation}
\label{eq:app_tail_mass_identity}
\mathbb E_\theta
\left[
\Ind{R(y)\ge r_{\theta,\alpha}(x)}
\right]
=
\alpha .
\end{equation}
Equivalently,
\[
\int_{r_{\theta,\alpha}(x)}^\infty p_{\theta,x}(u)\,du
=
\alpha .
\]
Differentiating this identity with respect to \(\theta\) and using Leibniz'
rule gives
\[
0
=
-
p_{\theta,x}\!\bigl(r_{\theta,\alpha}(x)\bigr)
\nabla_\theta r_{\theta,\alpha}(x)
+
\int_{r_{\theta,\alpha}(x)}^\infty
\nabla_\theta p_{\theta,x}(u)\,du .
\]
For the second term, the likelihood-ratio identity gives, with the threshold
held fixed at its current value \(r_{\theta,\alpha}(x)\),
\[
\int_{r_{\theta,\alpha}(x)}^\infty
\nabla_\theta p_{\theta,x}(u)\,du
=
\mathbb E_\theta
\left[
\Ind{R(y)\ge r_{\theta,\alpha}(x)}
S_\theta(x,y)
\right].
\]
Therefore,
\begin{equation}
\label{eq:app_quantile_boundary_identity}
p_{\theta,x}\!\bigl(r_{\theta,\alpha}(x)\bigr)
\nabla_\theta r_{\theta,\alpha}(x)
=
\mathbb E_\theta
\left[
\Ind{R(y)\ge r_{\theta,\alpha}(x)}
S_\theta(x,y)
\right].
\end{equation}

More generally, for any scalar function \(f\) that does not depend on
\(\theta\), the same calculation yields
\begin{align}
\nabla_\theta
\mathbb E_\theta
\left[
\Ind{R(y)\ge r_{\theta,\alpha}(x)}
f(R(y))
\right]
&=
\mathbb E_\theta
\left[
\Ind{R(y)\ge r_{\theta,\alpha}(x)}
f(R(y))S_\theta(x,y)
\right] \notag\\
&\quad
-
f\!\bigl(r_{\theta,\alpha}(x)\bigr)
p_{\theta,x}\!\bigl(r_{\theta,\alpha}(x)\bigr)
\nabla_\theta r_{\theta,\alpha}(x) \notag\\
&=
\mathbb E_\theta
\left[
\Ind{R(y)\ge r_{\theta,\alpha}(x)}
\bigl(f(R(y))-f(r_{\theta,\alpha}(x))\bigr)
S_\theta(x,y)
\right].
\label{eq:app_moving_threshold_identity}
\end{align}

We now apply this identity to the upper-tail mean. By definition,
\[
\mu_{\theta,\alpha}(x)
=
\frac{1}{\alpha}
\mathbb E_\theta
\left[
\Ind{R(y)\ge r_{\theta,\alpha}(x)}
R(y)
\right].
\]
Taking \(f(u)=u\) in \eqref{eq:app_moving_threshold_identity} gives
\begin{equation}
\label{eq:app_grad_tail_mean}
\nabla_\theta \mu_{\theta,\alpha}(x)
=
\frac{1}{\alpha}
\mathbb E_\theta
\left[
\Ind{R(y)\ge r_{\theta,\alpha}(x)}
\bigl(R(y)-r_{\theta,\alpha}(x)\bigr)
S_\theta(x,y)
\right].
\end{equation}

Next consider the upper-tail variance. Define
\[
B_\theta(x)
:=
\mathbb E_\theta
\left[
\Ind{R(y)\ge r_{\theta,\alpha}(x)}
\bigl(R(y)-\mu_{\theta,\alpha}(x)\bigr)^2
\right],
\]
so that
\[
\sigma_{\theta,\alpha}^2(x)
=
\frac{1}{\alpha}B_\theta(x).
\]
Differentiating \(B_\theta(x)\) gives
\begin{align}
\nabla_\theta B_\theta(x)
&=
\mathbb E_\theta
\left[
\Ind{R(y)\ge r_{\theta,\alpha}(x)}
\bigl(R(y)-\mu_{\theta,\alpha}(x)\bigr)^2
S_\theta(x,y)
\right] \notag\\
&\quad
-
\bigl(r_{\theta,\alpha}(x)-\mu_{\theta,\alpha}(x)\bigr)^2
p_{\theta,x}\!\bigl(r_{\theta,\alpha}(x)\bigr)
\nabla_\theta r_{\theta,\alpha}(x) \notag\\
&\quad
-
2\,
\mathbb E_\theta
\left[
\Ind{R(y)\ge r_{\theta,\alpha}(x)}
\bigl(R(y)-\mu_{\theta,\alpha}(x)\bigr)
\right]
\nabla_\theta\mu_{\theta,\alpha}(x).
\end{align}
The last term vanishes because \(\mu_{\theta,\alpha}(x)\) is the conditional
mean on the upper tail:
\[
\mathbb E_\theta
\left[
\Ind{R(y)\ge r_{\theta,\alpha}(x)}
\bigl(R(y)-\mu_{\theta,\alpha}(x)\bigr)
\right]
=
0.
\]
Using \eqref{eq:app_quantile_boundary_identity}, we obtain
\[
\nabla_\theta \sigma_{\theta,\alpha}^2(x)
=
\frac{1}{\alpha}
\mathbb E_\theta
\left[
\Ind{R(y)\ge r_{\theta,\alpha}(x)}
\Bigl(
\bigl(R(y)-\mu_{\theta,\alpha}(x)\bigr)^2
-
\bigl(r_{\theta,\alpha}(x)-\mu_{\theta,\alpha}(x)\bigr)^2
\Bigr)
S_\theta(x,y)
\right].
\]
Since \(\sigma_{\theta,\alpha}(x)>0\),
\begin{equation}
\label{eq:app_grad_tail_std}
\nabla_\theta \sigma_{\theta,\alpha}(x)
=
\frac{1}{2\alpha\,\sigma_{\theta,\alpha}(x)}
\mathbb E_\theta
\left[
\Ind{R(y)\ge r_{\theta,\alpha}(x)}
\Bigl(
\bigl(R(y)-\mu_{\theta,\alpha}(x)\bigr)^2
-
\bigl(r_{\theta,\alpha}(x)-\mu_{\theta,\alpha}(x)\bigr)^2
\Bigr)
S_\theta(x,y)
\right].
\end{equation}

Finally,
\[
h_\theta(x)
=
\mu_{\theta,\alpha}(x)
+
\widetilde c_N\,\sigma_{\theta,\alpha}(x).
\]
Combining \eqref{eq:app_grad_tail_mean} and
\eqref{eq:app_grad_tail_std} yields
\begin{align*}
\nabla_\theta h_\theta(x)
&=
\frac{1}{\alpha}
\mathbb E_\theta
\Bigg[
\Ind{R(y)\ge r_{\theta,\alpha}(x)}
\Bigg\{
\bigl(R(y)-r_{\theta,\alpha}(x)\bigr) \\
&\qquad\qquad
+
\frac{\widetilde c_N}{2\sigma_{\theta,\alpha}(x)}
\Bigl(
\bigl(R(y)-\mu_{\theta,\alpha}(x)\bigr)^2
-
\bigl(r_{\theta,\alpha}(x)-\mu_{\theta,\alpha}(x)\bigr)^2
\Bigr)
\Bigg\}
S_\theta(x,y)
\Bigg].
\end{align*}
By the definition of the tail-shaped reward
\(\widetilde R_{\eta}(u)\) in \eqref{eq:tail_shaped_reward}, with
\[
\eta_{\theta,\alpha}(x)
=
\bigl(
r_{\theta,\alpha}(x),
\mu_{\theta,\alpha}(x),
\sigma_{\theta,\alpha}(x)
\bigr),
\]
this is exactly
\[
\nabla_\theta h_\theta(x)
=
\frac{1}{\alpha}
\mathbb E_{y\sim\curpolicy(\cdot\mid x)}
\Bigl[
\Ind{R(x,y)\ge r_{\theta,\alpha}(x)}
\,
\widetilde R_{\eta_{\theta,\alpha}(x)}(R(x,y))
\,
S_\theta(x,y)
\Bigr].
\]
This proves the claim.
\end{proof}

\section{Proofs for finite-sample gradient estimators}
\label{app:gradient_estimator_proofs}
\subsection{Common notation}
\label{app:gradient_estimator_common_notation}

Throughout this section, we fix a prompt \(x\) and a policy parameter
\(\theta\). All expectations are with respect to the conditional rollout
distribution \(y\sim\curpolicy(\cdot\mid x)\) unless otherwise specified. We
write
\[
R(y):=\rewardmodel(x,y),
\qquad
S(y):=S_\theta(x,y)=\nabla_\theta\log\curpolicy(y\mid x).
\]
Let
\[
\eta
:=
\eta_{\theta,\alpha}(x)
=
(r,\mu,\sigma)
:=
\bigl(
r_{\theta,\alpha}(x),
\mu_{\theta,\alpha}(x),
\sigma_{\theta,\alpha}(x)
\bigr)
\]
denote the population upper-\(\alpha\) tail vector. For a generic tail vector
\(\eta'=(r',\mu',\sigma')\), define the shaped score map
\begin{equation}
\label{eq:app_phi_eta_def}
\varphi_{\eta'}(y)
:=
\frac1\alpha
\Ind{R(y)\ge r'}
\,
\widetilde R_{\eta'}(R(y))
\,
S(y).
\end{equation}
Thus, by \Cref{lem:tail_based_policy_gradient},
\[
g_\theta(x)=P\varphi_\eta,
\]
where
\[
Pf:=\mathbb E[f(y)].
\]
For a sample \(y_1,\dots,y_m\), we use
\[
P_m f:=\frac1m\sum_{i=1}^m f(y_i)
\]
for the empirical average.

For \(\eta'\) in a neighborhood of \(\eta\), define the population plug-in map
\begin{equation}
\label{eq:app_H_eta_def}
H(\eta')
:=
P\varphi_{\eta'}
=
\frac1\alpha
\mathbb E
\Big[
\Ind{R(y)\ge r'}
\widetilde R_{\eta'}(R(y))
S(y)
\Big].
\end{equation}
In particular,
\[
H(\eta)=g_\theta(x).
\]

Let \(k_m:=\lceil \alpha m\rceil\). Given rewards \(R_i=R(y_i)\), let
\[
R_{(1)}\le \cdots\le R_{(m)}
\]
be their order statistics, and define
\[
\hat r_m:=R_{(m-k_m+1)}.
\]
Let \(\mathcal I_m\) be the indices of the top \(k_m\) rewards. Define
\[
\hat\mu_m
:=
\frac1{k_m}\sum_{i\in\mathcal I_m}R_i,
\qquad
\hat\sigma_m^{\mathrm{raw}}
:=
\left(
\frac1{k_m}\sum_{i\in\mathcal I_m}(R_i-\hat\mu_m)^2
\right)^{1/2},
\]
and
\[
\hat\sigma_m
:=
\max\{\hat\sigma_m^{\mathrm{raw}},\varepsilon_\sigma\}.
\]
The empirical tail vector is
\[
\hat\eta_m:=(\hat r_m,\hat\mu_m,\hat\sigma_m).
\]

All constants denoted by \(C\) may change from line to line. Unless otherwise
stated, they depend only on \((\alpha,G,\sigma_{\min},M_R)\).
Throughout the finite-sample analysis, the clipping level is
\[
\varepsilon_\sigma
=
\frac12\sqrt{\delta_\alpha}\,\sigma_{\min}.
\]
where \(\delta_\alpha\) is defined in \Cref{lem:VN_tail_surrogate} and depends only on \(\alpha\).

By \Cref{cor:uniform_tail_bounds}, the population upper-tail scale satisfies
\begin{equation}
\label{eq:app_population_tail_bounds}
|r|+|\mu|+\sigma\le C,
\qquad
\sigma\ge 2\varepsilon_\sigma.
\end{equation}
Moreover,
\begin{equation}
\label{eq:app_tail_gap_bound}
r-r_{\theta,2\alpha}(x)\ge c_{\mathrm{gap}}>0,
\end{equation}
where \(c_{\mathrm{gap}}\) depends only on \((\alpha,\sigma_{\min})\). We also
use the latent Gaussian notation
\[
\bar\mu:=\mu_\theta(x),
\qquad
\bar\sigma:=\sigma_\theta(x),
\]
for which \Cref{cor:uniform_tail_bounds} gives
\[
|\bar\mu|+\bar\sigma\le C,
\qquad
\bar\sigma\ge \sigma_{\min}.
\]
Since \(N\ge2\), we will repeatedly use
\begin{equation}
\label{eq:app_ctilde_log_bound}
|\widetilde c_N|
\le
C_\alpha\sqrt{\log N},
\end{equation}
which follows from \(c_N=\mathbb E[\max_{1\le i\le N}Z_i]\le\sqrt{2\log N}\)
and the fact that \(\lambda_\alpha\) and \(\delta_\alpha\) depend only on
\(\alpha\).

\subsection{Empirical upper-tail statistics}
\label{app:empirical_tail_statistics}

This subsection collects the finite-sample properties of the empirical upper-tail statistics. We begin with a uniform order-statistic lemma.

\begin{lemma}
\label{lem:uniform_order_tail_facts}
Let \(U_1,\dots,U_m\stackrel{\mathrm{i.i.d.}}{\sim}\mathrm{Unif}(0,1)\), let
\(k_m:=\lceil \alpha m\rceil\), and define
\[
j_m:=m-k_m+1,
\qquad
B_m:=U_{(j_m)},
\qquad
u_\alpha:=1-\alpha .
\]
For every fixed \(\varepsilon\in(0,\alpha)\), define
\[
J_\varepsilon:=[u_\alpha-\varepsilon,u_\alpha+\varepsilon].
\]
Then there exist constants \(c_\varepsilon,C_\varepsilon>0\), depending only on
\((\alpha,\varepsilon)\), such that
\begin{equation}
\label{eq:app_order_localization}
\mathbb P(B_m\notin J_\varepsilon)
\le
C_\varepsilon e^{-c_\varepsilon m}.
\end{equation}
Moreover,
\begin{equation}
\label{eq:app_beta_moment_bounds}
\bigl|\mathbb E[B_m-u_\alpha]\bigr|
\le
\frac{C}{m},
\qquad
\mathbb E[(B_m-u_\alpha)^2]
\le
\frac{C}{m},
\qquad
\mathbb E[(B_m-u_\alpha)^4]
\le
\frac{C}{m^2},
\end{equation}
where \(C\) depends only on \(\alpha\).

Finally, let \(I\subset(1-2\alpha,1)\) be a compact interval containing
\(u_\alpha\), and let \(f\in C^2(I)\). On the event \(\{B_m\in I\}\),
\begin{equation}
\label{eq:app_smooth_function_of_Bm}
\bigl|
\mathbb E[(f(B_m)-f(u_\alpha))\Ind{B_m\in I}]
\bigr|
+
\mathbb E\!\left[
(f(B_m)-f(u_\alpha))^2\Ind{B_m\in I}
\right]
\le
\frac{C_f}{m},
\end{equation}
where \(C_f\) depends only on \(\alpha\) and
\(\sup_{u\in I}(|f'(u)|+|f''(u)|)\).
\end{lemma}

\begin{proof}
We first prove the localization bound. Let
\[
a:=u_\alpha-\varepsilon,
\qquad
b:=u_\alpha+\varepsilon.
\]
For the lower endpoint, define
\[
N_a:=\sum_{i=1}^m\Ind{U_i\le a}\sim \mathrm{Bin}(m,a).
\]
Since \(k_m=\lceil \alpha m\rceil\), we have
\[
j_m=m-k_m+1\ge (1-\alpha)m=u_\alpha m.
\]
The event \(B_m<a\) implies \(N_a\ge j_m\). Hence
\[
N_a-ma
\ge
u_\alpha m-(u_\alpha-\varepsilon)m
=
\varepsilon m.
\]
By Hoeffding's inequality,
\[
\mathbb P(B_m<a)
\le
\mathbb P(N_a-ma\ge \varepsilon m)
\le
\exp(-2\varepsilon^2m).
\]

For the upper endpoint, define
\[
N_b:=\sum_{i=1}^m\Ind{U_i\le b}\sim \mathrm{Bin}(m,b).
\]
The event \(B_m>b\) implies \(N_b\le j_m-1=m-k_m\). Since
\(m-k_m\le (1-\alpha)m=u_\alpha m\),
\[
mb-N_b
\ge
(u_\alpha+\varepsilon)m-u_\alpha m
=
\varepsilon m.
\]
Thus
\[
\mathbb P(B_m>b)
\le
\mathbb P(mb-N_b\ge \varepsilon m)
\le
\exp(-2\varepsilon^2m).
\]
Combining the two bounds proves \eqref{eq:app_order_localization}.

Next, since \(B_m=U_{(j_m)}\), we have
\[
B_m\sim \mathrm{Beta}(j_m,m-j_m+1)
=
\mathrm{Beta}(m-k_m+1,k_m).
\]
Therefore,
\[
\mathbb E[B_m]=\frac{j_m}{m+1},
\qquad
\mathrm{Var}(B_m)=
\frac{j_m(m-j_m+1)}{(m+1)^2(m+2)}.
\]
Writing
\[
k_m=\alpha m+\delta_m,
\qquad
\delta_m:=\lceil \alpha m\rceil-\alpha m\in[0,1),
\]
we have
\[
j_m=(1-\alpha)m+1-\delta_m.
\]
Hence
\[
\mathbb E[B_m]-u_\alpha
=
\frac{j_m}{m+1}-(1-\alpha)
=
\frac{\alpha-\delta_m}{m+1},
\]
so
\[
\bigl|\mathbb E[B_m-u_\alpha]\bigr|\le \frac{C}{m}.
\]
Also,
\[
\mathrm{Var}(B_m)
\le
\frac{1}{4(m+2)}.
\]
This proves the first two bounds in \eqref{eq:app_beta_moment_bounds}. The
fourth-moment bound follows from the standard fourth central moment formula for
the beta distribution:
\[
\mathbb E\!\left[(B_m-\mathbb E B_m)^4\right]\le \frac{C}{m^2}.
\]
Together with
\[
|\mathbb E B_m-u_\alpha|^4\le \frac{C}{m^4},
\]
this gives
\[
\mathbb E[(B_m-u_\alpha)^4]\le \frac{C}{m^2}.
\]

It remains to prove \eqref{eq:app_smooth_function_of_Bm}. On the event
\(B_m\in I\), Taylor's theorem gives
\[
f(B_m)-f(u_\alpha)
=
f'(u_\alpha)(B_m-u_\alpha)
+
\frac12 f''(\xi_m)(B_m-u_\alpha)^2
\]
for some \(\xi_m\) between \(B_m\) and \(u_\alpha\). Hence
\[
\bigl|
\mathbb E[(f(B_m)-f(u_\alpha))\Ind{B_m\in I}]
\bigr|
\le
C_f
\left(
|\mathbb E[B_m-u_\alpha]|
+
\mathbb E[(B_m-u_\alpha)^2]
\right)
\le
\frac{C_f}{m}.
\]
For the mean-squared error, the mean-value theorem gives
\[
|f(B_m)-f(u_\alpha)|^2\Ind{B_m\in I}
\le
C_f(B_m-u_\alpha)^2.
\]
Taking expectations and using \eqref{eq:app_beta_moment_bounds} completes the
proof.
\end{proof}

The next lemma gives a moment envelope for the empirical tail statistics.

\begin{lemma}
\label{lem:empirical_tail_moment_envelope}
Suppose Assumptions~\ref{assum:gaussian-tail-model} and
\ref{assum:nondegenerate_scale} hold. Then, for all sufficiently large \(m\),
\begin{equation}
\label{eq:app_tail_vector_envelope_strong}
\mathbb E\!\left[
\hat\sigma_m^{-4}
+
\hat\sigma_m^4
+
|\hat r_m|^8
+
|\hat\mu_m|^8
\right]
\le C,
\end{equation}
where \(C\) depends only on \((\alpha,M_R,\sigma_{\min})\).
\end{lemma}

\begin{proof}
The inverse-scale bound follows immediately from clipping:
\[
\hat\sigma_m^{-4}\le \varepsilon_\sigma^{-4}.
\]

We next bound the eighth moments of \(\hat r_m\) and \(\hat\mu_m\). By
\Cref{cor:uniform_tail_bounds}, the Gaussian upper tail and the uniform bound
on the latent location and scale imply that there exist constants \(C,c>0\)
such that, for all sufficiently large \(t\),
\begin{equation}
\label{eq:app_reward_positive_tail}
\mathbb P(R\ge t)\le C e^{-ct^2},
\end{equation}
uniformly over \(x\) and \(\theta\).

Since \(\hat r_m\) is the smallest of the top \(k_m\) rewards,
\[
\{\hat r_m\ge t\}
\subseteq
\left\{
\sum_{i=1}^m\Ind{R_i\ge t}\ge k_m
\right\}.
\]
Using Markov's inequality and \(k_m\ge \alpha m\),
\[
\mathbb P(\hat r_m\ge t)
\le
\frac{m\,\mathbb P(R\ge t)}{k_m}
\le
\alpha^{-1}\mathbb P(R\ge t)
\le
C e^{-ct^2}.
\]
For the negative tail, define
\[
N_t^-:=\sum_{i=1}^m \Ind{R_i\le -t},
\qquad
p_t^-:=\mathbb P(R\le -t).
\]
Then \(N_t^-\sim \mathrm{Bin}(m,p_t^-)\), and
\[
\{\hat r_m\le -t\}
=
\{N_t^-\ge m-k_m+1\}.
\]
By the fourth-moment bound,
\[
p_t^-\le \frac{M_R}{t^4}.
\]
Let \(\ell_m:=m-k_m+1\). Since \(k_m=\lceil \alpha m\rceil\) and
\(\alpha<1/2\), we have \(\ell_m\ge (1-\alpha)m\), and in particular
\(\ell_m\ge 3\) for all sufficiently large \(m\). Using the moment property of a binomial random variable,
\[
\mathbb P(N_t^-\ge \ell_m)
\le
\frac{\mathbb E[(N_t^-)_3]}{(\ell_m)_3}
=
\frac{(m)_3(p_t^-)^3}{(\ell_m)_3}
\le
C_\alpha (p_t^-)^3
\le
C t^{-12}.
\]
where $(x)_3=x(x-1)(x-2)$ is the falling factorial.
So for all sufficiently large \(t\) and \(m\),
\[
\mathbb P(\hat r_m\le -t)\le C t^{-12}.
\]
Therefore, 
\[
\mathbb E|\hat r_m|^8
=
8\int_0^\infty t^7\mathbb P(|\hat r_m|\ge t)\,dt
\le C.
\]

Since \(\hat\mu_m\ge \hat r_m\), the negative tail of \(\hat\mu_m\) is bounded by
the same negative-tail bound. For the positive tail, for any \(u>0\),
\[
\hat\mu_m
\le
u+\frac1{k_m}\sum_{i=1}^m(R_i-u)_+.
\]
Taking \(u=t/2\), the event \(\{\hat\mu_m\ge t\}\) implies
\[
\sum_{i=1}^m (R_i-t/2)_+ \ge k_m t/2.
\]
Thus Markov's inequality gives
\[
\mathbb P(\hat\mu_m\ge t)
\le
\frac{2}{k_m t}
\sum_{i=1}^m\mathbb E[(R_i-t/2)_+]
=
\frac{2m}{k_m t}\mathbb E[(R-t/2)_+].
\]
By the layer-cake representation,
\[
\mathbb E[(R-t/2)_+]
=
\int_{t/2}^{\infty}\mathbb P(R\ge u)\,du .
\]
For \(t\) sufficiently large, \eqref{eq:app_reward_positive_tail} gives
\[
\mathbb E[(R-t/2)_+]
\le
C\int_{t/2}^{\infty}e^{-cu^2}\,du
\le
C e^{-c t^2}.
\]

Hence
\[
\mathbb P(\hat\mu_m\ge t)\le Ct^{-1}e^{-ct^2}.
\]
Again by the layer-cake formula,
\[
\mathbb E|\hat\mu_m|^8\le C.
\]

It remains to bound \(\mathbb E[\hat\sigma_m^4]\). Let
\[
\hat M_{2,m}:=\frac1{k_m}\sum_{i\in\mathcal I_m}R_i^2 .
\]
Since the empirical variance is bounded by the empirical second moment,
\[
(\hat\sigma_m^{\mathrm{raw}})^2\le \hat M_{2,m}.
\]
Moreover, by clipping,
\[
\hat\sigma_m^4
\le
C\bigl(1+\hat M_{2,m}^2\bigr),
\]
where the constant depends on \(\varepsilon_\sigma\), hence only on
\((\alpha,\sigma_{\min})\). Using \(k_m\ge\alpha m\),
\[
\hat M_{2,m}
\le
\frac1{k_m}\sum_{i=1}^m R_i^2
\le
\frac1{\alpha m}\sum_{i=1}^m R_i^2.
\]
Therefore,
\[
\mathbb E[\hat M_{2,m}^2]
\le
\frac{1}{\alpha^2m^2}
\mathbb E\left[\left(\sum_{i=1}^m R_i^2\right)^2\right]
\le
C,
\]
where the last step uses \(\mathbb E[R^4]\le M_R\). Hence
\[
\mathbb E[\hat\sigma_m^4]\le C.
\]
This proves \eqref{eq:app_tail_vector_envelope_strong}.
\end{proof}

The next lemma gives the \(1/m\) bias and mean-squared error rates for the
empirical tail vector.

\begin{lemma}
\label{lem:conditional_tail_delta_method}
Suppose Assumptions~\ref{assum:gaussian-tail-model} and
\ref{assum:nondegenerate_scale} hold. Then, for all sufficiently large \(m\),
\begin{equation}
\label{eq:app_tail_vector_bias_mse_delta}
\bigl\|
\mathbb E[\hat\eta_m]-\eta
\bigr\|
\le
\frac{C}{m},
\qquad
\mathbb E\|\hat\eta_m-\eta\|^2
\le
\frac{C}{m},
\end{equation}
where \(C\) depends only on \((\alpha,M_R,\sigma_{\min})\).
\end{lemma}

\begin{proof}
Let \(F=F_{\theta,x}\) and \(Q=F^{-1}\). By the probability integral transform,
\(U_i:=F(R_i)\) are i.i.d. uniform random variables, and
\(R_{(j)}=Q(U_{(j)})\). Let
\[
u_\alpha:=1-\alpha,
\qquad
B_m:=U_{(m-k_m+1)}.
\]
Fix
\[
I_\alpha
:=
\left[
1-\frac32\alpha,\,
1-\frac12\alpha
\right]
\subset(1-2\alpha,1).
\]
By \Cref{lem:uniform_order_tail_facts},
\[
\mathbb P(B_m\notin I_\alpha)\le Ce^{-cm}.
\]
On \(I_\alpha\), the Gaussian-tail model gives
\[
Q(u)=\bar\mu+\bar\sigma\Phi^{-1}(u),
\]
where \(\bar\mu:=\mu_\theta(x)\) and \(\bar\sigma:=\sigma_\theta(x)\). Since
\(I_\alpha\) is compactly contained in \((1-2\alpha,1)\), and since
\Cref{cor:uniform_tail_bounds} gives uniform bounds on \(\bar\mu\) and
\(\bar\sigma\), the functions \(Q,Q'\), and \(Q''\) are uniformly bounded on
\(I_\alpha\).

For \(b\in I_\alpha\), define
\[
Y_b:=Q(b+(1-b)W),
\qquad
W\sim\mathrm{Unif}(0,1),
\]
and
\[
M_1(b):=\mathbb E[Y_b]
=
\frac1{1-b}\int_b^1 Q(u)\,du,
\qquad
M_2(b):=\mathbb E[Y_b^2]
=
\frac1{1-b}\int_b^1 Q(u)^2\,du.
\]
Then \(M_1,M_2\in C^2(I_\alpha)\) with uniformly bounded first two derivatives,
and
\[
M_1(u_\alpha)=\mu,
\qquad
M_2(u_\alpha)-M_1(u_\alpha)^2=\sigma^2.
\]
Moreover, uniformly over \(b\in I_\alpha\), by fourth-moment bounds on \(R\) and \Cref{cor:uniform_tail_bounds},
\begin{equation}
\label{eq:app_conditional_tail_moment_bounds}
\mathbb E[|Y_b|^4]\le C,
\qquad
|Q(b)|+|M_1(b)|+|M_2(b)|\le C.
\end{equation}

Conditioned on \(B_m=b\), the empirical upper tail consists of the threshold
value \(Q(b)\) and \(k_m-1\) unordered i.i.d. copies of \(Y_b\). Therefore,
\[
\mathbb E[\hat\mu_m\mid B_m=b]
=
\frac{Q(b)+(k_m-1)M_1(b)}{k_m},
\]
and
\begin{equation}
\label{eq:app_conditional_mean_bias}
\left|
\mathbb E[\hat\mu_m\mid B_m=b]-M_1(b)
\right|
\le
\frac{C}{m}.
\end{equation}
Also,
\begin{equation}
\label{eq:app_conditional_mean_variance}
\mathrm{Var}(\hat\mu_m\mid B_m=b)
\le
\frac{C}{m}.
\end{equation}

Similarly, define
\[
\hat M_{2,m}:=\frac1{k_m}\sum_{i\in\mathcal I_m}R_i^2 .
\]
Then
\[
\mathbb E[\hat M_{2,m}\mid B_m=b]
=
\frac{Q(b)^2+(k_m-1)M_2(b)}{k_m},
\]
and hence
\begin{equation}
\label{eq:app_conditional_second_moment_bias}
\left|
\mathbb E[\hat M_{2,m}\mid B_m=b]-M_2(b)
\right|
\le
\frac{C}{m}.
\end{equation}
Moreover,
\begin{equation}
\label{eq:app_conditional_second_moment_variance}
\mathrm{Var}(\hat M_{2,m}\mid B_m=b)
\le
\frac{C}{m},
\end{equation}
where we used the uniform fourth-moment bound in
\eqref{eq:app_conditional_tail_moment_bounds}.

We now control each component of \(\hat\eta_m\).

First, since \(\hat r_m=Q(B_m)\) and \(r=Q(u_\alpha)\),
\Cref{lem:uniform_order_tail_facts} applied to \(f=Q\) gives
\[
\bigl|
\mathbb E[(\hat r_m-r)\Ind{B_m\in I_\alpha}]
\bigr|
+
\mathbb E[(\hat r_m-r)^2\Ind{B_m\in I_\alpha}]
\le
\frac{C}{m}.
\]
On the complement, by Cauchy--Schwarz,
\[
\mathbb E[(\hat r_m-r)^2\Ind{B_m\notin I_\alpha}]
\le
\Bigl(\mathbb E[(\hat r_m-r)^4]\Bigr)^{1/2}
\mathbb P(B_m\notin I_\alpha)^{1/2}.
\]
By \Cref{lem:empirical_tail_moment_envelope} and
\Cref{cor:uniform_tail_bounds}, the fourth moment
\(\mathbb E[(\hat r_m-r)^4]\) is uniformly bounded, while
\(\mathbb P(B_m\notin I_\alpha)\le Ce^{-cm}\) by
\eqref{eq:app_order_localization}. Therefore
\[
\mathbb E[(\hat r_m-r)^2\Ind{B_m\notin I_\alpha}]
\le
Ce^{-cm}.
\]
The same argument gives
\[
\bigl|
\mathbb E[(\hat r_m-r)\Ind{B_m\notin I_\alpha}]
\bigr|
\le
Ce^{-cm}.
\]
Combining the localized and complement bounds,
\[
\bigl|
\mathbb E[\hat r_m]-r
\bigr|
+
\mathbb E[(\hat r_m-r)^2]
\le
\frac{C}{m}.
\]

Next, write
\[
\hat\mu_m-\mu
=
A_m+D_m,
\qquad
A_m:=\hat\mu_m-M_1(B_m),
\qquad
D_m:=M_1(B_m)-M_1(u_\alpha).
\]
By the tower property, \eqref{eq:app_conditional_mean_bias}, and
\eqref{eq:app_conditional_mean_variance},
\[
|\mathbb E[A_m]|+\mathbb E[A_m^2]\le \frac{C}{m}.
\]
By \Cref{lem:uniform_order_tail_facts} applied to \(f=M_1\),
\[
|\mathbb E[D_m]|+\mathbb E[D_m^2]\le \frac{C}{m}.
\]
Therefore,
\begin{equation}
\label{eq:app_muhat_bias_mse}
\bigl|\mathbb E[\hat\mu_m]-\mu\bigr|
+
\mathbb E[(\hat\mu_m-\mu)^2]
\le
\frac{C}{m}.
\end{equation}

We will also use the following fourth-moment bound:
\begin{equation}
\label{eq:app_muhat_fourth_moment}
\mathbb E[(\hat\mu_m-\mu)^4]\le \frac{C}{m^2}.
\end{equation}
To prove it, recall the decomposition
\[
\hat\mu_m-\mu
=
A_m+D_m,
\qquad
A_m:=\hat\mu_m-M_1(B_m),
\qquad
D_m:=M_1(B_m)-M_1(u_\alpha).
\]
By \((a+b)^4\le 8(a^4+b^4)\), it suffices to bound the fourth moments of
\(A_m\) and \(D_m\).

First, since \(M_1\) has uniformly bounded derivative on \(I_\alpha\),
\[
|D_m|^4\Ind{B_m\in I_\alpha}
\le
C|B_m-u_\alpha|^4 .
\]
Using \eqref{eq:app_beta_moment_bounds}, we get
\[
\mathbb E\bigl[D_m^4\Ind{B_m\in I_\alpha}\bigr]
\le
\frac{C}{m^2}.
\]

Next condition on \(B_m=b\in I_\alpha\). Let
\[
Y_{\ell,b}:=Q\bigl(b+(1-b)W_\ell\bigr),
\qquad
W_\ell\stackrel{\mathrm{i.i.d.}}{\sim}\mathrm{Unif}(0,1),
\qquad
\ell=1,\dots,k_m-1.
\]
Then
\[
\hat\mu_m
=
\frac{Q(b)+\sum_{\ell=1}^{k_m-1}Y_{\ell,b}}{k_m},
\]
and hence
\[
A_m
=
\frac{Q(b)-M_1(b)}{k_m}
+
\frac1{k_m}
\sum_{\ell=1}^{k_m-1}
\bigl(Y_{\ell,b}-M_1(b)\bigr).
\]
The first term is \(O(m^{-1})\) uniformly in \(b\in I_\alpha\). For the second
term, the variables
\[
Z_{\ell,b}:=Y_{\ell,b}-M_1(b)
\]
are conditionally i.i.d., mean-zero, and satisfy
\[
\sup_{b\in I_\alpha}\mathbb E[Z_{\ell,b}^4]\le C
\]
by \eqref{eq:app_conditional_tail_moment_bounds}. Therefore,
\[
\mathbb E\!\left[
\left(\sum_{\ell=1}^{k_m-1}Z_{\ell,b}\right)^4
\,\middle|\,
B_m=b
\right]
=
(k_m-1)\mathbb E[Z_{1,b}^4]
+
3(k_m-1)(k_m-2)\bigl(\mathbb E[Z_{1,b}^2]\bigr)^2
\le
Ck_m^2.
\]
Consequently,
\[
\mathbb E[A_m^4\mid B_m=b]\le \frac{C}{m^2},
\qquad b\in I_\alpha.
\]
Integrating over \(B_m\) gives
\[
\mathbb E\bigl[A_m^4\Ind{B_m\in I_\alpha}\bigr]
\le
\frac{C}{m^2}.
\]

Finally, the complement \(B_m\notin I_\alpha\) has exponentially small
probability by \eqref{eq:app_order_localization}; its contribution is
negligible by Cauchy--Schwarz and the moment envelope in
\Cref{lem:empirical_tail_moment_envelope}. Combining the bounds for \(A_m\)
and \(D_m\) proves \eqref{eq:app_muhat_fourth_moment}.

For the second moment, decompose
\[
\hat M_{2,m}-M_2(u_\alpha)
=
C_m+E_m,
\qquad
C_m:=\hat M_{2,m}-M_2(B_m),
\qquad
E_m:=M_2(B_m)-M_2(u_\alpha).
\]
Using \eqref{eq:app_conditional_second_moment_bias},
\eqref{eq:app_conditional_second_moment_variance}, and
\Cref{lem:uniform_order_tail_facts} applied to \(f=M_2\), we obtain
\begin{equation}
\label{eq:app_second_moment_bias_mse}
\left|
\mathbb E[\hat M_{2,m}]-M_2(u_\alpha)
\right|
+
\mathbb E[(\hat M_{2,m}-M_2(u_\alpha))^2]
\le
\frac{C}{m}.
\end{equation}

Let
\[
\hat v_m
:=
(\hat\sigma_m^{\mathrm{raw}})^2
=
\hat M_{2,m}-\hat\mu_m^2,
\qquad
v_\alpha
:=
\sigma^2
=
M_2(u_\alpha)-\mu^2.
\]
Define the second-moment and mean errors
\[
e_{2,m}:=\hat M_{2,m}-M_2(u_\alpha),
\qquad
e_{\mu,m}:=\hat\mu_m-\mu.
\]
Then, by the identity \(a^2-b^2=2b(a-b)+(a-b)^2\),
\begin{align}
\hat v_m-v_\alpha
&=
\bigl(\hat M_{2,m}-\hat\mu_m^2\bigr)
-
\bigl(M_2(u_\alpha)-\mu^2\bigr) \notag\\
&=
e_{2,m}
-
\bigl(\hat\mu_m^2-\mu^2\bigr) \notag\\
&=
e_{2,m}
-
2\mu e_{\mu,m}
-
e_{\mu,m}^2.
\label{eq:app_variance_error_identity}
\end{align}
Using \(|\mu|\le C\), together with
\eqref{eq:app_muhat_bias_mse},
\eqref{eq:app_muhat_fourth_moment}, and
\eqref{eq:app_second_moment_bias_mse}, this identity gives
\[
\left|
\mathbb E[\hat v_m]-v_\alpha
\right|
+
\mathbb E[(\hat v_m-v_\alpha)^2]
\le
\frac{C}{m}.
\]
Indeed, for the bias,
\[
\left|\mathbb E[\hat v_m-v_\alpha]\right|
\le
|\mathbb E[e_{2,m}]|
+
2|\mu|\,|\mathbb E[e_{\mu,m}]|
+
\mathbb E[e_{\mu,m}^2]
\le
\frac{C}{m},
\]
and for the mean-squared error,
\[
\mathbb E[(\hat v_m-v_\alpha)^2]
\le
C\Big(
\mathbb E[e_{2,m}^2]
+
\mathbb E[e_{\mu,m}^2]
+
\mathbb E[e_{\mu,m}^4]
\Big)
\le
\frac{C}{m}.
\]

Finally, define the clipped square-root map
\[
\psi(v):=\max\{\sqrt v,\varepsilon_\sigma\},
\qquad v\ge0.
\]
Then
\[
\hat\sigma_m=\psi(\hat v_m),
\qquad
\sigma=\psi(v_\alpha),
\]
because \(v_\alpha=\sigma^2\) and \(\sigma\ge 2\varepsilon_\sigma\).
Since \(v_\alpha\ge 4\varepsilon_\sigma^2\), the map \(\psi\) satisfies
\[
|\psi(v)-\psi(v_\alpha)|
\le
C|v-v_\alpha|,
\qquad
|\psi(v)-\psi(v_\alpha)-\psi'(v_\alpha)(v-v_\alpha)|
\le
C|v-v_\alpha|^2,
\qquad v\ge0,
\]
where \(C\) depends only on \((\alpha,G,\sigma_{\min},M_R)\). Using the bound just proved
above,
\[
\bigl|\mathbb E[\hat v_m-v_\alpha]\bigr|
+
\mathbb E[(\hat v_m-v_\alpha)^2]
\le
\frac{C}{m},
\]
we obtain
\[
\mathbb E[(\hat\sigma_m-\sigma)^2]
\le
C\mathbb E[(\hat v_m-v_\alpha)^2]
\le
\frac{C}{m}.
\]
For the bias, write
\[
\hat\sigma_m-\sigma
=
\psi'(v_\alpha)(\hat v_m-v_\alpha)
+
\mathcal R_m,
\qquad
|\mathcal R_m|\le C|\hat v_m-v_\alpha|^2.
\]
Therefore,
\[
\bigl|\mathbb E[\hat\sigma_m]-\sigma\bigr|
\le
C\bigl|\mathbb E[\hat v_m-v_\alpha]\bigr|
+
C\mathbb E[(\hat v_m-v_\alpha)^2]
\le
\frac{C}{m}.
\]
Combining the last two displays gives
\[
\bigl|\mathbb E[\hat\sigma_m]-\sigma\bigr|
+
\mathbb E[(\hat\sigma_m-\sigma)^2]
\le
\frac{C}{m}.
\]
Combining the three component bounds proves
\eqref{eq:app_tail_vector_bias_mse_delta}.
\end{proof}

We now package the preceding results into a single empirical-tail
interface that will be used in the gradient-estimator analysis.

\begin{corollary}
\label{lem:empirical_tail_vector_bounds}
Suppose Assumptions~\ref{assum:gaussian-tail-model} and
\ref{assum:nondegenerate_scale} hold. Then there exist constants
\(\delta_{\mathrm{tail}}>0\), \(c_{\mathrm{tail}}>0\), and
\(C_{\mathrm{tail}}>0\), depending only on
\((\alpha,M_R,\sigma_{\min})\), such that for all sufficiently large \(m\),
\begin{equation}
\label{eq:app_tail_vector_good_event}
\mathbb P\bigl(\hat\eta_m\notin B(\eta,\delta_{\mathrm{tail}})\bigr)
\le
C_{\mathrm{tail}}e^{-c_{\mathrm{tail}}m}.
\end{equation}
Moreover,
\begin{equation}
\label{eq:app_tail_vector_bias_mse}
\bigl\|
\mathbb E[\hat\eta_m]-\eta
\bigr\|
\le
\frac{C_{\mathrm{tail}}}{m},
\qquad
\mathbb E\|\hat\eta_m-\eta\|^2
\le
\frac{C_{\mathrm{tail}}}{m},
\end{equation}
and
\begin{equation}
\label{eq:app_tail_vector_envelope}
\mathbb E\!\left[
\hat\sigma_m^{-4}
+
\hat\sigma_m^4
+
|\hat r_m|^8
+
|\hat\mu_m|^8
\right]
\le
C_{\mathrm{tail}}.
\end{equation}
\end{corollary}

\begin{proof}
The bias and mean-squared error bounds in
\eqref{eq:app_tail_vector_bias_mse} are exactly
\Cref{lem:conditional_tail_delta_method}. The moment envelope
\eqref{eq:app_tail_vector_envelope} follows from
\Cref{lem:empirical_tail_moment_envelope}. It remains to prove the
exponential localization bound \eqref{eq:app_tail_vector_good_event}.

Let \(F=F_{\theta,x}\), \(Q=F^{-1}\), and \(U_i:=F(R_i)\). By the probability
integral transform, \(U_i\stackrel{\mathrm{i.i.d.}}{\sim}\mathrm{Unif}(0,1)\)
and \(R_{(j)}=Q(U_{(j)})\). Set
\[
u_\alpha:=1-\alpha,
\qquad
B_m:=U_{(m-k_m+1)}.
\]

For \(b\) in a neighborhood of \(u_\alpha\), define
\[
M_1(b):=\frac1{1-b}\int_b^1 Q(u)\,du,
\qquad
M_2(b):=\frac1{1-b}\int_b^1 Q(u)^2\,du.
\]
Also define
\[
\Psi(m_1,m_2)
:=
\left(
m_1,
\max\{\sqrt{m_2-m_1^2},\varepsilon_\sigma\}
\right).
\]
Since
\[
M_2(u_\alpha)-M_1(u_\alpha)^2
=
\sigma^2
\ge
4\varepsilon_\sigma^2,
\]
there exists a fixed radius \(\delta_{\mathrm{tail}}>0\), depending only on
\((\alpha,G,\sigma_{\min},M_R)\), such that on the \(2\delta_{\mathrm{tail}}\)-neighborhood
of \((\mu,M_2(u_\alpha))\), the quantity \(m_2-m_1^2\) is bounded below by
\(\varepsilon_\sigma^2\), and the map \(\Psi\) is Lipschitz with a constant
depending only on \((\alpha,M_R,\sigma_{\min})\).

For \(b\) in this neighborhood, define the local population tail vector
associated with threshold \(Q(b)\) by
\[
\eta(b)
:=
\left(
Q(b),
M_1(b),
\max\left\{
\sqrt{M_2(b)-M_1(b)^2},
\varepsilon_\sigma
\right\}
\right).
\]
By the uniform smoothness of \(Q,M_1,M_2\) on compact subintervals of
\((1-2\alpha,1)\), the map \(b\mapsto\eta(b)\) is uniformly continuous near
\(u_\alpha\), and \(\eta(u_\alpha)=\eta\). Hence there exists
\(\varepsilon_{\mathrm{tail}}\in(0,\alpha)\), depending only on
\((\alpha,G,\sigma_{\min},M_R)\), such that
\[
J_{\mathrm{tail}}
:=
[u_\alpha-\varepsilon_{\mathrm{tail}},
 u_\alpha+\varepsilon_{\mathrm{tail}}]
\subset
(1-2\alpha,1)
\]
and
\begin{equation}
\label{eq:app_population_tail_local_continuity}
\sup_{b\in J_{\mathrm{tail}}}
\|\eta(b)-\eta\|
\le
\frac{\delta_{\mathrm{tail}}}{3}.
\end{equation}

By \Cref{lem:uniform_order_tail_facts},
\begin{equation}
\label{eq:app_Bm_tail_localization}
\mathbb P(B_m\notin J_{\mathrm{tail}})
\le
Ce^{-cm}.
\end{equation}
Conditional on \(B_m=b\in J_{\mathrm{tail}}\), the unordered upper-tail
observations excluding the threshold have the same distribution as
\(Y_{b,1},\dots,Y_{b,k_m-1}\), where
\[
Y_{b,\ell}:=Q(b+(1-b)W_\ell),
\qquad
W_\ell\stackrel{\mathrm{i.i.d.}}{\sim}\mathrm{Unif}(0,1).
\]
Thus
\[
\hat\mu_m-M_1(b)
=
\frac1{k_m}\sum_{\ell=1}^{k_m-1}
\bigl(Y_{b,\ell}-M_1(b)\bigr)
+
\frac{Q(b)-M_1(b)}{k_m},
\]
and
\[
\hat M_{2,m}-M_2(b)
=
\frac1{k_m}\sum_{\ell=1}^{k_m-1}
\bigl(Y_{b,\ell}^2-M_2(b)\bigr)
+
\frac{Q(b)^2-M_2(b)}{k_m},
\]
where
\[
\hat M_{2,m}:=\frac1{k_m}\sum_{i\in\mathcal I_m}R_i^2.
\]
The deterministic threshold terms are uniformly \(O(k_m^{-1})\), since
\(|Q(b)|+|M_1(b)|+|M_2(b)|\le C\) on \(J_{\mathrm{tail}}\).
The family \(\{Y_b:b\in J_{\mathrm{tail}}\}\) has uniform Gaussian-tail
envelopes. Thus \(Y_b-\mathbb E Y_b\) is uniformly sub-Gaussian and
\(Y_b^2-\mathbb E Y_b^2\) is uniformly sub-exponential. Standard Bernstein
inequalities give
\begin{equation}
\label{eq:app_tail_empirical_local_concentration}
\mathbb P\!\left(
|\hat\mu_m-M_1(b)|
+
|\hat M_{2,m}-M_2(b)|
>
c_0\delta_{\mathrm{tail}}
\;\middle|\;
B_m=b
\right)
\le
Ce^{-cm},
\end{equation}
uniformly over \(b\in J_{\mathrm{tail}}\). Here, the deterministic \(O(k_m^{-1})\) contribution of the threshold sample
\(Q(b)\) and \(Q(b)^2\) is absorbed into \(c_0\delta_{\mathrm{tail}}\) for all
sufficiently large \(m\).

On the concentration event in
\eqref{eq:app_tail_empirical_local_concentration}, the threshold coordinate
has no error relative to \(\eta(B_m)\), since
\[
\hat r_m=Q(B_m).
\]
The mean coordinate is controlled by
\(|\hat\mu_m-M_1(B_m)|\). The scale coordinate is controlled by the Lipschitz
property of \(\Psi\), because
\[
\hat\sigma_m
=
\max\left\{
\sqrt{\hat M_{2,m}-\hat\mu_m^2},
\varepsilon_\sigma
\right\}.
\]
Therefore, after choosing \(c_0>0\) small enough,
\begin{equation}
\label{eq:app_empirical_tail_close_to_local_population}
\|\hat\eta_m-\eta(B_m)\|
\le
\frac{2\delta_{\mathrm{tail}}}{3}.
\end{equation}
Combining this with
\eqref{eq:app_population_tail_local_continuity}, we obtain
\[
\|\hat\eta_m-\eta\|
\le
\delta_{\mathrm{tail}}
\]
whenever \(B_m\in J_{\mathrm{tail}}\) and the empirical concentration event
holds. Consequently,
\[
\mathbb P(\|\hat\eta_m-\eta\|>\delta_{\mathrm{tail}})
\le
\mathbb P(B_m\notin J_{\mathrm{tail}})
+
Ce^{-cm}
\le
C_{\mathrm{tail}}e^{-c_{\mathrm{tail}}m}.
\]
This proves \eqref{eq:app_tail_vector_good_event}, and hence the result.
\end{proof}

\subsection{Regularity of the shaped score map}
\label{app:shaped_score_regular}

The next lemma controls the local behavior of the shaped score map
\(\eta'\mapsto \varphi_{\eta'}\) and its population average \(H(\eta')\). These
are the regularity properties used in the direct plug-in estimator analysis.

\begin{lemma}
\label{lem:shaped_score_local_regularity}
Suppose Assumptions~\ref{assum:gaussian-tail-model},
\ref{assum:bounded_score}, and \ref{assum:nondegenerate_scale} hold. Then there
exist constants \(\delta_{\mathrm{reg}}>0\) and \(C_{\mathrm{reg}}>0\),
depending only on \((\alpha,G,\sigma_{\min},M_R)\), such that the following
hold.

First,
\begin{equation}
\label{eq:app_phi_oracle_second_moment}
P\|\varphi_\eta\|^2
\le
C_{\mathrm{reg}}\log N.
\end{equation}
Second, for every
\(\eta_1,\eta_2\in B(\eta,\delta_{\mathrm{reg}})\),
\begin{equation}
\label{eq:app_phi_l2_lipschitz}
P\|\varphi_{\eta_1}-\varphi_{\eta_2}\|^2
\le
C_{\mathrm{reg}}\log N\,\|\eta_1-\eta_2\|^2.
\end{equation}
Third, there exists a linear map
\[
\dot H_\eta:\mathbb R^3\to\mathbb R^{d_\theta}
\]
where \(d_\theta\) is the dimension of the policy parameter vector \(\theta\),
such that
\begin{equation}
\label{eq:app_H_derivative_bound}
\|\dot H_\eta\|
\le
C_{\mathrm{reg}}\sqrt{\log N},
\end{equation}
and for every \(\eta'\in B(\eta,\delta_{\mathrm{reg}})\),
\begin{equation}
\label{eq:app_H_second_order_remainder}
\bigl\|
H(\eta')-H(\eta)-\dot H_\eta(\eta'-\eta)
\bigr\|
\le
C_{\mathrm{reg}}\sqrt{\log N}\,\|\eta'-\eta\|^2.
\end{equation}
Consequently,
\begin{equation}
\label{eq:app_H_lipschitz}
\|H(\eta')-H(\eta)\|
\le
C_{\mathrm{reg}}\sqrt{\log N}\,\|\eta'-\eta\|.
\end{equation}
\end{lemma}

\begin{proof}
By \eqref{eq:app_population_tail_bounds}, the population tail vector satisfies
\[
|r|+|\mu|+\sigma\le C,
\qquad
\sigma\ge 2\varepsilon_\sigma.
\]
Choose \(\delta_{\mathrm{reg}}>0\) small enough that every
\(\eta'=(r',\mu',\sigma')\in B(\eta,\delta_{\mathrm{reg}})\) satisfies
\begin{equation}
\label{eq:app_eta_prime_local_bounds}
\sigma'\ge \varepsilon_\sigma,
\qquad
|r'|+|\mu'|\le C,
\qquad
r'\ge r_{\theta,2\alpha}(x)+\frac12 c_{\mathrm{gap}},
\end{equation}
where \(c_{\mathrm{gap}}\) is defined in \eqref{eq:app_tail_gap_bound}. Thus all
thresholds in the local ball remain inside the modeled Gaussian tail. We also
write
\[
I_{\mathrm{loc}}
:=
[r-\delta_{\mathrm{reg}},\,r+\delta_{\mathrm{reg}}],
\]
which is a compact interval contained in the modeled Gaussian tail region.

We first calculate local envelope bounds for the shaped reward. Recall that
\[
\widetilde R_{\eta'}(u)
=
(u-r')
+
\frac{\widetilde c_N}{2\sigma'}
\Bigl((u-\mu')^2-(r'-\mu')^2\Bigr).
\]
Using \eqref{eq:app_eta_prime_local_bounds} and
\eqref{eq:app_ctilde_log_bound}, we have, uniformly over
\(\eta'\in B(\eta,\delta_{\mathrm{reg}})\),
\begin{equation}
\label{eq:app_shaped_reward_growth}
|\widetilde R_{\eta'}(u)|
\le
C\sqrt{\log N}\,(1+|u|^2).
\end{equation}
Moreover, differentiating the shaped reward with respect to the tail vector coordinates
 gives
\[
\partial_{r'}\widetilde R_{\eta'}(u)
=
-1-\frac{\widetilde c_N}{\sigma'}(r'-\mu'),
\]
\[
\partial_{\mu'}\widetilde R_{\eta'}(u)
=
-\frac{\widetilde c_N}{\sigma'}(u-r'),
\]
and
\[
\partial_{\sigma'}\widetilde R_{\eta'}(u)
=
-\frac{\widetilde c_N}{2(\sigma')^2}
\Bigl((u-\mu')^2-(r'-\mu')^2\Bigr).
\]
Therefore,
\begin{equation}
\label{eq:app_shaped_reward_eta_derivative}
\|\nabla_{\eta'}\widetilde R_{\eta'}(u)\|
\le
C\sqrt{\log N}\,(1+|u|^2).
\end{equation}
Differentiating once more gives
\begin{equation}
\label{eq:app_shaped_reward_eta_second_derivative}
\|\nabla_{\eta'}^2\widetilde R_{\eta'}(u)\|
\le
C\sqrt{\log N}\,(1+|u|^2),
\qquad
\eta'\in B(\eta,\delta_{\mathrm{reg}}).
\end{equation}
We will also use the one-dimensional derivative bound
\begin{equation}
\label{eq:app_shaped_reward_u_derivative}
\left|
\partial_u\widetilde R_{\eta'}(u)
\right|
=
\left|
1+\frac{\widetilde c_N}{\sigma'}(u-\mu')
\right|
\le
C\sqrt{\log N},
\end{equation}
whenever \(u\in I_{\mathrm{loc}}\) and
\(\eta'\in B(\eta,\delta_{\mathrm{reg}})\).

Taking \(\eta'=\eta\) in \eqref{eq:app_shaped_reward_growth}, using
\(\|S(y)\|\le G\), and using \(\mathbb E[R^4]\le M_R\), we obtain
\[
P\|\varphi_\eta\|^2
=
\frac1{\alpha^2}
\mathbb E\!\left[
\Ind{R\ge r}
|\widetilde R_\eta(R)|^2
\|S(y)\|^2
\right]
\le
C\log N.
\]
This proves \eqref{eq:app_phi_oracle_second_moment}.

We next prove the local \(L^2(P)\) stability. Let
\[
\eta_1=(r_1,\mu_1,\sigma_1),
\qquad
\eta_2=(r_2,\mu_2,\sigma_2),
\qquad
\Delta:=\|\eta_1-\eta_2\|.
\]
Assume without loss of generality that \(r_1\le r_2\). On the common tail
region \(\{R\ge r_2\}\), the indicators agree. By the mean-value theorem and
\eqref{eq:app_shaped_reward_eta_derivative},
\[
\|\varphi_{\eta_1}(y)-\varphi_{\eta_2}(y)\|
\le
C\sqrt{\log N}\,(1+|R(y)|^2)\Delta .
\]
Thus,
\begin{equation}
\label{eq:app_common_tail_l2_bound}
P\!\left[
\|\varphi_{\eta_1}-\varphi_{\eta_2}\|^2
\Ind{R\ge r_2}
\right]
\le
C\log N\,\Delta^2 .
\end{equation}

It remains to control the symmetric difference of the two tail regions,
\(\{r_1\le R<r_2\}\). On this region only \(\varphi_{\eta_1}\) contributes.
Since
\[
\widetilde R_{\eta_1}(r_1)=0,
\]
and \([r_1,r_2]\subset I_{\mathrm{loc}}\), the mean-value
theorem and \eqref{eq:app_shaped_reward_u_derivative} give, for every
\(u\in[r_1,r_2]\),
\[
|\widetilde R_{\eta_1}(u)|
\le
C\sqrt{\log N}\,|u-r_1|
\le
C\sqrt{\log N}\,|r_2-r_1|
\le
C\sqrt{\log N}\,\Delta .
\]
The reward density is uniformly bounded on this interval because it is Gaussian
there and \(\sigma_\theta(x)\ge\sigma_{\min}\). Hence
\[
\mathbb P(r_1\le R<r_2)
\le
C|r_2-r_1|
\le
C\Delta.
\]
Using \(\|S(y)\|\le G\), we get
\begin{equation}
\label{eq:app_boundary_l2_bound}
P\!\left[
\|\varphi_{\eta_1}\|^2
\Ind{r_1\le R<r_2}
\right]
\le
C\log N\,\Delta^2\,\mathbb P(r_1\le R<r_2)
\le
C\log N\,\Delta^3
\le
C\log N\,\Delta^2,
\end{equation}
where the last inequality uses \(\Delta\le 2\delta_{\mathrm{reg}}\) and
\(\delta_{\mathrm{reg}}\le1\). Combining
\eqref{eq:app_common_tail_l2_bound} and
\eqref{eq:app_boundary_l2_bound} proves
\eqref{eq:app_phi_l2_lipschitz}. The case \(r_2<r_1\) is symmetric.

We will also use the following pointwise local consequence of the same
argument. For any \(\eta_1,\eta_2\in B(\eta,\delta_{\mathrm{reg}})\), and all
\(y\),
\begin{equation}
\label{eq:app_phi_pointwise_local_bound}
\|\varphi_{\eta_1}(y)-\varphi_{\eta_2}(y)\|^2
\le
C\log N\,(1+|R(y)|^4)\,\|\eta_1-\eta_2\|^2 .
\end{equation}
Indeed, on the common tail region this follows from the mean-value theorem in
\(\eta\) and \eqref{eq:app_shaped_reward_eta_derivative}. On the symmetric
difference of the two tail regions, it follows from
\(\widetilde R_{\eta_j}(r_j)=0\), the one-dimensional derivative bound
\eqref{eq:app_shaped_reward_u_derivative}, and
\(|r_1-r_2|\le \|\eta_1-\eta_2\|\). Outside the union of the two tail regions,
both sides of the shaped-score difference vanish.

We now turn to the population plug-in map \(H\). Let \(a(t)\) be a bounded
version of the conditional score, i.e.
\[
a(t)=\mathbb E[S(y)\mid R(y)=t]
\]
for \(p_{\theta,x}\)-almost every \(t\). Since \(\|S(y)\|\le G\), we may choose
this version so that
\[
\|a(t)\|\le G
\]
for \(p_{\theta,x}\)-almost every \(t\). For
\(\eta'=(r',\mu',\sigma')\in B(\eta,\delta_{\mathrm{reg}})\), the threshold
\(r'\) lies in the Gaussian tail, so
\begin{equation}
\label{eq:app_H_integral_representation}
H(\eta')
=
\frac1\alpha
\int_{r'}^\infty
\widetilde R_{\eta'}(t)\,
a(t)\,
p_{\theta,x}(t)\,dt .
\end{equation}

Define the linear map
\begin{equation}
\label{eq:app_H_derivative_definition}
\dot H_\eta[h]
:=
\frac1\alpha
\int_{r}^\infty
\left\langle
\nabla_{\eta}\widetilde R_{\eta}(t),
h
\right\rangle
a(t)p_{\theta,x}(t)\,dt,
\qquad h\in\mathbb R^3.
\end{equation}
The boundary term that would arise from differentiating the lower limit is zero
because
\[
\widetilde R_{\eta}(r)=0.
\]
By \eqref{eq:app_shaped_reward_eta_derivative}, the Gaussian tail envelope, and
\(\|a(t)\|\le G\),
\[
\|\dot H_\eta\|
\le
C\sqrt{\log N}.
\]
This proves \eqref{eq:app_H_derivative_bound}.

It remains to prove the quadratic remainder. Let
\[
h:=\eta'-\eta=(h_r,h_\mu,h_\sigma),
\qquad
r_h:=r+h_r.
\]
We compare \(H(\eta+h)\) with \(H(\eta)+\dot H_\eta[h]\). By
\eqref{eq:app_H_integral_representation} and
\eqref{eq:app_H_derivative_definition},
\[
H(\eta+h)
=
\frac1\alpha
\int_{r_h}^{\infty}
\widetilde R_{\eta+h}(t)a(t)p_{\theta,x}(t)\,dt,
\]
whereas
\[
H(\eta)+\dot H_\eta[h]
=
\frac1\alpha
\int_{r}^{\infty}
\Bigl(
\widetilde R_{\eta}(t)
+
\left\langle
\nabla_\eta \widetilde R_{\eta}(t),h
\right\rangle
\Bigr)
a(t)p_{\theta,x}(t)\,dt .
\]
Let
\[
\ell_h:=\max\{r,r_h\},
\qquad
J_h:=[\min\{r,r_h\},\max\{r,r_h\}].
\]
Then \(J_h\subset I_{\mathrm{loc}}\), \(|J_h|\le |h_r|\le \|h\|\), and the common
integration region is
\[
[\ell_h,\infty)
=
[r,\infty)\cap[r_h,\infty).
\]

Define the pointwise Taylor remainder
\[
\mathcal R_h(t)
:=
\widetilde R_{\eta+h}(t)
-
\widetilde R_{\eta}(t)
-
\left\langle
\nabla_{\eta}\widetilde R_{\eta}(t),h
\right\rangle .
\]
If \(h_r\ge0\), then \(r_h\ge r\), and
\[
\begin{aligned}
H(\eta+h)-H(\eta)-\dot H_\eta[h]
&=
\frac1\alpha
\int_{r_h}^{\infty}
\mathcal R_h(t)a(t)p_{\theta,x}(t)\,dt  \\
&\quad
-
\frac1\alpha
\int_{r}^{r_h}
\Bigl(
\widetilde R_{\eta}(t)
+
\left\langle
\nabla_\eta\widetilde R_{\eta}(t),h
\right\rangle
\Bigr)
a(t)p_{\theta,x}(t)\,dt .
\end{aligned}
\]
If \(h_r<0\), then \(r_h<r\), and
\[
\begin{aligned}
H(\eta+h)-H(\eta)-\dot H_\eta[h]
&=
\frac1\alpha
\int_{r}^{\infty}
\mathcal R_h(t)a(t)p_{\theta,x}(t)\,dt  \\
&\quad
+
\frac1\alpha
\int_{r_h}^{r}
\widetilde R_{\eta+h}(t)a(t)p_{\theta,x}(t)\,dt .
\end{aligned}
\]
Thus, in both cases,
\[
\begin{aligned}
\bigl\|
H(\eta+h)-H(\eta)-\dot H_\eta[h]
\bigr\|
&\le
\frac1\alpha
\int_{\ell_h}^{\infty}
|\mathcal R_h(t)|\,\|a(t)\|p_{\theta,x}(t)\,dt \\
&\quad
+
\frac1\alpha
\int_{J_h}
\Bigl(
|\widetilde R_\eta(t)|
+
\left|
\left\langle
\nabla_\eta\widetilde R_\eta(t),h
\right\rangle
\right|
+
|\widetilde R_{\eta+h}(t)|
\Bigr)
\|a(t)\|p_{\theta,x}(t)\,dt .
\end{aligned}
\]
We bound the two terms separately.

First, on the common region \([\ell_h,\infty)\), Taylor's theorem in the tail
vector and \eqref{eq:app_shaped_reward_eta_second_derivative} give
\[
|\mathcal R_h(t)|
\le
C\sqrt{\log N}\,(1+t^2)\|h\|^2 .
\]
Since \(\ell_h\) lies inside the modeled Gaussian tail and
\(\|a(t)\|\le G\), the Gaussian tail envelope gives
\[
\int_{\ell_h}^{\infty}
(1+t^2)\|a(t)\|p_{\theta,x}(t)\,dt
\le C .
\]
Therefore,
\[
\frac1\alpha
\int_{\ell_h}^{\infty}
|\mathcal R_h(t)|\,\|a(t)\|p_{\theta,x}(t)\,dt
\le
C\sqrt{\log N}\,\|h\|^2 .
\]

Second, consider the moving-boundary interval \(J_h\). Since
\(\widetilde R_\eta(r)=0\), \(\widetilde R_{\eta+h}(r_h)=0\), and
\(J_h\subset I_{\mathrm{loc}}\), the one-dimensional derivative bound
\eqref{eq:app_shaped_reward_u_derivative} implies that, for every \(t\in J_h\),
\[
|\widetilde R_\eta(t)|
\le
C\sqrt{\log N}\,|t-r|
\le
C\sqrt{\log N}\,\|h\|,
\]
and
\[
|\widetilde R_{\eta+h}(t)|
\le
C\sqrt{\log N}\,|t-r_h|
\le
C\sqrt{\log N}\,\|h\|.
\]
Moreover, since \(J_h\) is a compact local interval and
\eqref{eq:app_shaped_reward_eta_derivative} holds uniformly,
\[
\left|
\left\langle
\nabla_\eta\widetilde R_\eta(t),h
\right\rangle
\right|
\le
C\sqrt{\log N}\,\|h\|,
\qquad t\in J_h .
\]
The score profile and the reward density are uniformly bounded on \(J_h\).
Hence
\[
\begin{aligned}
&\frac1\alpha
\int_{J_h}
\Bigl(
|\widetilde R_\eta(t)|
+
\left|
\left\langle
\nabla_\eta\widetilde R_\eta(t),h
\right\rangle
\right|
+
|\widetilde R_{\eta+h}(t)|
\Bigr)
\|a(t)\|p_{\theta,x}(t)\,dt  \\
&\qquad\le
C\sqrt{\log N}\,\|h\|\,|J_h|
\le
C\sqrt{\log N}\,\|h\|^2 .
\end{aligned}
\]
Combining the common-region and moving-boundary bounds gives
\[
\bigl\|
H(\eta+h)-H(\eta)-\dot H_\eta[h]
\bigr\|
\le
C\sqrt{\log N}\,\|h\|^2,
\]
which proves \eqref{eq:app_H_second_order_remainder}.

Finally, \eqref{eq:app_H_lipschitz} follows directly from
\eqref{eq:app_phi_l2_lipschitz}. Indeed, by Jensen's inequality,
\[
\|H(\eta')-H(\eta)\|
=
\|P(\varphi_{\eta'}-\varphi_\eta)\|
\le
\left(P\|\varphi_{\eta'}-\varphi_\eta\|^2\right)^{1/2}
\le
C_{\mathrm{reg}}\sqrt{\log N}\|\eta'-\eta\|.
\]
This completes the proof.
\end{proof}

\subsection{Proof of \Cref{thm:direct_plugin_estimator}}
\label{app:proof-direct-plugin-estimator}

Recall that, with the notation from
\Cref{app:gradient_estimator_common_notation},
\[
\widehat g_m^{\mathrm{dir}}(\theta;x)
=
P_m\varphi_{\hat\eta_m},
\qquad
g_\theta(x)=P\varphi_\eta.
\]
The direct estimator is coupled because the same batch is used both to construct
\(\hat\eta_m\) and to evaluate the empirical average \(P_m\). Before proving the main result, we first introduce the leave-one-out estimators that will be used to control the influence of individual samples on the tail vector and the empirical average.

We fix \(i=1\) by exchangeability. Let
\[
k_m:=\lceil \alpha m\rceil,
\qquad
k_{m-1}:=\lceil \alpha(m-1)\rceil .
\]
Let
\[
R_{(1)}\le\cdots\le R_{(m)}
\]
be the order statistics of the full batch, and let
\[
Z_{(1)}\le\cdots\le Z_{(m-1)}
\]
be the order statistics of the leave-one-out rewards
\(\{R_2,\dots,R_m\}\). Define the full-sample and leave-one-out threshold
ranks
\[
j_m:=m-k_m+1,
\qquad
\ell_m:=m-k_{m-1}.
\]
Thus
\[
\hat r_m=R_{(j_m)},
\qquad
\hat r_m^{(-1)}=Z_{(\ell_m)}.
\]
Since \(k_m-k_{m-1}\in\{0,1\}\), we have
\[
|j_m-\ell_m|\le 1.
\]

For \(0\le k\le m-1\), define the leave-one-out top-\(k\) sums
\[
T_k:=\sum_{j=m-k}^{m-1}Z_{(j)},
\qquad
T_k^{(2)}:=\sum_{j=m-k}^{m-1}Z_{(j)}^2,
\qquad
T_0=T_0^{(2)}:=0.
\]
The leave-one-out empirical tail statistics are
\[
\hat r_m^{(-1)}:=Z_{(\ell_m)},
\qquad
\hat\mu_m^{(-1)}
:=
\frac{1}{k_{m-1}}T_{k_{m-1}},
\]
and
\[
(\hat\sigma_m^{(-1),\mathrm{raw}})^2
:=
\frac{1}{k_{m-1}}T_{k_{m-1}}^{(2)}
-
\bigl(\hat\mu_m^{(-1)}\bigr)^2,
\qquad
\hat\sigma_m^{(-1)}
:=
\max\{\hat\sigma_m^{(-1),\mathrm{raw}},\varepsilon_\sigma\},
\]
where $\varepsilon_\sigma>0$ is the same as defined in \Cref{sec:finite_sample_gradient_estimators}. 
We write
\[
\hat\eta_m^{(-1)}
:=
(\hat r_m^{(-1)},\hat\mu_m^{(-1)},\hat\sigma_m^{(-1)}).
\]

We also define two local leave-one-out envelopes around the empirical tail
threshold:
\begin{equation}
\label{eq:app_leave_one_out_spacing_envelope}
D_1
:=
\max_{\ell_m-5\le j\le \ell_m+5}
|Z_{(j+1)}-Z_{(j)}|,
\end{equation}
and
\begin{equation}
\label{eq:app_leave_one_out_boundary_envelope}
B_1
:=
1+
\max_{\ell_m-5\le j\le \ell_m+5}
\bigl(|Z_{(j)}|+Z_{(j)}^2\bigr),
\end{equation}
where indices outside \(\{1,\dots,m-1\}\) are ignored. The spacing \(D_1\)
controls threshold changes, while \(B_1/m\) controls the effect of the constant
number of boundary terms that may enter or leave the top-tail sums. We first bound
the moments of these envelopes.

\begin{lemma}
\label{lem:leave_one_out_envelope_moments}
Under Assumptions~\ref{assum:gaussian-tail-model} and
\ref{assum:nondegenerate_scale}, for all sufficiently large \(m\),
\begin{equation}
\label{eq:app_D1_moment_bounds}
\mathbb E[D_1]\le \frac{C}{m},
\qquad
\mathbb E[D_1^2]\le \frac{C}{m^2},
\end{equation}
and
\begin{equation}
\label{eq:app_B1_moment_bound}
\mathbb E[B_1^2]\le C.
\end{equation}
Moreover, \(D_1\) and \(B_1\) are functions only of the leave-one-out sample
\(\{y_2,\dots,y_m\}\), and hence are independent of \(R_1\).
\end{lemma}

\begin{proof}
The independence statement follows directly from the definitions.

Let
\[
U_j:=F_{\theta,x}(R_j),\qquad j=1,\dots,m,
\]
and let
\[
U^{(-1)}_{(1)}\le \cdots \le U^{(-1)}_{(m-1)}
\]
be the order statistics of \(U_2,\dots,U_m\). By the probability integral
transform, \(U_2,\dots,U_m\) are i.i.d.\ \(\mathrm{Unif}(0,1)\), and
\[
Z_{(j)}=Q(U^{(-1)}_{(j)}),
\qquad
Q:=F_{\theta,x}^{-1}.
\]

We first localize the order statistics in the fixed window
\(\ell_m-5\le j\le \ell_m+6\). Since \(\ell_m/(m-1)\) stays in a fixed neighborhood of \(1-\alpha\), the same
binomial-count argument as in \Cref{lem:uniform_order_tail_facts} gives an event
\(\mathcal A_m\) such that
\[
\mathbb P(\mathcal A_m^c)\le Ce^{-cm},
\]
and on \(\mathcal A_m\), all \(U^{(-1)}_{(j)}\) in this window lie in a compact
subinterval \(I_{\mathrm{loc}}\Subset(1-2\alpha,1)\). On this interval, the
Gaussian-tail model gives
\[
Q(u)=\bar\mu+\bar\sigma\Phi^{-1}(u).
\]
Writing \(z(u):=\Phi^{-1}(u)\), we have
\[
Q'(u)=\frac{\bar\sigma}{\phi(z(u))}.
\]
Since \(I_{\mathrm{loc}}\) is compactly contained in \((1-2\alpha,1)\),
\(\phi(z(u))\) is bounded away from zero on \(I_{\mathrm{loc}}\). Together with
the uniform bound \(\bar\sigma\le C\) from
\Cref{cor:uniform_tail_bounds}, this gives
\[
\sup_{u\in I_{\mathrm{loc}}}|Q'(u)|\le C.
\]

Hence, on
\(\mathcal A_m\),
\[
D_1
\le
\sup_{I_{\mathrm{loc}}
}|Q'(u)|
\max_{\ell_m-5\le j\le \ell_m+5}
\left(
U^{(-1)}_{(j+1)}-U^{(-1)}_{(j)}
\right) \leq C \max_{\ell_m-5\le j\le \ell_m+5}
\left(
U^{(-1)}_{(j+1)}-U^{(-1)}_{(j)}
\right)
\]

Let
\[
S_j:=U^{(-1)}_{(j+1)}-U^{(-1)}_{(j)}
\]
be an adjacent uniform spacing. The vector of uniform spacings, including the
two endpoint spacings, has a \(\mathrm{Dirichlet}(1,\dots,1)\) distribution
with \(m\) components. Thus each interior spacing has marginal distribution
\[
S_j\sim \mathrm{Beta}(1,m-1),
\]
and therefore
\[
\mathbb E[S_j]=\frac1m,
\qquad
\mathbb E[S_j^2]=\frac{2}{m(m+1)}.
\]
Since the maximum in \(D_1\) is over a fixed number of adjacent spacings,
\[
\mathbb E[D_1\Ind{\mathcal A_m}]
\le
\frac{C}{m},
\qquad
\mathbb E[D_1^2\Ind{\mathcal A_m}]
\le
\frac{C}{m^2}.
\]

We next control the contribution from \(\mathcal A_m^c\). Since
\(D_1\) is a reward-spacing envelope, we first bound the moments of the
boundary order statistics. For every \(j\) in the fixed window
\(\ell_m-5\le j\le \ell_m+6\), the rank \(j/(m-1)\) stays in a fixed
neighborhood of \(1-\alpha\). We claim that
\begin{equation}
\label{eq:app_boundary_order_stat_fourth_moment}
\mathbb E\!\left[
\max_{\ell_m-5\le j\le \ell_m+6}|Z_{(j)}|^4
\right]
\le C .
\end{equation}

To prove the claim, fix such a \(j\). For the positive tail, if
\(Z_{(j)}\ge t\), then at least \((m-1)-j+1\ge c_\alpha m\) leave-one-out
observations are at least \(t\). Thus, with
\[
N_t^+:=\sum_{\ell=2}^{m}\Ind{R_\ell\ge t},
\qquad
p_t^+:=\mathbb P(R\ge t),
\]
we have \(N_t^+\sim\mathrm{Bin}(m-1,p_t^+)\) and
\[
\mathbb P(Z_{(j)}\ge t)
\le
\mathbb P(N_t^+\ge c_\alpha m)
\le
C p_t^+.
\]
By the Gaussian upper-tail model and the uniform bounds from
\Cref{cor:uniform_tail_bounds},
\[
p_t^+\le C e^{-ct^2}
\]
for all sufficiently large \(t\). Hence
\[
\mathbb P(Z_{(j)}\ge t)\le C e^{-ct^2}.
\]

For the negative tail, if \(Z_{(j)}\le -t\), then at least
\(j\ge c_\alpha m\) leave-one-out observations are at most \(-t\). Let
\[
N_t^-:=\sum_{\ell=2}^{m}\Ind{R_\ell\le -t},
\qquad
p_t^-:=\mathbb P(R\le -t).
\]
Then \(N_t^-\sim\mathrm{Bin}(m-1,p_t^-)\). By the fourth-moment bound,
\[
p_t^-\le M_R t^{-4}.
\]
Using the second falling-factorial moment of a binomial random variable,
\[
\mathbb P(Z_{(j)}\le -t)
\le
\mathbb P(N_t^-\ge c_\alpha m)
\le
\frac{\mathbb E[(N_t^-)_2]}{(c_\alpha m)_2}
=
\frac{(m-1)_2(p_t^-)^2}{(c_\alpha m)_2}
\le
C t^{-8}.
\]
where $(n)_2=n(n-1)$ is the second falling factorial. 
Combining the positive and negative tail bounds and using the layer-cake
formula gives
\[
\mathbb E|Z_{(j)}|^4
=
4\int_0^\infty t^3\mathbb P(|Z_{(j)}|\ge t)\,dt
\le C.
\]
Since the maximum in \eqref{eq:app_boundary_order_stat_fourth_moment} is over
only a fixed number of indices, the claim follows.

Finally, because
\[
D_1
=
\max_{\ell_m-5\le j\le \ell_m+5}|Z_{(j+1)}-Z_{(j)}|
\le
2\max_{\ell_m-5\le j\le \ell_m+6}|Z_{(j)}|,
\]
\eqref{eq:app_boundary_order_stat_fourth_moment} implies
\[
\mathbb E[D_1^2]\le C,
\qquad
\mathbb E[D_1^4]\le C.
\]
Therefore, by Cauchy--Schwarz and
\(\mathbb P(\mathcal A_m^c)\le Ce^{-cm}\),
\[
\mathbb E[D_1\Ind{\mathcal A_m^c}]
\le
\bigl(\mathbb E[D_1^2]\bigr)^{1/2}
\mathbb P(\mathcal A_m^c)^{1/2}
\le
Ce^{-cm},
\]
and
\[
\mathbb E[D_1^2\Ind{\mathcal A_m^c}]
\le
\bigl(\mathbb E[D_1^4]\bigr)^{1/2}
\mathbb P(\mathcal A_m^c)^{1/2}
\le
Ce^{-cm},
\]
after decreasing \(c>0\) if necessary.
Together with the localized bounds, this proves \eqref{eq:app_D1_moment_bounds}.

Finally, by definition,
\[
B_1^2
\le
C\left(
1+
\max_{\ell_m-5\le j\le \ell_m+5}|Z_{(j)}|^4
\right).
\]
The fourth-moment bound \eqref{eq:app_boundary_order_stat_fourth_moment}, and
the fact that the maximum is over only a fixed number of indices, give
\[
\mathbb E[B_1^2]\le C.
\]
This proves \eqref{eq:app_B1_moment_bound}.
\end{proof}

Then we will analyze the deterministic influence of the first sample on the leave-one-out tail vector.
\begin{lemma}
\label{lem:deterministic_leave_one_out_tail_stability}
For $\delta_{\mathrm{reg}}$ as in \Cref{lem:shaped_score_local_regularity}, define the event
\[
\mathcal E_1
:=
\left\{
\hat\eta_m\in B(\eta,\delta_{\mathrm{reg}}),
\,
\hat\eta_m^{(-1)}\in B(\eta,\delta_{\mathrm{reg}})
\right\}.
\]
Then, on \(\mathcal E_1\),
\begin{equation}
\label{eq:app_leave_one_out_influence_bound}
\|\hat\eta_m-\hat\eta_m^{(-1)}\|
\le
C\left(
D_1+\frac{B_1}{m}
+
\frac{|R_1|+R_1^2}{m}
\Ind{R_1\ge r_{\theta,2\alpha}(x)}
\right),
\end{equation}
where \(C\) depends only on \((\alpha,G,\sigma_{\min},M_R)\).
\end{lemma}

\begin{proof}
We prove the three coordinates separately.

\emph{Threshold.}
By the interlacing property of order statistics, after deleting \(R_1\), every
full-sample order statistic is either unchanged or shifted by one rank relative
to the leave-one-out order statistics. Since \(|j_m-\ell_m|\le1\), the two
thresholds \(R_{(j_m)}\) and \(Z_{(\ell_m)}\) are separated by at most a fixed
number of adjacent leave-one-out spacings around rank \(\ell_m\). Therefore,
by the definition of \(D_1\),
\begin{equation}
\label{eq:app_threshold_deterministic_influence}
|\hat r_m-\hat r_m^{(-1)}|
\le
C D_1
\qquad
\text{on }\mathcal E_1.
\end{equation}

\emph{Tail mean.}
Let \(\mathcal I_m\) be the full-sample top-\(k_m\) index set. The full-sample
top-tail sum can be written exactly in terms of the leave-one-out top sums:
\begin{equation}
\label{eq:app_full_top_sum_from_loo}
\sum_{j\in\mathcal I_m}R_j
=
\Ind{1\in\mathcal I_m}R_1
+
T_{k_m-\Ind{1\in\mathcal I_m}}.
\end{equation}
Indeed, if \(R_1\) is selected, then the remaining selected rewards are the top
\(k_m-1\) rewards among \(\{R_2,\dots,R_m\}\); otherwise they are the top
\(k_m\) leave-one-out rewards.

Subtracting the leave-one-out top-tail sum gives
\begin{align}
\left|
\sum_{j\in\mathcal I_m}R_j
-
T_{k_{m-1}}
\right|
&\le
|R_1|\Ind{1\in\mathcal I_m}
+
\left|
T_{k_m-\Ind{1\in\mathcal I_m}}
-
T_{k_{m-1}}
\right|.
\label{eq:app_top_sum_difference}
\end{align}
Because \(k_m-k_{m-1}\in\{0,1\}\) and
\(\Ind{1\in\mathcal I_m}\in\{0,1\}\), the two leave-one-out top sums in the
second term differ by at most one boundary order statistic. By the definition
of \(B_1\),
\[
\left|
T_{k_m-\Ind{1\in\mathcal I_m}}
-
T_{k_{m-1}}
\right|
\le
B_1.
\]
Hence
\begin{equation}
\label{eq:app_top_sum_difference_bound}
\left|
\sum_{j\in\mathcal I_m}R_j
-
T_{k_{m-1}}
\right|
\le
B_1
+
|R_1|\Ind{1\in\mathcal I_m}.
\end{equation}

Moreover, if \(1\in\mathcal I_m\), then \(R_1\ge \hat r_m\). On
\(\mathcal E_1\), the threshold \(\hat r_m\) lies in the local ball around
\(r=r_{\theta,\alpha}(x)\), and \(\delta_{\mathrm{reg}}\) is chosen smaller
than the fixed gap between \(r_{\theta,\alpha}(x)\) and
\(r_{\theta,2\alpha}(x)\). Thus
\begin{equation}
\label{eq:app_removed_if_selected_in_tail}
\Ind{1\in\mathcal I_m}
\le
\Ind{R_1\ge r_{\theta,2\alpha}(x)}
\qquad
\text{on }\mathcal E_1.
\end{equation}

We decompose
\begin{align}
|\hat\mu_m-\hat\mu_m^{(-1)}|
&=
\left|
\frac1{k_m}\sum_{j\in\mathcal I_m}R_j
-
\frac1{k_{m-1}}T_{k_{m-1}}
\right| \notag\\
&\le
\frac1{k_m}
\left|
\sum_{j\in\mathcal I_m}R_j-T_{k_{m-1}}
\right|
+
\left|
\frac1{k_m}-\frac1{k_{m-1}}
\right|
|T_{k_{m-1}}|.
\label{eq:app_mean_diff_decomposition}
\end{align}
For the second term, on \(\mathcal E_1\), we can uniformly bound the leave-one-out empirical tail mean \(| \hat{\mu}_m^{(-1)}|=|T_{k_{m-1}}|/k_{m-1} \leq C\) because \(\hat\eta_m^{(-1)}\) lies in the local ball around \(\eta\). Therefore, we have:

\[
\left|
\frac1{k_m}-\frac1{k_{m-1}}
\right|
|T_{k_{m-1}}| = \frac{| k_m - k_{m-1}|}{k_m} |\hat{\mu}_m^{(-1)}| \leq \frac{C}{m}.
\]

Using
\eqref{eq:app_top_sum_difference_bound} and
\eqref{eq:app_removed_if_selected_in_tail},
\begin{equation}
\label{eq:app_mean_deterministic_influence}
|\hat\mu_m-\hat\mu_m^{(-1)}|
\le
C\left(
\frac{B_1}{m}
+
\frac{|R_1|}{m}\Ind{R_1\ge r_{\theta,2\alpha}(x)}
\right).
\end{equation}

\emph{Second moment and scale.}
The empirical second tail moment is handled similarly. From
\[
\sum_{j\in\mathcal I_m}R_j^2
=
\Ind{1\in\mathcal I_m}R_1^2
+
T^{(2)}_{k_m-\Ind{1\in\mathcal I_m}},
\]
and the fact that the two leave-one-out top sums differ by at most one boundary
term, we obtain
\[
\left|
\sum_{j\in\mathcal I_m}R_j^2
-
T^{(2)}_{k_{m-1}}
\right|
\le
B_1
+
R_1^2\Ind{R_1\ge r_{\theta,2\alpha}(x)}
\]
on \(\mathcal E_1\). Therefore, with
\[
\hat M_{2,m}:=\frac1{k_m}\sum_{j\in\mathcal I_m}R_j^2,
\qquad
\hat M_{2,m}^{(-1)}
:=
\frac1{k_{m-1}}T^{(2)}_{k_{m-1}},
\]
and using that \(\hat M_{2,m}^{(-1)}\) is uniformly bounded on
\(\mathcal E_1\), the denominator-change term is \(O(m^{-1})\). Hence
\begin{equation}
\label{eq:app_second_moment_deterministic_influence}
|\hat M_{2,m}-\hat M_{2,m}^{(-1)}|
\le
C\left(
\frac{B_1}{m}
+
\frac{R_1^2}{m}\Ind{R_1\ge r_{\theta,2\alpha}(x)}
\right).
\end{equation}

Now write
\[
\hat v_m:=(\hat\sigma_m^{\mathrm{raw}})^2
=
\hat M_{2,m}-\hat\mu_m^2,
\qquad
\hat v_m^{(-1)}
:=
(\hat\sigma_m^{(-1),\mathrm{raw}})^2
=
\hat M_{2,m}^{(-1)}-(\hat\mu_m^{(-1)})^2.
\]
On \(\mathcal E_1\), the empirical tail means are uniformly bounded. Therefore,
using
\[
|\hat v_m-\hat v_m^{(-1)}|
\le
|\hat M_{2,m}-\hat M_{2,m}^{(-1)}|
+
|\hat\mu_m+\hat\mu_m^{(-1)}|\,
|\hat\mu_m-\hat\mu_m^{(-1)}|,
\]
together with \eqref{eq:app_mean_deterministic_influence} and
\eqref{eq:app_second_moment_deterministic_influence}, we obtain
\[
|\hat v_m-\hat v_m^{(-1)}|
\le
C\left(
\frac{B_1}{m}
+
\frac{R_1^2}{m}
\Ind{R_1\ge r_{\theta,2\alpha}(x)}
\right).
\]
The clipped square-root map \(v\mapsto\max\{\sqrt v,\varepsilon_\sigma\}\) is
Lipschitz on \([0,\infty)\). Hence
\begin{equation}
\label{eq:app_scale_deterministic_influence}
|\hat\sigma_m-\hat\sigma_m^{(-1)}|
\le
C\left(
\frac{B_1}{m}
+
\frac{R_1^2}{m}
\Ind{R_1\ge r_{\theta,2\alpha}(x)}
\right).
\end{equation}

Combining
\eqref{eq:app_threshold_deterministic_influence},
\eqref{eq:app_mean_deterministic_influence}, and
\eqref{eq:app_scale_deterministic_influence} gives
\[
\|\hat\eta_m-\hat\eta_m^{(-1)}\|
\le
C\left(
D_1+\frac{B_1}{m}
+
\frac{|R_1|+R_1^2}{m}
\Ind{R_1\ge r_{\theta,2\alpha}(x)}
\right),
\]
which proves \eqref{eq:app_leave_one_out_influence_bound}.
\end{proof}

Combining the samplewise influence bound with the envelope moment bounds gives
the following stability estimates for the leave-one-out tail vectors.

\begin{lemma}
\label{lem:leave_one_out_tail_stability}
For each \(i\in\{1,\dots,m\}\), let \(\hat\eta_m^{(-i)}\) be the empirical tail
vector computed from the \(m-1\) samples with the \(i\)-th sample removed, using
\(k_{m-1}:=\lceil \alpha(m-1)\rceil\) and the same clipped scale
\(\varepsilon_\sigma\). Under Assumptions~\ref{assum:gaussian-tail-model} and
\ref{assum:nondegenerate_scale}, for all sufficiently large \(m\),
\begin{equation}
\label{eq:app_leave_one_out_tail_stability}
\mathbb E\bigl[
\|\hat\eta_m-\hat\eta_m^{(-i)}\|^2
\bigr]
\le
\frac{C}{m^2},
\qquad
\mathbb E\bigl[
\|\hat\eta_m-\hat\eta_m^{(-i)}\|
\bigr]
\le
\frac{C}{m}.
\end{equation}
Moreover,
\begin{equation}
\label{eq:app_weighted_leave_one_out_l2}
\mathbb E\!\left[
(1+|R_i|^4)
\|\hat\eta_m-\hat\eta_m^{(-i)}\|^2
\Ind{\hat\eta_m,\hat\eta_m^{(-i)}\in B(\eta,\delta_{\mathrm{reg}})}
\right]
\le
\frac{C}{m^2},
\end{equation}
and
\begin{equation}
\label{eq:app_weighted_leave_one_out_l1}
\mathbb E\!\left[
(1+|R_i|^2)
\|\hat\eta_m-\hat\eta_m^{(-i)}\|
\Ind{\hat\eta_m,\hat\eta_m^{(-i)}\in B(\eta,\delta_{\mathrm{reg}})}
\right]
\le
\frac{C}{m}.
\end{equation}
Finally,
\begin{equation}
\label{eq:app_leave_one_out_localization}
\mathbb P\!\left(
\hat\eta_m\notin B(\eta,\delta_{\mathrm{reg}})
\ \text{or}\
\hat\eta_m^{(-i)}\notin B(\eta,\delta_{\mathrm{reg}})
\right)
\le
Ce^{-cm}.
\end{equation}
The constants are uniform over \(i,x,\theta\), and
\(\delta_{\mathrm{reg}}\) is as in
\Cref{lem:shaped_score_local_regularity}.
\end{lemma}

\begin{proof}
By exchangeability, it suffices to consider \(i=1\). Let
\[
\mathcal E_1
:=
\left\{
\hat\eta_m\in B(\eta,\delta_{\mathrm{reg}}),
\,
\hat\eta_m^{(-1)}\in B(\eta,\delta_{\mathrm{reg}})
\right\}.
\]
The localization bound \eqref{eq:app_leave_one_out_localization} follows by
repeating the good-event argument in
\Cref{lem:empirical_tail_vector_bounds} with the fixed localization radius
\(\delta_{\mathrm{reg}}\), once for the full sample and once for the
leave-one-out sample.

We now prove the weighted moment bounds. On \(\mathcal E_1\),
\Cref{lem:deterministic_leave_one_out_tail_stability} gives
\[
\|\hat\eta_m-\hat\eta_m^{(-1)}\|
\le
C\left(
D_1+\frac{B_1}{m}
+
\frac{|R_1|+R_1^2}{m}
\Ind{R_1\ge r_{\theta,2\alpha}(x)}
\right).
\]
By \Cref{lem:leave_one_out_envelope_moments},
\[
\mathbb E[D_1]\le \frac{C}{m},
\qquad
\mathbb E[D_1^2]\le \frac{C}{m^2},
\qquad
\mathbb E[B_1^2]\le C,
\]
and \(D_1,B_1\) are independent of \(R_1\).

For the weighted \(L^2\) bound, squaring the influence inequality and
multiplying by \(1+|R_1|^4\) gives
\begin{align}
&\mathbb E\!\left[
(1+|R_1|^4)
\|\hat\eta_m-\hat\eta_m^{(-1)}\|^2
\Ind{\mathcal E_1}
\right] \notag\\
&\quad\le
C\,\mathbb E[(1+|R_1|^4)D_1^2]
+
\frac{C}{m^2}\mathbb E[(1+|R_1|^4)B_1^2] \notag\\
&\qquad
+
\frac{C}{m^2}
\mathbb E\!\left[
(1+|R_1|^4)(R_1^2+R_1^4)
\Ind{R_1\ge r_{\theta,2\alpha}(x)}
\right].
\label{eq:app_weighted_l2_decomposition}
\end{align}
The first term satisfies
\[
\mathbb E[(1+|R_1|^4)D_1^2]
=
\mathbb E[1+|R_1|^4]\,
\mathbb E[D_1^2]
\le
\frac{C}{m^2},
\]
using independence and \(\mathbb E[R_1^4]\le M_R\). The second term satisfies
\[
\frac{1}{m^2}\mathbb E[(1+|R_1|^4)B_1^2]
=
\frac{1}{m^2}
\mathbb E[1+|R_1|^4]\,
\mathbb E[B_1^2]
\le
\frac{C}{m^2}.
\]
For the third term, the modeled Gaussian upper tail implies that, for every
fixed \(q\ge1\),
\[
\mathbb E\!\left[
|R_1|^q
\Ind{R_1\ge r_{\theta,2\alpha}(x)}
\right]
\le C_q.
\]
Therefore the third term in
\eqref{eq:app_weighted_l2_decomposition} is also at most \(C/m^2\). This proves
\eqref{eq:app_weighted_leave_one_out_l2}.

For the weighted \(L^1\) bound, the same influence inequality gives
\begin{align}
&\mathbb E\!\left[
(1+|R_1|^2)
\|\hat\eta_m-\hat\eta_m^{(-1)}\|
\Ind{\mathcal E_1}
\right] \notag\\
&\quad\le
C\,\mathbb E[(1+|R_1|^2)D_1]
+
\frac{C}{m}\mathbb E[(1+|R_1|^2)B_1] \notag\\
&\qquad
+
\frac{C}{m}
\mathbb E\!\left[
(1+|R_1|^2)(|R_1|+R_1^2)
\Ind{R_1\ge r_{\theta,2\alpha}(x)}
\right].
\end{align}
The first term is bounded by
\[
\mathbb E[(1+|R_1|^2)D_1]
=
\mathbb E[1+|R_1|^2]\mathbb E[D_1]
\le
\frac{C}{m}.
\]
For the second term, by independence and Cauchy--Schwarz,
\[
\frac{1}{m}\mathbb E[(1+|R_1|^2)B_1]
\le
\frac{1}{m}
\bigl(\mathbb E[(1+|R_1|^2)^2]\bigr)^{1/2}
\bigl(\mathbb E[B_1^2]\bigr)^{1/2}
\le
\frac{C}{m}.
\]
The third term is \(O(m^{-1})\) by the Gaussian upper-tail moment bound above.
This proves \eqref{eq:app_weighted_leave_one_out_l1}.

The unweighted local \(L^2\) and \(L^1\) bounds follow from the weighted bounds
because the weights are at least one on \(\mathcal E_1\). It remains only to
control the complement \(\mathcal E_1^c\). By Cauchy--Schwarz, the empirical
tail moment envelope in \Cref{lem:empirical_tail_moment_envelope}, and
\(\mathbb P(\mathcal E_1^c)\le Ce^{-cm}\),
\[
\mathbb E\!\left[
\|\hat\eta_m-\hat\eta_m^{(-1)}\|^2\Ind{\mathcal E_1^c}
\right]
\le
Ce^{-cm}.
\]
Therefore,
\[
\mathbb E\bigl[
\|\hat\eta_m-\hat\eta_m^{(-1)}\|^2
\bigr]
\le
\frac{C}{m^2}.
\]
The \(L^1\) bound follows by Cauchy--Schwarz:
\[
\mathbb E\bigl[
\|\hat\eta_m-\hat\eta_m^{(-1)}\|
\bigr]
\le
\left(
\mathbb E\|\hat\eta_m-\hat\eta_m^{(-1)}\|^2
\right)^{1/2}
\le
\frac{C}{m}.
\]
This completes the proof.
\end{proof}

The next lemma isolates the same-batch error of the direct plug-in estimator.
Although \(\hat\eta_m\) is computed from the same samples used in \(P_m\), the
plug-in fluctuation remains \(O(\log N/m)\) in second moment, and the resulting
empirical-process bias is only \(O(\sqrt{\log N}/m)\).

\begin{lemma}
\label{lem:same_batch_plugin_bounds}
Under Assumptions~\ref{assum:gaussian-tail-model},
\ref{assum:bounded_score}, and \ref{assum:nondegenerate_scale}, for all
sufficiently large \(m\),
\begin{equation}
\label{eq:app_direct_plugin_l2_bound}
\mathbb E\Bigl[
\bigl\|
P_m(\varphi_{\hat\eta_m}-\varphi_\eta)
\bigr\|^2
\Bigr]
\le
C\frac{\log N}{m},
\end{equation}
and
\begin{equation}
\label{eq:app_same_batch_bias_bound}
\left\|
\mathbb E\bigl[(P_m-P)\varphi_{\hat\eta_m}\bigr]
\right\|
\le
C\frac{\sqrt{\log N}}{m}.
\end{equation}
\end{lemma}

\begin{proof}
The main issue is that \(\hat\eta_m\) is computed from the same batch used in
\(P_m\), so \(\hat\eta_m\) is not independent of the summands. We remove this
dependence by inserting the leave-one-out vector \(\hat\eta_m^{(-i)}\), which
is independent of \(y_i\), and then control the leave-one-out perturbation.

We first prove \eqref{eq:app_direct_plugin_l2_bound}. By Jensen's inequality,
\[
\bigl\|
P_m(\varphi_{\hat\eta_m}-\varphi_\eta)
\bigr\|^2
\le
P_m\|\varphi_{\hat\eta_m}-\varphi_\eta\|^2.
\]
By exchangeability, it is enough to control the expectation of the first
summand. Decompose
\[
\varphi_{\hat\eta_m}(y_1)-\varphi_\eta(y_1)
=
\bigl(\varphi_{\hat\eta_m^{(-1)}}(y_1)-\varphi_\eta(y_1)\bigr)
+
\bigl(\varphi_{\hat\eta_m}(y_1)-\varphi_{\hat\eta_m^{(-1)}}(y_1)\bigr).
\]

For the first term, \(\hat\eta_m^{(-1)}\) is independent of \(y_1\). Therefore,
using \Cref{lem:shaped_score_local_regularity} and
\Cref{lem:empirical_tail_vector_bounds} for sample size \(m-1\),
\[
\begin{aligned}
\mathbb E\bigl[
\|\varphi_{\hat\eta_m^{(-1)}}(y_1)-\varphi_\eta(y_1)\|^2
\bigr]
&=
\mathbb E\Big[
P\|\varphi_{\hat\eta_m^{(-1)}}-\varphi_\eta\|^2
\Big] \\
&\le
C\log N\,
\mathbb E\|\hat\eta_m^{(-1)}-\eta\|^2
+
C\log N\,e^{-cm} \\
&\le
C\frac{\log N}{m}.
\end{aligned}
\]
Here and below, bad events are events on which an empirical tail vector leaves
\(B(\eta,\delta_{\mathrm{reg}})\). On such events we use the following crude
bound. By the definition of \(\varphi_{\eta'}\) in
\eqref{eq:app_phi_eta_def},
\[
\|\varphi_{\hat\eta_m}(y)\|^2
\le
\frac{G^2}{\alpha^2}
\left|
\widetilde R_{\hat\eta_m}(R(y))
\right|^2,
\]
where we used Assumption~\ref{assum:bounded_score}. By the explicit form of the
shaped reward,
\[
\widetilde R_{\hat\eta_m}(u)
=
(u-\hat r_m)
+
\frac{\widetilde c_N}{2\hat\sigma_m}
\left(
(u-\hat\mu_m)^2-(\hat r_m-\hat\mu_m)^2
\right).
\]
Using \eqref{eq:app_ctilde_log_bound} and the clipping
\(\hat\sigma_m\ge\varepsilon_\sigma\), we obtain
\[
\left|
\widetilde R_{\hat\eta_m}(u)
\right|
\le
C\sqrt{\log N}
\left(
1+|u|^2+|\hat r_m|^2+|\hat\mu_m|^2
\right).
\]
Consequently,
\begin{equation}
\label{eq:app_phi_bad_event_crude_bound}
\|\varphi_{\hat\eta_m}(y)\|^2
\le
C\log N
\left(
1+|R(y)|^4+|\hat r_m|^4+|\hat\mu_m|^4
\right).
\end{equation}
The empirical terms in this bound are controlled by the moment envelope in
\Cref{lem:empirical_tail_moment_envelope}, while the \(R(y)\) term is controlled
by Assumption~\ref{assum:gaussian-tail-model}, namely
\(\mathbb E[R^4]\le M_R\). Combining
\eqref{eq:app_phi_bad_event_crude_bound} with the exponential localization
probability gives a \(C\log N e^{-cm}\) contribution in squared bounds and a
\(C\sqrt{\log N}e^{-cm}\) contribution in first-moment bounds.

For the second term, the pointwise local bound
\eqref{eq:app_phi_pointwise_local_bound} and the weighted leave-one-out bound
\eqref{eq:app_weighted_leave_one_out_l2} give
\[
\begin{aligned}
&\mathbb E\bigl[
\|\varphi_{\hat\eta_m}(y_1)
-\varphi_{\hat\eta_m^{(-1)}}(y_1)\|^2
\bigr] \\
&\qquad\le
C\log N\,
\mathbb E\!\left[
(1+|R_1|^4)
\|\hat\eta_m-\hat\eta_m^{(-1)}\|^2
\Ind{\hat\eta_m,\hat\eta_m^{(-1)}\in B(\eta,\delta_{\mathrm{reg}})}
\right]
+
C\log N\,e^{-cm} \\
&\qquad\le
C\frac{\log N}{m^2}
+
C\log N\,e^{-cm}.
\end{aligned}
\]
Combining the two terms gives
\[
\mathbb E\bigl[
\|\varphi_{\hat\eta_m}(y_1)-\varphi_\eta(y_1)\|^2
\bigr]
\le
C\frac{\log N}{m}.
\]
And thus:
\begin{align*}
\mathbb E\Bigl[
\bigl\|
P_m(\varphi_{\hat\eta_m}-\varphi_\eta)
\bigr\|^2
\Bigr] \leq \frac{1}{m}\sum_{i=1}^{m} \mathbb E\bigl[\|\varphi_{\hat\eta_m}(y_i)-\varphi_\eta(y_i)\|^2\bigr] = \mathbb E\bigl[\|\varphi_{\hat\eta_m}(y_1)-\varphi_\eta(y_1)\|^2\bigr] \leq C\frac{\log N}{m},
\end{align*}
which proves \eqref{eq:app_direct_plugin_l2_bound}.

We now prove the same-batch bias bound
\eqref{eq:app_same_batch_bias_bound}. By exchangeability,
\[
\mathbb E\bigl[(P_m-P)\varphi_{\hat\eta_m}\bigr]
=
\mathbb E\bigl[
\varphi_{\hat\eta_m}(y_1)-P\varphi_{\hat\eta_m}
\bigr].
\]
Insert the leave-one-out vector \(\hat\eta_m^{(-1)}\). Since
\(\hat\eta_m^{(-1)}\) is independent of \(y_1\),
\[
\mathbb E\bigl[
\varphi_{\hat\eta_m^{(-1)}}(y_1)
-
P\varphi_{\hat\eta_m^{(-1)}}
\bigr]
=
0.
\]
Therefore,
\[
\begin{aligned}
\mathbb E\bigl[(P_m-P)\varphi_{\hat\eta_m}\bigr]
&=
\mathbb E\bigl[
\varphi_{\hat\eta_m}(y_1)-\varphi_{\hat\eta_m^{(-1)}}(y_1)
\bigr] \\
&\quad
-
\mathbb E\bigl[
P(\varphi_{\hat\eta_m}-\varphi_{\hat\eta_m^{(-1)}})
\bigr].
\end{aligned}
\]

For the first term, the pointwise local bound and
\eqref{eq:app_weighted_leave_one_out_l1} give
\[
\begin{aligned}
\mathbb E\bigl[
\|\varphi_{\hat\eta_m}(y_1)
-\varphi_{\hat\eta_m^{(-1)}}(y_1)\|
\bigr]
&\le
C\sqrt{\log N}\,
\mathbb E\!\left[
(1+|R_1|^2)
\|\hat\eta_m-\hat\eta_m^{(-1)}\|
\Ind{\hat\eta_m,\hat\eta_m^{(-1)}\in B(\eta,\delta_{\mathrm{reg}})}
\right]
+
C\sqrt{\log N}\,e^{-cm} \\
&\le
C\frac{\sqrt{\log N}}{m}.
\end{aligned}
\]
For the second term, Jensen's inequality and
\Cref{lem:shaped_score_local_regularity} imply
\[
\begin{aligned}
\mathbb E\bigl[
\|P(\varphi_{\hat\eta_m}-\varphi_{\hat\eta_m^{(-1)}})\|
\bigr]
&\le
\mathbb E\Big[
\bigl(P\|\varphi_{\hat\eta_m}
-\varphi_{\hat\eta_m^{(-1)}}\|^2\bigr)^{1/2}
\Big] \\
&\le
C\sqrt{\log N}\,
\mathbb E\|\hat\eta_m-\hat\eta_m^{(-1)}\|
+
C\sqrt{\log N}\,e^{-cm} \\
&\le
C\frac{\sqrt{\log N}}{m}.
\end{aligned}
\]
Combining the two bounds proves
\eqref{eq:app_same_batch_bias_bound}.
\end{proof}

We're now ready to prove the main result on the direct plug-in estimator.

\begin{proof}[Proof of \Cref{thm:direct_plugin_estimator}]
Fix \(x\) and \(\theta\). Throughout the proof, write
\[
\widehat g_m^{\mathrm{dir}}(\theta;x)=P_m\varphi_{\hat\eta_m},
\qquad
g_\theta(x)=P\varphi_\eta.
\]

\paragraph{Variance.}
Since variance is minimized by centering at the mean,
\[
\mathbb E\Big[
\bigl\|
\widehat g_m^{\mathrm{dir}}
-
\mathbb E[\widehat g_m^{\mathrm{dir}}]
\bigr\|^2
\Big]
\le
\mathbb E\bigl[
\|\widehat g_m^{\mathrm{dir}}-g_\theta(x)\|^2
\bigr].
\]
Decompose
\[
\widehat g_m^{\mathrm{dir}}-g_\theta(x)
=
(P_m\varphi_\eta-P\varphi_\eta)
+
P_m(\varphi_{\hat\eta_m}-\varphi_\eta).
\]
The first term comes from the Monte Carlo fluctuation of the ideal estimator \(P_m\varphi_\eta\), and we have:
\[
\mathbb E\bigl[
\|P_m\varphi_\eta-P\varphi_\eta\|^2
\bigr]
\le
\frac{1}{m}P\|\varphi_\eta\|^2
\le
C\frac{\log N}{m},
\]
where we used \Cref{lem:shaped_score_local_regularity}. The plug-in term is
controlled by \Cref{lem:same_batch_plugin_bounds}:
\[
\mathbb E\Bigl[
\bigl\|
P_m(\varphi_{\hat\eta_m}-\varphi_\eta)
\bigr\|^2
\Bigr]
\le
C\frac{\log N}{m}.
\]
Therefore,
\[
\mathbb E\Big[
\bigl\|
\widehat g_m^{\mathrm{dir}}
-
\mathbb E[\widehat g_m^{\mathrm{dir}}]
\bigr\|^2
\Big]
\le
C\frac{\log N}{m}.
\]

\paragraph{Bias.}
We decompose the bias as
\[
\mathbb E[\widehat g_m^{\mathrm{dir}}]-g_\theta(x)
=\mathbb E\bigl[(P_m-P)\varphi_{\hat\eta_m}\bigr] + \mathbb E\bigl[P\varphi_{\hat\eta_m}-P\varphi_\eta\bigr] = 
\mathbb E\bigl[(P_m-P)\varphi_{\hat\eta_m}\bigr]
+
\mathbb E\bigl[H(\hat\eta_m)-H(\eta)\bigr].
\]
The first term comes from evaluating the plug-in estimator on the same batch used to compute \(\hat\eta_m\), and is bounded by \Cref{lem:same_batch_plugin_bounds}:
\[
\left\|
\mathbb E\bigl[(P_m-P)\varphi_{\hat\eta_m}\bigr]
\right\|
\le
C\frac{\sqrt{\log N}}{m}.
\]

It remains to control the population plug-in term. Let
\[
\mathcal E_m:=\{\hat\eta_m\in B(\eta,\delta_{\mathrm{reg}})\}.
\]
where \(\delta_{\mathrm{reg}}\) is as in \Cref{lem:shaped_score_local_regularity}.
On \(\mathcal E_m\), \Cref{lem:shaped_score_local_regularity} gives
\[
H(\hat\eta_m)-H(\eta)
=
\dot H_\eta(\hat\eta_m-\eta)
+
\mathcal R_m,
\qquad
\|\mathcal R_m\|
\le
C\sqrt{\log N}\,\|\hat\eta_m-\eta\|^2.
\]
Therefore,
\[
\begin{aligned}
\mathbb E[H(\hat\eta_m)-H(\eta)]
&=
\dot H_\eta\bigl(\mathbb E[\hat\eta_m]-\eta\bigr)
+
\mathbb E[\mathcal R_m\Ind{\mathcal E_m}] \\
&\quad
-
\dot H_\eta\,
\mathbb E[(\hat\eta_m-\eta)\Ind{\mathcal E_m^c}]
+
\mathbb E[(H(\hat\eta_m)-H(\eta))\Ind{\mathcal E_m^c}].
\end{aligned}
\]
The first term is bounded by
\[
\|\dot H_\eta\|\,
\|\mathbb E[\hat\eta_m]-\eta\|
\le
C\sqrt{\log N}\cdot \frac1m.
\]
where we used \Cref{lem:shaped_score_local_regularity} and the bias bound in \Cref{lem:empirical_tail_vector_bounds}.
For the remainder term on the good event, by the second-moment bound in \Cref{lem:empirical_tail_vector_bounds} and $\Ind{\mathcal E_m}\le 1$,
\[
\mathbb E[\|\mathcal R_m\|\Ind{\mathcal E_m}]
\le
C\sqrt{\log N}\,
\mathbb E\|\hat\eta_m-\eta\|^2
\le
C\frac{\sqrt{\log N}}{m}.
\]
For the two complement terms, \(\mathbb P(\mathcal E_m^c)\le Ce^{-cm}\) by
\Cref{lem:empirical_tail_vector_bounds}. Using the empirical tail moment
envelope and the polynomial growth of \(H(\eta')=P\varphi_{\eta'}\), their total
contribution is
\[
C\sqrt{\log N}\,e^{-cm},
\]
which is absorbed into \(C\sqrt{\log N}/m\) for sufficiently large \(m\). Hence
\[
\left\|
\mathbb E[H(\hat\eta_m)-H(\eta)]
\right\|
\le
C\frac{\sqrt{\log N}}{m}.
\]
Combining this with the same-batch bias bound yields
\[
\bigl\|
\mathbb E[\widehat g_m^{\mathrm{dir}}(\theta;x)]
-
g_\theta(x)
\bigr\|
\le
C\frac{\sqrt{\log N}}{m}.
\]
This proves \eqref{eq:direct_bias_bound}.
\end{proof}

\subsection{Auxiliary results for fixed-order debiasing}
\label{app:fixed_order_aux}

\paragraph{Cross-fitted base estimator.}
We first record the basic properties of the cross-fitted plug-in gradient. We
state the result for a generic base sample size \(n\), since the fixed-order
estimator will later combine estimators evaluated at prefix sizes
\(n_1,\dots,n_J\). For such an \(n\), draw two independent batches
\[
\mathcal D_A=\{y_i^{(A)}\}_{i=1}^n,
\qquad
\mathcal D_B=\{y_i^{(B)}\}_{i=1}^n,
\]
with all samples drawn i.i.d. from \(\curpolicy(\cdot\mid x)\). Let
\[
\hat\eta_n^{(A)}
=
(\hat r_n^{(A)},\hat\mu_n^{(A)},\hat\sigma_n^{(A)})
\]
be the empirical tail vector computed from batch \(A\) in the same way as
\eqref{eq:empirical_tail_vector}. Define
\[
P_n^{(B)}f
:=
\frac1n\sum_{i=1}^n f(y_i^{(B)}).
\]
Then the cross-fitted base estimator can be written as
\begin{equation}
\label{eq:app_cross_fitted_base_compact}
G_n^{\mathrm{cf}}(\theta;x)
:=
P_n^{(B)}\varphi_{\hat\eta_n^{(A)}}.
\end{equation}

\begin{lemma}
\label{lem:cross_fitted_identity_base_variance}
Suppose Assumptions~\ref{assum:gaussian-tail-model},
\ref{assum:bounded_score}, and \ref{assum:nondegenerate_scale} hold. Then, 
\begin{equation}
\label{eq:app_cross_fitted_identity}
\mathbb E\!\left[
G_n^{\mathrm{cf}}(\theta;x)
\right]
=
\mathbb E\!\left[
H(\hat\eta_n)
\right],
\end{equation}
where \(\hat\eta_n\) denotes an empirical tail vector computed from an
independent size-\(n\) batch. Moreover, for all sufficiently large \(n\),
\begin{equation}
\label{eq:app_cross_fitted_base_variance}
\mathbb E\!\left[
\left\|
G_n^{\mathrm{cf}}(\theta;x)
-
\mathbb E[G_n^{\mathrm{cf}}(\theta;x)]
\right\|^2
\right]
\le
C\frac{\log N}{n}.
\end{equation}
The constant \(C\) depends only on \((\alpha,G,\sigma_{\min},M_R)\).
\end{lemma}

\begin{proof}
Let
\[
\mathcal F_A:=\sigma(\mathcal D_A)
\]
be the sigma-field generated by the tail-estimation batch. Conditional on
\(\mathcal F_A\), the vector \(\hat\eta_n^{(A)}\) is fixed and the evaluation
samples in batch \(B\) are i.i.d. Therefore,
\[
\mathbb E\!\left[
G_n^{\mathrm{cf}}(\theta;x)
\,\middle|\,
\mathcal F_A
\right]
=
P\varphi_{\hat\eta_n^{(A)}}
=
H(\hat\eta_n^{(A)}).
\]
Taking expectations gives
\[
\mathbb E[G_n^{\mathrm{cf}}(\theta;x)]
=
\mathbb E[H(\hat\eta_n^{(A)})].
\]
Since \(\hat\eta_n^{(A)}\) has the same distribution as an independent
size-\(n\) empirical tail vector \(\hat\eta_n\), this proves
\eqref{eq:app_cross_fitted_identity}.

It remains to prove the variance bound. Write
\[
H_A:=H(\hat\eta_n^{(A)}).
\]
We decompose
\[
G_n^{\mathrm{cf}}-\mathbb E[G_n^{\mathrm{cf}}]
=
\bigl(G_n^{\mathrm{cf}}-H_A\bigr)
+
\bigl(H_A-\mathbb E H_A\bigr).
\]
Thus
\begin{equation}
\label{eq:app_cf_variance_decomposition}
\mathbb E\bigl[
\|G_n^{\mathrm{cf}}-\mathbb E G_n^{\mathrm{cf}}\|^2
\bigr]
\le
2\mathbb E\bigl[
\|G_n^{\mathrm{cf}}-H_A\|^2
\bigr]
+
2\mathbb E\bigl[
\|H_A-\mathbb E H_A\|^2
\bigr].
\end{equation}

The first term captures the noise of evaluating the plug-in estimator on batch \(B\).
Conditional on
\(\mathcal F_A\),
\[
G_n^{\mathrm{cf}}-H_A
=
\frac1n
\sum_{i=1}^n
\left(
\varphi_{\hat\eta_n^{(A)}}(y_i^{(B)})
-
P\varphi_{\hat\eta_n^{(A)}}
\right),
\]
where the summands are i.i.d. and mean zero. Hence
\begin{align}
\mathbb E\!\left[
\|G_n^{\mathrm{cf}}-H_A\|^2
\,\middle|\,
\mathcal F_A
\right]
&\le
\frac1n
P\|\varphi_{\hat\eta_n^{(A)}}\|^2 .
\label{eq:app_cf_conditional_noise}
\end{align}
We claim that
\begin{equation}
\label{eq:app_cf_phi_second_moment}
\mathbb E\!\left[
P\|\varphi_{\hat\eta_n^{(A)}}\|^2
\right]
\le
C\log N.
\end{equation}
Indeed, on the event
\(\hat\eta_n^{(A)}\in B(\eta,\delta_{\mathrm{reg}})\), by
\Cref{lem:shaped_score_local_regularity},
\[
P\|\varphi_{\hat\eta_n^{(A)}}\|^2
\le
2P\|\varphi_\eta\|^2
+
2P\|\varphi_{\hat\eta_n^{(A)}}-\varphi_\eta\|^2
\le
C\log N\bigl(1+\|\hat\eta_n^{(A)}-\eta\|^2\bigr).
\]
Taking expectations and using \Cref{lem:empirical_tail_vector_bounds} gives a
\(C\log N\) bound on this event. On the complement, the polynomial growth of
\(\varphi_{\eta'}\), the clipped-scale envelope, and
\Cref{lem:empirical_tail_moment_envelope} give a contribution
\(C\log N\,e^{-cn}\), which is absorbed into the same bound. This proves
\eqref{eq:app_cf_phi_second_moment}. Combining
\eqref{eq:app_cf_conditional_noise} and
\eqref{eq:app_cf_phi_second_moment},
\begin{equation}
\label{eq:app_cf_eval_noise_bound}
\mathbb E\bigl[
\|G_n^{\mathrm{cf}}-H_A\|^2
\bigr]
\le
C\frac{\log N}{n}.
\end{equation}

We now bound the fluctuation of the conditional mean \(H_A\). Since variance is
bounded by the second moment around any fixed vector,
\[
\mathbb E\bigl[
\|H_A-\mathbb E H_A\|^2
\bigr]
\le
\mathbb E\bigl[
\|H_A-H(\eta)\|^2
\bigr].
\]
On the event \(\hat\eta_n^{(A)}\in B(\eta,\delta_{\mathrm{reg}})\),
\Cref{lem:shaped_score_local_regularity} gives
\[
\|H(\hat\eta_n^{(A)})-H(\eta)\|^2
\le
C\log N\,\|\hat\eta_n^{(A)}-\eta\|^2.
\]
Taking expectations and applying
\Cref{lem:empirical_tail_vector_bounds} yields
\[
\mathbb E\bigl[
\|H(\hat\eta_n^{(A)})-H(\eta)\|^2
\Ind{\hat\eta_n^{(A)}\in B(\eta,\delta_{\mathrm{reg}})}
\bigr]
\le
C\frac{\log N}{n}.
\]
The complement contributes \(C\log N\,e^{-cn}\) by the same moment-envelope
argument as above. Therefore,
\begin{equation}
\label{eq:app_cf_conditional_mean_variance}
\mathbb E\bigl[
\|H_A-\mathbb E H_A\|^2
\bigr]
\le
C\frac{\log N}{n}.
\end{equation}

Plugging \eqref{eq:app_cf_eval_noise_bound} and
\eqref{eq:app_cf_conditional_mean_variance} into
\eqref{eq:app_cf_variance_decomposition} proves
\eqref{eq:app_cross_fitted_base_variance}.
\end{proof}

\paragraph{Compatible-prefix construction and cancellation weights.}
The fixed-order estimator combines several cross-fitted base estimators
evaluated at different prefix lengths. The compatible-prefix construction has
two goals: first, the prefix lengths should have a common rounding offset
\[
\delta_n:=\lceil \alpha n\rceil-\alpha n,
\]
so that the bias expansions at different prefixes have aligned coefficients;
second, the weights should cancel the first \(k-1\) powers in these expansions.

Throughout this paragraph, assume
\[
\alpha=p/q\in(0,1/2)\cap\mathbb Q,
\]
where \(p,q\) are positive coprime integers. For an internal split budget \(n\),
define
\[
\rho_j:=\frac12+\frac{j}{2(J+1)},
\qquad
j=1,\dots,J,
\]
and set
\begin{equation}
\label{eq:app_compatible_prefix_construction}
n_j
:=
q\left\lfloor \frac{\rho_j n}{q}\right\rfloor,
\qquad
j=1,\dots,J.
\end{equation}

\begin{lemma}
\label{lem:compatible_prefixes}
For the prefix lengths \(n_1,\dots,n_J\) defined in
\eqref{eq:app_compatible_prefix_construction}, the following hold for all
sufficiently large \(n\). First, the prefixes are distinct and satisfy
\begin{equation}
\label{eq:app_prefix_size_bounds}
\frac{n}{2}
\le
n_j
\le
n,
\qquad
j=1,\dots,J.
\end{equation}
Second, they have common rounding offset
\begin{equation}
\label{eq:app_common_rounding_offset}
\delta_{n_j}
=
\lceil \alpha n_j\rceil-\alpha n_j
=
0,
\qquad
j=1,\dots,J.
\end{equation}
Finally, with \(z_j:=n/n_j\), there exists a constant \(C>0\), depending only
on \((\alpha,J)\), such that
\begin{equation}
\label{eq:app_prefix_ratio_limit}
\left|z_j-\frac1{\rho_j}\right|
\le
\frac{C}{n},
\qquad
j=1,\dots,J.
\end{equation}
In particular, \(z_1,\dots,z_J\) converge to \(J\) distinct positive constants.
\end{lemma}

\begin{proof}
Since \(n_j\) is a multiple of \(q\), we have
\[
\alpha n_j
=
\frac{p}{q}n_j
\in\mathbb Z.
\]
Therefore
\[
\delta_{n_j}
=
\lceil \alpha n_j\rceil-\alpha n_j
=
0,
\]
which proves \eqref{eq:app_common_rounding_offset}.

By the definition of \(n_j\), there exists \(\xi_j\in[0,q)\) such that
\[
n_j=\rho_j n-\xi_j.
\]
Since \(\rho_j<1\), we have \(n_j\le n\). Also,
\[
\rho_j\ge \rho_1=\frac12+\frac{1}{2(J+1)}.
\]
Thus, for all \(n\ge 2q(J+1)\),
\[
n_j
\ge
\left(\frac12+\frac{1}{2(J+1)}\right)n-q
\ge
\frac{n}{2}.
\]
This proves \eqref{eq:app_prefix_size_bounds}.

For adjacent indices,
\[
n_{j+1}-n_j
=
(\rho_{j+1}-\rho_j)n-(\xi_{j+1}-\xi_j).
\]
Since \(\xi_j,\xi_{j+1}\in[0,q)\) and
\[
\rho_{j+1}-\rho_j=\frac{1}{2(J+1)},
\]
we have \(n_{j+1}-n_j>0\) for all \(n>2q(J+1)\). Hence
\(n_1,\dots,n_J\) are distinct.

Finally,
\[
z_j-\frac1{\rho_j}
=
\frac{n}{\rho_jn-\xi_j}
-
\frac1{\rho_j}
=
\frac{\xi_j}{\rho_j(\rho_jn-\xi_j)}.
\]
For sufficiently large \(n\), \eqref{eq:app_prefix_size_bounds} gives
\[
\rho_jn-\xi_j=n_j\ge \frac{n}{2},
\]
and \(\rho_j\ge1/2\). Hence
\[
\left|z_j-\frac1{\rho_j}\right|
\le
\frac{2\xi_j}{\rho_j n}
\le
\frac{4q}{n}.
\]
This proves \eqref{eq:app_prefix_ratio_limit}. The limits \(1/\rho_j\) are
distinct because the \(\rho_j\)'s are distinct.
\end{proof}

For the compatible prefixes above and a fixed debiasing order \(k\le J\), define
\[
z_j:=\frac n{n_j},
\qquad
\mathsf A\in\mathbb R^{k\times J},
\qquad
\mathsf A_{\ell j}:=z_j^\ell,
\qquad
\ell=0,\dots,k-1,
\]
and let \(e_0=(1,0,\dots,0)^\top\). For all sufficiently large \(n\), define
\begin{equation}
\label{eq:app_min_norm_weights}
w^{(k,J)}(n)
:=
\mathsf A^\top(\mathsf A\mathsf A^\top)^{-1}e_0.
\end{equation}

The next lemma records the deterministic bounds used in the bias and variance
analysis.

\begin{lemma}
\label{lem:moment_cancellation_weight_complexity}
Fix \(k\le J\). Let \(w^{(k,J)}(n)\) be the weights defined in
\eqref{eq:app_min_norm_weights}. For each fixed \(k,J\), there exists
\(C_{k,J}<\infty\) such that, for all sufficiently large \(n\), the matrix
\(\mathsf A\) has full row rank and the weights are well defined. Moreover,
\begin{equation}
\label{eq:app_weight_complexity_uniform_bound}
\|w^{(k,J)}(n)\|_2^2\le C_{k,J}.
\end{equation}

The weights satisfy
\begin{equation}
\label{eq:app_moment_cancellation_constraints}
\sum_{j=1}^J w_j^{(k,J)}(n)=1,
\qquad
\sum_{j=1}^J w_j^{(k,J)}(n)z_j^\ell=0,
\qquad
\ell=1,\dots,k-1.
\end{equation}
Consequently, for every \(\ell=1,\dots,k-1\),
\begin{equation}
\label{eq:app_cancel_mj_powers}
\sum_{j=1}^J w_j^{(k,J)}(n)n_j^{-\ell}=0.
\end{equation}
Finally,
\begin{equation}
\label{eq:app_weight_l1_bound}
\|w^{(k,J)}(n)\|_1
\le
C_{k,J},
\end{equation}
\begin{equation}
\label{eq:app_remainder_weight_bound}
\sum_{j=1}^J |w_j^{(k,J)}(n)|n_j^{-k}
\le
C_{k,J}n^{-k},
\end{equation}
and
\begin{equation}
\label{eq:app_variance_weight_bound}
\|w^{(k,J)}(n)\|_2^2
\sum_{j=1}^J\frac1{n_j}
\le
C_{k,J}\frac1n.
\end{equation}
The constant \(C_{k,J}\) depends only on \((\alpha,k,J)\).
\end{lemma}

\begin{proof}
By \Cref{lem:compatible_prefixes},
\[
z_j=\frac{n}{n_j}\to \frac1{\rho_j},
\qquad
j=1,\dots,J,
\]
and the limits \(1/\rho_1,\dots,1/\rho_J\) are distinct. Hence \(\mathsf A\) converges
entrywise to the Vandermonde-type matrix
\[
\mathsf A^\star_{\ell j}:=
\left(\frac1{\rho_j}\right)^\ell,
\qquad
\ell=0,\dots,k-1,
\qquad
j=1,\dots,J.
\]
Since \(k\le J\) and the nodes \(1/\rho_1,\dots,1/\rho_J\) are distinct,
\(\mathsf A^\star\) has full row rank. Therefore \(\mathsf A\) has full row rank for all
sufficiently large \(n\), and \(\mathsf A\mathsf A^\top\) is invertible.

Moreover,
\[
\mathsf A\mathsf A^\top\to \mathsf A^\star(\mathsf A^\star)^\top,
\]
whose smallest eigenvalue is positive. Thus \((\mathsf A\mathsf A^\top)^{-1}\) is uniformly
bounded for all sufficiently large \(n\). Since \(z_j\le2\) by
\Cref{lem:compatible_prefixes}, \(\mathsf A\) is uniformly bounded for fixed
\(k,J\). Therefore
\[
w^{(k,J)}(n)=\mathsf A^\top(\mathsf A\mathsf A^\top)^{-1}e_0
\]
satisfies \(\|w^{(k,J)}(n)\|_2^2\le C_{k,J}\). This proves
\eqref{eq:app_weight_complexity_uniform_bound}.

Next, by definition,
\[
w^{(k,J)}(n)=\mathsf A^\top(\mathsf A\mathsf A^\top)^{-1}e_0,
\]
so
\[
\mathsf A w^{(k,J)}(n)
=
\mathsf A\mathsf A^\top(\mathsf A\mathsf A^\top)^{-1}e_0
=
e_0.
\]
Writing this coordinatewise gives
\[
\sum_{j=1}^J w_j^{(k,J)}(n)=1,
\qquad
\sum_{j=1}^J w_j^{(k,J)}(n)z_j^\ell=0,
\qquad
\ell=1,\dots,k-1.
\]
This proves \eqref{eq:app_moment_cancellation_constraints}.

For \(\ell=1,\dots,k-1\), since
\[
n_j^{-\ell}
=
n^{-\ell}z_j^\ell,
\]
the cancellation constraints imply
\[
\sum_{j=1}^J w_j^{(k,J)}(n)n_j^{-\ell}
=
n^{-\ell}
\sum_{j=1}^J w_j^{(k,J)}(n)z_j^\ell
=
0.
\]
This proves \eqref{eq:app_cancel_mj_powers}.

The \(L^1\) bound follows from Cauchy--Schwarz:
\[
\|w^{(k,J)}(n)\|_1
\le
\sqrt J\,\|w^{(k,J)}(n)\|_2
\le
C_{k,J}.
\]
This proves \eqref{eq:app_weight_l1_bound}.

By \Cref{lem:compatible_prefixes}, \(n_j\ge n/2\). Therefore,
\[
\sum_{j=1}^J |w_j^{(k,J)}(n)|n_j^{-k}
\le
\left(\frac2n\right)^k
\sum_{j=1}^J |w_j^{(k,J)}(n)|
\le
C_{k,J}n^{-k}.
\]
This proves \eqref{eq:app_remainder_weight_bound}.

Finally,
\[
\sum_{j=1}^J\frac1{n_j}
\le
\frac{2J}{n}.
\]
Thus
\[
\|w^{(k,J)}(n)\|_2^2
\sum_{j=1}^J\frac1{n_j}
\le
\|w^{(k,J)}(n)\|_2^2\frac{2J}{n}
\le
C_{k,J}\frac1n.
\]
This proves \eqref{eq:app_variance_weight_bound}.
\end{proof}

To keep the split-budget rollout cost equal to \(2n\), we draw two
independent batches \(\mathcal D_A,\mathcal D_B\), each of size \(n\), once.
After fixing an arbitrary ordering of the samples in each batch, for each
prefix length \(n_j\), \(G_{n_j}^{\mathrm{cf}}\) is computed using the first
\(n_j\) samples of \(\mathcal D_A\) to estimate the tail vector and the first
\(n_j\) samples of \(\mathcal D_B\) to evaluate the score-function average.
Thus the \(J\) estimators are nested-prefix estimators, not estimators built
from fresh independent batches for each \(j\). These nested estimators are
generally dependent across \(j\), but the variance bound below only uses a
Cauchy--Schwarz argument and does not require independence across prefixes.

The split-budget fixed-order estimator is
\begin{equation}
\label{eq:app_split_budget_fixed_order_estimator}
\widetilde g_{n,k,J}(\theta;x)
:=
\sum_{j=1}^J
w_j^{(k,J)}(n)\,
G_{n_j}^{\mathrm{cf}}(\theta;x).
\end{equation}
The main-text total-budget estimator is
\begin{equation}
\label{eq:app_total_budget_fixed_order_estimator}
\widehat g_{m,k,J}^{\mathrm{fo}}(\theta;x)
:=
\widetilde g_{\lfloor m/2\rfloor,k,J}(\theta;x).
\end{equation}
Thus \(\widehat g_{m,k,J}^{\mathrm{fo}}\) uses at most \(m\) rollouts, namely
\(2\lfloor m/2\rfloor\) rollouts. This is the estimator referred to in
\Cref{thm:fixed_order_debiased_gradient}.

\paragraph{Finite-order plug-in expansion.}
We now prove the finite-order bias expansion for the cross-fitted base
estimator. This is the technical ingredient behind the fixed-order debiasing
step: it shows that the bias of a single cross-fitted plug-in gradient admits a
finite expansion in powers of the sample size, with coefficients depending on
the rounding offset
\[
\delta_n:=\lceil \alpha n\rceil-\alpha n.
\]
The compatible-prefix construction above will later ensure that this rounding
offset is the same across all prefix lengths.

We state the result for a generic base sample size \(n\), since the
fixed-order estimator combines cross-fitted estimators evaluated at several
prefix sizes \(n_1,\dots,n_J\). By
\Cref{lem:cross_fitted_identity_base_variance}, the expectation of the
cross-fitted base estimator is
\[
\mathbb E[G_n^{\mathrm{cf}}(\theta;x)]
=
\mathbb E[H(\hat\eta_n)].
\]
Thus the goal of this paragraph is to expand
\(\mathbb E[H(\hat\eta_n)]\).

We first clarify the notations used in the finite-order expansion. Fix a
generic base sample size \(n\), and write
\[
F:=F_{\theta,x},
\qquad
Q:=F^{-1},
\qquad
u_\alpha:=1-\alpha,
\qquad
k_n:=\lceil \alpha n\rceil,
\qquad
\delta_n:=k_n-\alpha n .
\]
Let
\[
U_i:=F(R_i),
\qquad i=1,\dots,n,
\]
and let \(U_{(1)}\le\cdots\le U_{(n)}\) be their order statistics. The empirical
upper-tail threshold in uniform coordinates is
\[
B_n:=U_{(n-k_n+1)}.
\]
Thus, if \(R_{(1)}\le\cdots\le R_{(n)}\) are the reward order statistics, then
\[
\hat r_n=R_{(n-k_n+1)}=Q(B_n).
\]

For a threshold level \(b\in(1-2\alpha,1)\), define the population upper-tail
first and second moment functions
\begin{equation}
\label{eq:app_M1_M2_def}
M_1(b):=\frac1{1-b}\int_b^1 Q(u)\,du,
\qquad
M_2(b):=\frac1{1-b}\int_b^1 Q(u)^2\,du.
\end{equation}
At \(b=u_\alpha\), these recover the population upper-tail moments:
\[
M_1(u_\alpha)=\mu,
\qquad
M_2(u_\alpha)-M_1(u_\alpha)^2=\sigma^2.
\]

For the empirical upper tail, let \(\mathcal I_n\) be the indices of the top
\(k_n\) rewards, and write
\[
\hat M_{1,n}:=\hat\mu_n,
\qquad
\hat M_{2,n}:=\frac1{k_n}\sum_{i\in\mathcal I_n}R_i^2.
\]
Then the clipped empirical scale is
\[
\hat\sigma_n
=
\max\left\{
\sqrt{\hat M_{2,n}-\hat M_{1,n}^2},
\varepsilon_\sigma
\right\},
\qquad
\hat\eta_n=(\hat r_n,\hat\mu_n,\hat\sigma_n).
\]

We measure the empirical tail vector by the local error vector
\begin{equation}
\label{eq:app_delta_def}
\Delta_n
:=
(\Delta_{0,n},\Delta_{1,n},\Delta_{2,n})
:=
\bigl(
B_n-u_\alpha,\,
\hat M_{1,n}-M_1(B_n),\,
\hat M_{2,n}-M_2(B_n)
\bigr).
\end{equation}
Here \(\Delta_{0,n}\) is the threshold order-statistic error, while
\(\Delta_{1,n}\) and \(\Delta_{2,n}\) are the empirical first- and second-moment
errors around their conditional upper-tail means given the threshold.

Since \(M_1\) and \(M_2\) are continuous near \(u_\alpha\), and
\[
M_2(u_\alpha)-M_1(u_\alpha)^2
=
\sigma^2
\ge
4\varepsilon_\sigma^2,
\]
we can choose a fixed \(\tau>0\) such that
\[
[u_\alpha-\tau,u_\alpha+\tau]\subset(1-2\alpha,1)
\]
and, for every \(d=(d_0,d_1,d_2)\) with
\[
\|d\|_1:=|d_0|+|d_1|+|d_2|\le\tau,
\]
we have
\begin{equation}
\label{eq:app_local_scale_positive}
M_2(u_\alpha+d_0)+d_2
-
\bigl(M_1(u_\alpha+d_0)+d_1\bigr)^2
\ge
2\varepsilon_\sigma^2.
\end{equation}

Define the local event
\[
\mathcal L_n:=\{\|\Delta_n\|_1\le\tau\}.
\]
Finally, define the deterministic reconstruction map
\begin{equation}
\label{eq:app_reconstruction_map_def}
\mathcal G(d)
:=
H\!\left(
Q(u_\alpha+d_0),\,
M_1(u_\alpha+d_0)+d_1,\,
\sqrt{
M_2(u_\alpha+d_0)+d_2
-
\bigl(M_1(u_\alpha+d_0)+d_1\bigr)^2
}
\right).
\end{equation}
This map reconstructs a tail vector from the local errors \(d\), and then
applies \(H\).

The proof is organized through the following four lemmas. First, we show that, on an
exponentially high-probability local event, the empirical plug-in quantity
\(H(\hat\eta_n)\) is exactly equal to the deterministic reconstruction
\(\mathcal G(\Delta_n)\). Second, we Taylor expand \(\mathcal G\) around the
origin. Third, we bound the high moments of the local error \(\Delta_n\), which
controls the Taylor remainder. Finally, we expand the mixed moments of
\(\Delta_n\) in powers of \(n^{-1}\), with coefficients depending on the
rounding offset \(\delta_n\).

\begin{lemma}
\label{lem:local_reconstruction_empirical_tail_vector}
Under Assumptions~\ref{assum:gaussian-tail-model} and
\ref{assum:nondegenerate_scale}, there exist constants \(C,c>0\), depending
only on \((\alpha,G,\sigma_{\min},M_R)\), such that for all sufficiently large \(n\),
\begin{equation}
\label{eq:app_expansion_local_event}
\mathbb P(\mathcal L_n^c)\le Ce^{-cn}.
\end{equation}
Moreover, the reconstruction map \(\mathcal G\) in
\eqref{eq:app_reconstruction_map_def} is smooth on the local neighborhood
\(\{d:\|d\|_1\le\tau\}\), satisfies
\[
\mathcal G(0)=H(\eta),
\]
and, on \(\mathcal L_n\),
\begin{equation}
\label{eq:app_H_hateta_as_G_delta}
H(\hat\eta_n)=\mathcal G(\Delta_n).
\end{equation}
\end{lemma}

\begin{proof}
The smoothness of \(\mathcal G\) follows from the smoothness of \(Q,M_1,M_2\)
on \([u_\alpha-\tau,u_\alpha+\tau]\) and from
\eqref{eq:app_local_scale_positive}, which keeps the square-root argument
uniformly positive.

We first prove the localization bound. Let
\[
J_\tau:=[u_\alpha-\tau/3,u_\alpha+\tau/3].
\]
By the order-statistic localization bound in
\eqref{eq:app_order_localization},
\[
\mathbb P(B_n\notin J_\tau)
=
\mathbb P(|\Delta_{0,n}|>\tau/3)
\le
Ce^{-cn}.
\]
Conditional on \(B_n=b\in J_\tau\), the empirical upper tail consists of the
threshold value \(Q(b)\) and \(k_n-1\) unordered i.i.d. copies of
\[
Y_b:=Q(b+(1-b)W),
\qquad
W\sim\mathrm{Unif}(0,1).
\]
Hence, with \(Y_{1,b},\dots,Y_{k_n-1,b}\) i.i.d. copies of \(Y_b\),
\begin{align}
\Delta_{1,n}
&=
\frac{Q(b)-M_1(b)}{k_n}
+
\frac1{k_n}
\sum_{\ell=1}^{k_n-1}
\bigl(Y_{\ell,b}-M_1(b)\bigr),
\label{eq:app_delta1_conditional_reconstruction}\\
\Delta_{2,n}
&=
\frac{Q(b)^2-M_2(b)}{k_n}
+
\frac1{k_n}
\sum_{\ell=1}^{k_n-1}
\bigl(Y_{\ell,b}^2-M_2(b)\bigr).
\label{eq:app_delta2_conditional_reconstruction}
\end{align}
Since \(J_\tau\Subset(1-2\alpha,1)\), the variables \(Y_b\) have uniform
Gaussian-tail envelopes over \(b\in J_\tau\). Also, \(Q(b)\), \(M_1(b)\), and
\(M_2(b)\) are uniformly bounded on \(J_\tau\), so the deterministic threshold
terms in
\eqref{eq:app_delta1_conditional_reconstruction}--\eqref{eq:app_delta2_conditional_reconstruction}
are \(O(k_n^{-1})\). For all sufficiently large \(n\), these deterministic
terms are at most \(\tau/6\). Applying the same conditional Bernstein bound as
in \eqref{eq:app_tail_empirical_local_concentration}, uniformly over
\(b\in J_\tau\), gives
\[
\mathbb P\!\left(
|\Delta_{1,n}|+|\Delta_{2,n}|>\frac{2\tau}{3}
\,\middle|\,
B_n=b
\right)
\le
Ce^{-cn}.
\]
Combining this with the bound for \(B_n\) proves
\eqref{eq:app_expansion_local_event}.

It remains to verify the reconstruction identity. On \(\mathcal L_n\), by the
definition of \(\Delta_n\),
\[
Q(u_\alpha+\Delta_{0,n})
=
Q(B_n)
=
\hat r_n,
\]
and
\[
M_1(u_\alpha+\Delta_{0,n})+\Delta_{1,n}
=
M_1(B_n)+\Delta_{1,n}
=
\hat M_{1,n}
=
\hat\mu_n.
\]
Similarly,
\[
M_2(u_\alpha+\Delta_{0,n})+\Delta_{2,n}
=
M_2(B_n)+\Delta_{2,n}
=
\hat M_{2,n}.
\]
Therefore the third coordinate reconstructed by \(\mathcal G\) is
\[
\sqrt{
\hat M_{2,n}-\hat\mu_n^2
}
=
\hat\sigma_n^{\mathrm{raw}}.
\]
By \eqref{eq:app_local_scale_positive}, this raw scale is larger than
\(\varepsilon_\sigma\) on \(\mathcal L_n\). Hence the clipping is inactive, and
\[
\hat\sigma_n
=
\hat\sigma_n^{\mathrm{raw}}.
\]
Thus
\[
\mathcal G(\Delta_n)
=
H(\hat r_n,\hat\mu_n,\hat\sigma_n)
=
H(\hat\eta_n)
\qquad
\text{on }\mathcal L_n.
\]

Finally,
\[
\mathcal G(0)
=
H\!\left(
Q(u_\alpha),
M_1(u_\alpha),
\sqrt{M_2(u_\alpha)-M_1(u_\alpha)^2}
\right)
=
H(r,\mu,\sigma)
=
H(\eta).
\]
This completes the proof.
\end{proof}

The previous lemma reduces the plug-in expansion to a deterministic Taylor
expansion of the reconstruction map \(\mathcal G\) around the origin. The next
lemma proves this expansion directly from the fixed-domain integral
representation of \(H\), which can be seen as an extension to \Cref{lem:shaped_score_local_regularity}.

\begin{lemma}
\label{lem:local_taylor_reconstructed_plugin_map}
Fix an integer \(L\ge1\). Under Assumption~\ref{assum:gaussian-tail-model}, \ref{assum:bounded_score} and \ref{assum:nondegenerate_scale},
there exist coefficient tensors
\[
\widetilde{\mathcal H}_\nu(x,\theta),
\qquad
1\le |\nu|\le 2L-1,
\]
such that, for every \(d\) with \(\|d\|_1\le \tau\),
\begin{equation}
\label{eq:app_G_delta_local_expansion}
\mathcal G(d)
=
H(\eta)
+
\sum_{1\le|\nu|\le 2L-1}
\widetilde{\mathcal H}_\nu(x,\theta)d^\nu
+
\mathcal E_{\mathcal G,L}(d),
\end{equation}
where \(d^\nu=\prod_{i=1}^3 d_i^{\nu_i}\). Moreover,
\begin{equation}
\label{eq:app_G_delta_remainder}
\max_{1\le|\nu|\le 2L-1}
\|\widetilde{\mathcal H}_\nu(x,\theta)\|
\le
C_0A_0^L L!\,\sqrt{\log N},
\qquad
\|\mathcal E_{\mathcal G,L}(d)\|
\le
C_0A_0^L L!\,\sqrt{\log N}\,\|d\|^{2L}.
\end{equation}
The constants \(C_0,A_0>0\) depend only on
\((\alpha,G,\sigma_{\min},M_R)\), uniformly over \(x,\theta,L\).
\end{lemma}

\begin{proof}
Let
\[
\mathcal D_\tau:=\{d:\|d\|_1\le\tau\}.
\]
For \(d\in\mathcal D_\tau\), write the reconstructed tail vector inside
\(\mathcal G\) as
\[
\eta_d:=(r_d,\mu_d,\sigma_d),
\]
where
\[
r_d:=Q(u_\alpha+d_0),
\qquad
\mu_d:=M_1(u_\alpha+d_0)+d_1,
\]
and
\[
\sigma_d
:=
\sqrt{
M_2(u_\alpha+d_0)+d_2
-
\bigl(M_1(u_\alpha+d_0)+d_1\bigr)^2
}.
\]
Then
\[
\mathcal G(d)=H(\eta_d),
\qquad
\eta_0=\eta.
\]

We first record derivative bounds for the reconstructed coordinates. On
\([u_\alpha-\tau,u_\alpha+\tau]\), the Gaussian-tail model gives
\[
Q(u)=\bar\mu+\bar\sigma\Phi^{-1}(u).
\]
Since this interval is compactly contained in \((1-2\alpha,1)\), the functions
\(Q,M_1,M_2\) have finite-order derivative bounds of the form
\[
\sup_{u\in[u_\alpha-\tau,u_\alpha+\tau]}
\bigl(
|Q^{(q)}(u)|+|M_1^{(q)}(u)|+|M_2^{(q)}(u)|
\bigr)
\le
C A_0^L q!,
\qquad
q\le 2L.
\]
For \(r_d\) and \(\mu_d\), this immediately gives
\[
\max_{|\beta|\le 2L}
\bigl(
|\partial_d^\beta r_d|
+
|\partial_d^\beta \mu_d|
\bigr)
\le
C A_0^L|\beta|!.
\]
Indeed, \(r_d\) depends only on \(d_0\), while
\(\mu_d=M_1(u_\alpha+d_0)+d_1\) is linear in \(d_1\) and independent of \(d_2\).

For the scale coordinate, define
\[
v_d:=\sigma_d^2
=
M_2(u_\alpha+d_0)+d_2
-
\bigl(M_1(u_\alpha+d_0)+d_1\bigr)^2.
\]
By \eqref{eq:app_local_scale_positive},
\[
v_d\ge 2\varepsilon_\sigma^2,
\qquad d\in\mathcal D_\tau.
\]
The derivative bounds on \(M_1,M_2\) imply
\[
\max_{|\beta|\le 2L}|\partial_d^\beta v_d|
\le
C A_0^L|\beta|!.
\]
Since the functions \(z^{1/2}\) and \(z^{-1/2}\) are smooth on
\([2\varepsilon_\sigma^2,\infty)\), the finite-order chain rule gives
\[
\max_{|\beta|\le 2L}
\bigl(
|\partial_d^\beta \sigma_d|
+
|\partial_d^\beta \sigma_d^{-1}|
\bigr)
\le
C A_0^L|\beta|!.
\]
Combining the preceding displays, we have
\begin{equation}
\label{eq:app_reconstructed_tail_derivative_bounds}
\max_{|\beta|\le 2L}
\left(
|\partial_d^\beta r_d|
+
|\partial_d^\beta \mu_d|
+
|\partial_d^\beta \sigma_d|
+
|\partial_d^\beta \sigma_d^{-1}|
\right)
\le
C A_0^L|\beta|!,
\qquad d\in\mathcal D_\tau .
\end{equation}

Using the change of variables \(t=r_d+s\), the moving boundary in \(H\) becomes
fixed:
\begin{equation}
\label{eq:app_G_fixed_domain_representation}
\mathcal G(d)
=
\frac1\alpha
\int_0^\infty
\widetilde R_{\eta_d}(r_d+s)\,
b(r_d+s)\,ds,
\end{equation}
where
\[
b(t):=a(t)p_{\theta,x}(t),
\qquad
a(t):=\mathbb E[S(y)\mid R(y)=t].
\]

We now bound the two factors in the integrand. First, set
\[
a_d:=r_d-\mu_d.
\]
By the definition of the shaped reward,
\begin{align}
\widetilde R_{\eta_d}(r_d+s)
&=
s+
\frac{\widetilde c_N}{2\sigma_d}
\Bigl((a_d+s)^2-a_d^2\Bigr) \notag\\
&=
s+
\frac{\widetilde c_N}{2\sigma_d}
\bigl(2a_ds+s^2\bigr).
\label{eq:app_G_shaped_reward_coordinate}
\end{align}
From \eqref{eq:app_reconstructed_tail_derivative_bounds},
\[
|a_d|\le C,
\qquad
|\partial_d^\beta a_d|\le C A_0^L|\beta|!,
\qquad
|\partial_d^\beta\sigma_d^{-1}|\le C A_0^L|\beta|!.
\]
Let \(q:=|\kappa|\le 2L\). If \(q=0\), then
\[
|\widetilde R_{\eta_d}(r_d+s)|
\le
C\sqrt{\log N}\,(1+s^2),
\]
using \(|\widetilde c_N|\le C\sqrt{\log N}\). If \(q\ge1\), the derivative of
the first term \(s\) is zero, and the product rule gives
\[
\partial_d^\kappa
\widetilde R_{\eta_d}(r_d+s)
=
\frac{\widetilde c_N}{2}
\sum_{\gamma\le\kappa}
{\kappa\choose\gamma}
\partial_d^\gamma(\sigma_d^{-1})
\,
\partial_d^{\kappa-\gamma}(2a_ds+s^2).
\]
For each \(\gamma\le\kappa\),
\[
\left|
\partial_d^\gamma(\sigma_d^{-1})
\right|
\le
C A_0^L|\gamma|!,
\]
and
\[
\left|
\partial_d^{\kappa-\gamma}(2a_ds+s^2)
\right|
\le
C A_0^L|\kappa-\gamma|!\,(1+s^2).
\]
Moreover,
\[
{\kappa\choose\gamma}|\gamma|!\,|\kappa-\gamma|!
\le
C\,|\kappa|!,
\]
and the number of multi-indices \(\gamma\le\kappa\) is at most \(C^{|\kappa|}\).
Absorbing this finite product-rule factor into \(A_0^L\), we obtain
\begin{equation}
\label{eq:app_G_shaped_reward_derivative_bound}
\left|
\partial_d^\kappa
\widetilde R_{\eta_d}(r_d+s)
\right|
\le
C A_0^L q!\sqrt{\log N}\,(1+s^2),
\qquad q=|\kappa|\le 2L.
\end{equation}

Next, by the likelihood-ratio identity for the reward density, for any smooth test
function \(\psi\),
\[
\int \psi(t)\nabla_\theta p_{\theta,x}(t)\,dt
=
\nabla_\theta \mathbb E_\theta[\psi(R)]
=
\mathbb E_\theta[\psi(R)S(y)]
=
\int \psi(t)a(t)p_{\theta,x}(t)\,dt,
\]
where \(a(t):=\mathbb E_\theta[S(y)\mid R=t]\). Hence
\[
a(t)p_{\theta,x}(t)=\nabla_\theta p_{\theta,x}(t)
\]
for \(p_{\theta,x}\)-almost every \(t\).
And on the modeled Gaussian tail,
\[
p_{\theta,x}(t)
=
\frac1{\bar\sigma}
\phi\!\left(\frac{t-\bar\mu}{\bar\sigma}\right).
\]
Writing
\[
z(t):=\frac{t-\bar\mu}{\bar\sigma},
\]
we get
\[
b(t)
=
p_{\theta,x}(t)
\left[
\frac{z(t)}{\bar\sigma}\nabla_\theta\bar\mu
+
\frac{z(t)^2-1}{\bar\sigma}\nabla_\theta\bar\sigma
\right].
\]
Thus \(b(t)\) is a Gaussian density times a polynomial in \(z(t)\), with
uniformly controlled coefficients.

This structure is preserved by \(t\)-differentiation. Indeed, if
\(b^{(\ell)}(t)=p_{\theta,x}(t)P_\ell(z(t))\) for some vector-valued polynomial
\(P_\ell\), then
\[
b^{(\ell+1)}(t)
=
p_{\theta,x}(t)\,
\frac{P_\ell'(z(t))-z(t)P_\ell(z(t))}{\bar\sigma},
\]
so, by induction,
\[
b^{(\ell)}(t)=p_{\theta,x}(t)P_\ell(z(t)),
\qquad
\ell\le 2L,
\]
where \(P_\ell\) has degree at most \(\ell+2\) and coefficients controlled by
\((\alpha,G,\sigma_{\min},M_R)\) and the order \(L\).

Now set \(t=r_d+s\). Since \(r_d\) stays in a compact part of the modeled
Gaussian tail, and since \(|\bar\mu|+\bar\sigma\le C\) and
\(\bar\sigma\ge\sigma_{\min}\), we have
\[
p_{\theta,x}(r_d+s)\le C e^{-cs^2},
\qquad
|z(r_d+s)|\le C(1+s),
\qquad s\ge0.
\]
Therefore, for every \(\ell\le 2L\),
\begin{equation}
\label{eq:app_G_b_t_derivative_bound}
\left\|
b^{(\ell)}(r_d+s)
\right\|
\le
C_0A_0^L L!\,(1+s^{C_L})e^{-cs^2},
\qquad s\ge0.
\end{equation}

By the product rule, \eqref{eq:app_G_shaped_reward_derivative_bound} and
\eqref{eq:app_G_b_t_derivative_bound} imply
\begin{equation}
\label{eq:app_G_integrand_derivative_envelope}
\left\|
\partial_d^\kappa
\Bigl[
\widetilde R_{\eta_d}(r_d+s)\,
b(r_d+s)
\Bigr]
\right\|
\le
C_0A_0^L q!L!\,\sqrt{\log N}\,
(1+s^{C_L})e^{-cs^2},
\qquad q=|\kappa|\le 2L .
\end{equation}
The right-hand side is integrable in \(s\), uniformly over \(d\in\mathcal D_\tau\).
Therefore, differentiating under the integral sign in
\eqref{eq:app_G_fixed_domain_representation},
\begin{align}
\left\|
\partial_d^\kappa\mathcal G(d)
\right\|
&\le
\frac1\alpha
\int_0^\infty
\left\|
\partial_d^\kappa
\Bigl[
\widetilde R_{\eta_d}(r_d+s)b(r_d+s)
\Bigr]
\right\|\,ds \notag\\
&\le
C_0A_0^L q!L!\,\sqrt{\log N},
\qquad q=|\kappa|\le 2L.
\label{eq:app_G_derivative_bound}
\end{align}
where we used $\int_0^\infty (1+s^{C_L})e^{-cs^2} ds \le C_L$.

Taylor's theorem at the origin gives
\[
\mathcal G(d)
=
\mathcal G(0)
+
\sum_{1\le|\nu|\le 2L-1}
\frac{\partial_d^\nu\mathcal G(0)}{\nu!}d^\nu
+
\mathcal E_{\mathcal G,L}(d).
\]
Since \(\mathcal G(0)=H(\eta)\), set
\[
\widetilde{\mathcal H}_\nu(x,\theta)
:=
\frac{\partial_d^\nu\mathcal G(0)}{\nu!}.
\]
Using \eqref{eq:app_G_derivative_bound},
\[
\|\widetilde{\mathcal H}_\nu(x,\theta)\|
\le
\frac{C_0A_0^L|\nu|!L!\sqrt{\log N}}{\nu!}.
\]
Since
\[
\frac{|\nu|!}{\nu!}\le 3^{|\nu|}\le 9^L,
\]
we enlarge \(A_0\) and obtain
\[
\max_{1\le|\nu|\le 2L-1}
\|\widetilde{\mathcal H}_\nu(x,\theta)\|
\le
C_0A_0^L L!\sqrt{\log N}.
\]

For the remainder, Taylor's theorem gives
\[
\|\mathcal E_{\mathcal G,L}(d)\|
\le
\sup_{\|u\|_1\le\|d\|_1}
\max_{|\kappa|=2L}
\|\partial_u^\kappa\mathcal G(u)\|
\sum_{|\kappa|=2L}\frac{|d^\kappa|}{\kappa!}.
\]
By \eqref{eq:app_G_derivative_bound},
\[
\sup_{\|u\|_1\le\|d\|_1}
\max_{|\kappa|=2L}
\|\partial_u^\kappa\mathcal G(u)\|
\le
C_0A_0^L(2L)!L!\sqrt{\log N}.
\]
The multinomial formula gives
\[
\sum_{|\kappa|=2L}\frac{|d^\kappa|}{\kappa!}
=
\frac{(|d_0|+|d_1|+|d_2|)^{2L}}{(2L)!}
\le
\frac{C^L\|d\|^{2L}}{(2L)!}.
\]
Thus the factor \((2L)!\) cancels, and after enlarging \(A_0\),
\[
\|\mathcal E_{\mathcal G,L}(d)\|
\le
C_0A_0^L L!\sqrt{\log N}\,\|d\|^{2L}.
\]
This proves \eqref{eq:app_G_delta_local_expansion} and
\eqref{eq:app_G_delta_remainder}.
\end{proof}

We next control the Taylor remainder after substituting the random local error
\(\Delta_n\). The following lemma gives the high-moment estimate needed for
the remainder term.

\begin{lemma}
\label{lem:local_tail_error_high_moments}
Fix an integer \(L\ge1\). Under Assumptions~\ref{assum:gaussian-tail-model} and
\ref{assum:nondegenerate_scale}, for all sufficiently large \(n\),
\begin{equation}
\label{eq:app_delta_high_moment_bound}
\mathbb E\!\left[
\|\Delta_n\|^{2L}\Ind{\mathcal L_n}
\right]
\le
C_L n^{-L},
\end{equation}
where \(\mathcal L_n\) is the local event defined above. The constant
\(C_L\) depends only on \(L\) and \((\alpha,G,\sigma_{\min},M_R)\).
\end{lemma}

\begin{proof}
Since \(\Delta_n\in\mathbb R^3\),
\begin{equation}
\label{eq:app_delta_norm_coordinate_bound}
\|\Delta_n\|^{2L}
\le
C_L\left(
|\Delta_{0,n}|^{2L}
+
|\Delta_{1,n}|^{2L}
+
|\Delta_{2,n}|^{2L}
\right).
\end{equation}

We first control the threshold coordinate. Since
\[
B_n=U_{(n-k_n+1)},
\qquad
k_n=\lceil \alpha n\rceil,
\]
we have
\[
B_n\sim \mathrm{Beta}(n-k_n+1,k_n).
\]
Moreover,
\[
\mathbb E[B_n]
=
\frac{n-k_n+1}{n+1}.
\]
Using \(u_\alpha=1-\alpha\) and \(k_n=\alpha n+\delta_n\), with
\(\delta_n\in[0,1)\), we get
\[
\left|
\mathbb E[B_n]-u_\alpha
\right|
=
\left|
\frac{n-k_n+1}{n+1}
-
(1-\alpha)
\right|
=
\frac{|\alpha-\delta_n|}{n+1}
\le
\frac{1}{n+1}.
\]
Standard beta concentration gives
\begin{equation}
\label{eq:app_beta_high_central_moment}
\mathbb E\!\left[
|B_n-\mathbb E B_n|^{2L}
\right]
\le
C_L n^{-L}.
\end{equation}
Indeed, \(B_n\) is an order statistic at a fixed quantile considering that $n-k_n+1$ and $k_n$ are both
of order \(n\). Hence the beta-binomial tail
identity and Chernoff's inequality imply
\[
\mathbb P\!\left(
|B_n-\mathbb E B_n|\ge t
\right)
\le
C e^{-c n t^2},
\qquad
t\ge0.
\]
Therefore, by the layer-cake formula,
\[
\begin{aligned}
\mathbb E\!\left[
|B_n-\mathbb E B_n|^{2L}
\right]
&=
2L\int_0^\infty
t^{2L-1}
\mathbb P\!\left(
|B_n-\mathbb E B_n|\ge t
\right)\,dt  \\
&\le
C_L\int_0^\infty
t^{2L-1}e^{-c n t^2}\,dt
\le
C_L n^{-L}.
\end{aligned}
\]

Therefore,
\begin{align}
\mathbb E|\Delta_{0,n}|^{2L}
&=
\mathbb E|B_n-u_\alpha|^{2L} \notag\\
&\le
C_L
\mathbb E|B_n-\mathbb E B_n|^{2L}
+
C_L|\mathbb E B_n-u_\alpha|^{2L} \notag\\
&\le
C_L n^{-L}.
\label{eq:app_delta0_high_moment}
\end{align}

It remains to control the empirical moment coordinates. Let
\[
I_\tau:=[u_\alpha-\tau,u_\alpha+\tau].
\]
On \(\mathcal L_n\), we have \(B_n\in I_\tau\). Conditional on \(B_n=b\in
I_\tau\), the empirical upper tail consists of the threshold value \(Q(b)\) and
\[
K_n:=k_n-1
\]
unordered i.i.d. copies
\[
Y_{\ell,b}:=Q(b+(1-b)W_\ell),
\qquad
W_\ell\stackrel{\mathrm{i.i.d.}}{\sim}\mathrm{Unif}(0,1),
\qquad
\ell=1,\dots,K_n.
\]
Thus
\[
M_1(b)=\mathbb E[Y_{\ell,b}],
\qquad
M_2(b)=\mathbb E[Y_{\ell,b}^2],
\]
and
\begin{align}
\Delta_{1,n}
&=
\frac{Q(b)-M_1(b)}{k_n}
+
\frac1{k_n}
\sum_{\ell=1}^{K_n}
\bigl(Y_{\ell,b}-M_1(b)\bigr),
\label{eq:app_delta1_conditional_decomp_high_moment}\\
\Delta_{2,n}
&=
\frac{Q(b)^2-M_2(b)}{k_n}
+
\frac1{k_n}
\sum_{\ell=1}^{K_n}
\bigl(Y_{\ell,b}^2-M_2(b)\bigr).
\label{eq:app_delta2_conditional_decomp_high_moment}
\end{align}

Since \(I_\tau\Subset(1-2\alpha,1)\), the variables \(Y_{\ell,b}\) are in the
modeled Gaussian upper tail, uniformly over \(b\in I_\tau\). Hence
\begin{equation}
\label{eq:app_Yb_uniform_moments}
\sup_{b\in I_\tau}
\mathbb E|Y_{\ell,b}|^{4L}
\le
C_L.
\end{equation}
Consequently,
\[
\sup_{b\in I_\tau}
\mathbb E|Y_{\ell,b}-M_1(b)|^{2L}
+
\sup_{b\in I_\tau}
\mathbb E|Y_{\ell,b}^2-M_2(b)|^{2L}
\le
C_L.
\]

For all sufficiently large \(n\),
\(
k_n\ge \alpha n,
\)
and \(
K_n=k_n-1\ge \frac{\alpha n}{2}.
\)
Also, since \(Q,M_1,M_2\) are bounded on \(I_\tau\),
\[
\sup_{b\in I_\tau}
\left|
\frac{Q(b)-M_1(b)}{k_n}
\right|^{2L}
+
\sup_{b\in I_\tau}
\left|
\frac{Q(b)^2-M_2(b)}{k_n}
\right|^{2L}
\le
C_L n^{-2L}
\le
C_L n^{-L}.
\]

We use the Rosenthal moment bound \citep{rosenthal1970subspaces}: if
\(Z_1,\dots,Z_K\) are i.i.d. mean-zero variables with
\(\mathbb E|Z_1|^{2L}\le C_L\), then
\[
\mathbb E\left|
\frac1K\sum_{\ell=1}^K Z_\ell
\right|^{2L}
\le
C_L K^{-L}.
\]
Applying this to
\[
Z_{\ell,b}^{(1)}:=Y_{\ell,b}-M_1(b),
\qquad
Z_{\ell,b}^{(2)}:=Y_{\ell,b}^2-M_2(b),
\]
and using \(K_n\ge \alpha n/2\), we obtain uniformly over \(b\in I_\tau\),
\[
\mathbb E\!\left[
\left|
\frac1{K_n}\sum_{\ell=1}^{K_n} Z_{\ell,b}^{(1)}
\right|^{2L}
+
\left|
\frac1{K_n}\sum_{\ell=1}^{K_n} Z_{\ell,b}^{(2)}
\right|^{2L}
\,\middle|\,
B_n=b
\right]
\le
C_L n^{-L}.
\]
Since
\[
\frac1{k_n}\sum_{\ell=1}^{K_n} Z_{\ell,b}^{(j)}
=
\frac{K_n}{k_n}
\left(
\frac1{K_n}\sum_{\ell=1}^{K_n} Z_{\ell,b}^{(j)}
\right),
\qquad
0\le \frac{K_n}{k_n}\le 1,
\]
the same bound holds with \(k_n^{-1}\sum_{\ell=1}^{K_n}\) in place of
\(K_n^{-1}\sum_{\ell=1}^{K_n}\). Combining this with
\eqref{eq:app_delta1_conditional_decomp_high_moment} and
\eqref{eq:app_delta2_conditional_decomp_high_moment}, we get
\begin{equation}
\label{eq:app_delta12_conditional_high_moment}
\mathbb E\!\left[
|\Delta_{1,n}|^{2L}
+
|\Delta_{2,n}|^{2L}
\,\middle|\,
B_n=b
\right]
\le
C_L n^{-L},
\qquad b\in I_\tau .
\end{equation}

Finally, since \(\mathcal L_n\subseteq\{B_n\in I_\tau\}\),
\begin{align}
&\mathbb E\!\left[
\bigl(|\Delta_{1,n}|^{2L}+|\Delta_{2,n}|^{2L}\bigr)
\Ind{\mathcal L_n}
\right] 
\le
\mathbb E\!\left[
\Ind{B_n\in I_\tau}
\mathbb E\!\left[
|\Delta_{1,n}|^{2L}+|\Delta_{2,n}|^{2L}
\,\middle|\,
B_n
\right]
\right]
\le
C_L n^{-L}.
\label{eq:app_delta12_local_high_moment}
\end{align}
Together with
\eqref{eq:app_delta_norm_coordinate_bound} and
\eqref{eq:app_delta0_high_moment}, this proves
\eqref{eq:app_delta_high_moment_bound}.
\end{proof}

The high-moment lemma controls the Taylor remainder. It remains to identify the
expectation of each polynomial term in the Taylor expansion. The next lemma
shows that the localized mixed moments of \(\Delta_n\) admit finite expansions
in powers of \(n^{-1}\), with coefficients depending on the rounding offset
\(\delta_n\).

\begin{lemma}
\label{lem:local_tail_error_mixed_moment_expansion}
Fix an integer \(L\ge1\). Under Assumptions~\ref{assum:gaussian-tail-model}
and \ref{assum:nondegenerate_scale}, for every multi-index
\[
\nu=(\nu_0,\nu_1,\nu_2),
\qquad
1\le |\nu|\le 2L-1,
\]
there exist coefficients
\[
c_{\nu,\ell}(\delta_n;x,\theta),
\qquad
\ell=1,\dots,L-1,
\]
depending on \(n\) only through
\(\delta_n=\lceil\alpha n\rceil-\alpha n\), such that, for all sufficiently
large \(n\),
\begin{equation}
\label{eq:app_delta_mixed_moment_expansion}
\left|
\mathbb E\!\left[
\Delta_n^\nu\Ind{\mathcal L_n}
\right]
-
\sum_{\ell=1}^{L-1}
c_{\nu,\ell}(\delta_n;x,\theta)n^{-\ell}
\right|
\le
C_L n^{-L}.
\end{equation}
Here
\[
\Delta_n^\nu
:=
\Delta_{0,n}^{\nu_0}
\Delta_{1,n}^{\nu_1}
\Delta_{2,n}^{\nu_2},
\]
and \(C_L\) depends only on \(L\) and \((\alpha,G,\sigma_{\min},M_R)\).
\end{lemma}

\begin{proof}
Let
\[
I_\tau:=[u_\alpha-\tau,u_\alpha+\tau].
\]
Since \(\mathcal L_n\subseteq\{B_n\in I_\tau\}\),
\[
\Ind{\mathcal L_n}-\Ind{B_n\in I_\tau}
=
-\Ind{\{B_n\in I_\tau\}\cap \mathcal L_n^c}.
\]
Therefore, by Cauchy--Schwarz,
\begin{align}
&\left|
\mathbb E[\Delta_n^\nu\Ind{\mathcal L_n}]
-
\mathbb E[\Delta_n^\nu\Ind{B_n\in I_\tau}]
\right| \notag\\
&\qquad\le
\mathbb E\!\left[
|\Delta_n^\nu|
\Ind{\{B_n\in I_\tau\}\cap \mathcal L_n^c}
\right] \notag\\
&\qquad\le
\left(
\mathbb E\!\left[
|\Delta_n^\nu|^2
\Ind{B_n\in I_\tau}
\right]
\right)^{1/2}
\mathbb P(\mathcal L_n^c)^{1/2}.
\end{align}
On \(B_n\in I_\tau\), the same conditional Gaussian-tail moment bounds used in
\Cref{lem:local_tail_error_high_moments} give
\[
\mathbb E\!\left[
|\Delta_n^\nu|^2\Ind{B_n\in I_\tau}
\right]
\le C_L.
\]
Together with \eqref{eq:app_expansion_local_event}, this yields
\begin{equation}
\label{eq:app_mixed_local_to_threshold_local}
\left|
\mathbb E[\Delta_n^\nu\Ind{\mathcal L_n}]
-
\mathbb E[\Delta_n^\nu\Ind{B_n\in I_\tau}]
\right|
\le
C_L e^{-cn},
\end{equation}
after decreasing \(c>0\) if necessary. Hence it is enough to expand
\(\mathbb E[\Delta_n^\nu\Ind{B_n\in I_\tau}]\).

Fix \(b\in I_\tau\), and condition on \(B_n=b\). Let
\[
K_n:=k_n-1=\alpha n+\delta_n-1,
\]
and define
\[
Y_{\ell,b}:=Q(b+(1-b)W_\ell),
\qquad
W_\ell\stackrel{\mathrm{i.i.d.}}{\sim}\mathrm{Unif}(0,1),
\qquad
\ell=1,\dots,K_n.
\]
Set
\[
A_{\ell,b}:=Y_{\ell,b}-M_1(b),
\qquad
C_{\ell,b}:=Y_{\ell,b}^2-M_2(b),
\]
and
\[
q_1(b):=Q(b)-M_1(b),
\qquad
q_2(b):=Q(b)^2-M_2(b).
\]
Then
\begin{equation}
\label{eq:app_delta_conditional_representation}
\Delta_{1,n}
=
\frac{q_1(b)+\sum_{\ell=1}^{K_n}A_{\ell,b}}{k_n},
\qquad
\Delta_{2,n}
=
\frac{q_2(b)+\sum_{\ell=1}^{K_n}C_{\ell,b}}{k_n}.
\end{equation}

We first expand the conditional moment of
\(\Delta_{1,n}^{\nu_1}\Delta_{2,n}^{\nu_2}\). Define
\[
S_{A,b}:=\sum_{\ell=1}^{K_n}A_{\ell,b},
\qquad
S_{C,b}:=\sum_{\ell=1}^{K_n}C_{\ell,b}.
\]
By \eqref{eq:app_delta_conditional_representation} and the binomial theorem,
\begin{align}
&\mathbb E\!\left[
\Delta_{1,n}^{\nu_1}\Delta_{2,n}^{\nu_2}
\,\middle|\,
B_n=b
\right] \notag\\
&\quad =
k_n^{-(\nu_1+\nu_2)}
\sum_{a=0}^{\nu_1}
\sum_{c=0}^{\nu_2}
\binom{\nu_1}{a}
\binom{\nu_2}{c}
q_1(b)^{\nu_1-a}
q_2(b)^{\nu_2-c}
\,
\mathbb E\!\left[
S_{A,b}^{a}S_{C,b}^{c}
\right].
\label{eq:app_conditional_delta12_binomial_expansion}
\end{align}

For fixed \(a,c\), expand the powers in
\(\mathbb E[S_{A,b}^{a}S_{C,b}^{c}]\). Group each resulting term by the number
\(r\) of distinct sample indices appearing in the product. Since
\[
\mathbb E[A_{\ell,b}]=0,
\qquad
\mathbb E[C_{\ell,b}]=0,
\]
any term in which some used index appears exactly once has zero expectation.
Thus every contributing index appears at least twice, and hence
\(r\le \lfloor(a+c)/2\rfloor\). Therefore there exist functions
\(\Psi_{a,c,r}(b;x,\theta)\) such that
\begin{equation}
\label{eq:app_centered_sum_collision_expansion}
\mathbb E[S_{A,b}^{a}S_{C,b}^{c}]
=
\sum_{r=0}^{\lfloor(a+c)/2\rfloor}
K_n^{\underline r}\,
\Psi_{a,c,r}(b;x,\theta),
\end{equation}
where
\[
K_n^{\underline r}:=K_n(K_n-1)\cdots(K_n-r+1),
\qquad
K_n^{\underline 0}:=1.
\]
Each \(\Psi_{a,c,r}\) is a finite linear combination of products of mixed
centered tail moments
\[
\mu_{p,q}(b)
:=
\mathbb E\!\left[
\bigl(Y_b-M_1(b)\bigr)^p
\bigl(Y_b^2-M_2(b)\bigr)^q
\right].
\]
These functions are smooth on \(I_\tau\). Indeed, under the Gaussian-tail model,
\[
Y_b=\bar\mu+\bar\sigma\Phi^{-1}(b+(1-b)W),
\]
and \(I_\tau\Subset(1-2\alpha,1)\). Hence all fixed-order truncated-Gaussian
moments and their \(b\)-derivatives are uniformly bounded on \(I_\tau\).

Substituting \eqref{eq:app_centered_sum_collision_expansion} into
\eqref{eq:app_conditional_delta12_binomial_expansion}, each summand is indexed
by a triple \((a,c,r)\) and has the form
\[
k_n^{-(\nu_1+\nu_2)}
K_n^{\underline r}
\psi_{a,c,r}(b;x,\theta),
\]
where
\[
\psi_{a,c,r}(b;x,\theta)
:=
\binom{\nu_1}{a}
\binom{\nu_2}{c}
q_1(b)^{\nu_1-a}
q_2(b)^{\nu_2-c}
\Psi_{a,c,r}(b;x,\theta).
\]
Here
\[
0\le r\le \left\lfloor\frac{a+c}{2}\right\rfloor
\le
\left\lfloor\frac{\nu_1+\nu_2}{2}\right\rfloor,
\]
and \(\psi_{a,c,r}(\cdot;x,\theta)\) is smooth on \(I_\tau\).

Let
\[
m_\nu:=\nu_1+\nu_2.
\]
Since
\[
k_n=\alpha n+\delta_n,
\qquad
K_n=\alpha n+\delta_n-1,
\]
we have
\[
k_n^{-m_\nu}K_n^{\underline r}
=
(\alpha n+\delta_n)^{-m_\nu}
\prod_{h=0}^{r-1}(\alpha n+\delta_n-1-h).
\]
Factoring out powers of \(n\),
\[
k_n^{-m_\nu}K_n^{\underline r}
=
n^{r-m_\nu}
G_r(n^{-1},\delta_n),
\]
where
\[
G_r(x,\delta)
:=
(\alpha+\delta x)^{-m_\nu}
\prod_{h=0}^{r-1}
\bigl(\alpha+(\delta-1-h)x\bigr).
\]
Since \(\alpha>0\) and \(\delta_n\in[0,1)\), \(G_r(x,\delta)\) is smooth in
\(x\) near \(0\), uniformly over \(\delta\in[0,1)\). Hence, for every fixed
\(L\),
\[
G_r(n^{-1},\delta_n)
=
\sum_{j=0}^{L-1}g_{r,j}(\delta_n)n^{-j}
+
\rho_{r,L}(n,\delta_n),
\qquad
|\rho_{r,L}(n,\delta_n)|\le C_L n^{-L}.
\]
Since \(r\le \lfloor m_\nu/2\rfloor\le m_\nu\), the prefactor
\(n^{r-m_\nu}\) is never a positive power of \(n\). Reindexing the preceding
expansion gives
\[
k_n^{-m_\nu}K_n^{\underline r}
=
\sum_{j=0}^{L-1}a_{r,j}(\delta_n)n^{-j}
+
\widetilde\rho_{r,L}(n,\delta_n),
\qquad
|\widetilde\rho_{r,L}(n,\delta_n)|\le C_Ln^{-L},
\]
with
\[
a_{r,j}(\delta_n)
=
\begin{cases}
g_{r,j-(m_\nu-r)}(\delta_n),
& j\ge m_\nu-r,\\
0,
& j<m_\nu-r.
\end{cases}
\]
Thus the first \(m_\nu-r\) coefficients are zero.

Since the number of triples \((a,c,r)\) is finite and depends only on \(L\), and
since all \(\psi_{a,c,r}\) are uniformly bounded on \(I_\tau\), we obtain,
uniformly for \(b\in I_\tau\),
\begin{equation}
\label{eq:app_conditional_delta12_power_expansion}
\mathbb E\!\left[
\Delta_{1,n}^{\nu_1}\Delta_{2,n}^{\nu_2}
\,\middle|\,
B_n=b
\right]
=
\sum_{j=0}^{L-1}
\psi_j(b,\delta_n;x,\theta)n^{-j}
+
\rho_L(b,n),
\qquad
|\rho_L(b,n)|\le C_L n^{-L},
\end{equation}
where each \(\psi_j(\cdot,\delta_n;x,\theta)\) is smooth on \(I_\tau\).

Using the tower property,
\begin{align}
\mathbb E[\Delta_n^\nu\Ind{B_n\in I_\tau}]
&=
\mathbb E\!\left[
(B_n-u_\alpha)^{\nu_0}
\Ind{B_n\in I_\tau}
\mathbb E\!\left[
\Delta_{1,n}^{\nu_1}\Delta_{2,n}^{\nu_2}
\,\middle|\,
B_n
\right]
\right] \notag\\
&=
\sum_{j=0}^{L-1}
n^{-j}
\mathbb E\!\left[
(B_n-u_\alpha)^{\nu_0}
\psi_j(B_n,\delta_n;x,\theta)
\Ind{B_n\in I_\tau}
\right]
+
O_L(n^{-L}).
\label{eq:app_mixed_after_conditioning}
\end{align}
Here \(O_L(n^{-L})\) denotes a term bounded in absolute value by
\(C_Ln^{-L}\).

It remains to expand the beta expectations in
\eqref{eq:app_mixed_after_conditioning}. Let \(\psi\) be a smooth function on
\(I_\tau\). Taylor's theorem at \(u_\alpha\) gives, on \(B_n\in I_\tau\),
\[
\psi(B_n)
=
\sum_{s=0}^{2L-1}
\frac{\psi^{(s)}(u_\alpha)}{s!}
(B_n-u_\alpha)^s
+
O_L(|B_n-u_\alpha|^{2L}).
\]
Therefore,
\begin{align}
&\mathbb E\!\left[
(B_n-u_\alpha)^{\nu_0}
\psi(B_n)
\Ind{B_n\in I_\tau}
\right] \notag\\
&=
\sum_{s=0}^{2L-1}
\frac{\psi^{(s)}(u_\alpha)}{s!}
\mathbb E\!\left[
(B_n-u_\alpha)^{\nu_0+s}
\Ind{B_n\in I_\tau}
\right]
+
O_L\!\left(
\mathbb E[
|B_n-u_\alpha|^{2L+\nu_0}
\Ind{B_n\in I_\tau}]
\right).
\end{align}
Since \(|B_n-u_\alpha|\le\tau\) on \(I_\tau\), and
\(\mathbb E|B_n-u_\alpha|^{2L}\le C_Ln^{-L}\),
\[
\mathbb E[
|B_n-u_\alpha|^{2L+\nu_0}
\Ind{B_n\in I_\tau}]
\le C_L n^{-L}.
\]
Moreover,
\[
\left|
\mathbb E[(B_n-u_\alpha)^a\Ind{B_n\in I_\tau}]
-
\mathbb E[(B_n-u_\alpha)^a]
\right|
\le
Ce^{-cn},
\]
because \(\mathbb P(B_n\notin I_\tau)\le Ce^{-cn}\) and
\(|B_n-u_\alpha|\le1\). Hence
\begin{align}
&\mathbb E\!\left[
(B_n-u_\alpha)^{\nu_0}
\psi(B_n)
\Ind{B_n\in I_\tau}
\right] \notag\\
&=
\sum_{s=0}^{2L-1}
\frac{\psi^{(s)}(u_\alpha)}{s!}
\mathbb E[(B_n-u_\alpha)^{\nu_0+s}]
+
O_L(n^{-L}).
\label{eq:app_smooth_beta_expectation_expansion}
\end{align}

Finally, for every fixed integer \(a\),
\begin{equation}
\label{eq:app_beta_central_moment_expansion}
\mathbb E[(B_n-u_\alpha)^a]
=
\sum_{\ell=0}^{L-1}
\beta_{a,\ell}(\delta_n)n^{-\ell}
+
\rho_{a,L}(n,\delta_n),
\qquad
|\rho_{a,L}(n,\delta_n)|\le C_L n^{-L}.
\end{equation}
Indeed, raw beta moments satisfy
\[
\mathbb E[B_n^m]
=
\frac{(n-k_n+1)^{\overline m}}{(n+1)^{\overline m}},
\qquad
x^{\overline m}:=x(x+1)\cdots(x+m-1).
\]
After substituting \(k_n=\alpha n+\delta_n\), each raw moment is a rational
function of \(n\) and \(\delta_n\), and hence has a finite expansion in powers
of \(n^{-1}\). Central moments follow by expanding
\((B_n-u_\alpha)^a\) into raw moments.

Combining
\eqref{eq:app_mixed_after_conditioning},
\eqref{eq:app_smooth_beta_expectation_expansion}, and
\eqref{eq:app_beta_central_moment_expansion}, we obtain
\[
\mathbb E[\Delta_n^\nu\Ind{B_n\in I_\tau}]
=
\sum_{\ell=0}^{L-1}
\widetilde c_{\nu,\ell}(\delta_n;x,\theta)n^{-\ell}
+
\widetilde\rho_{\nu,L}(n;x,\theta),
\qquad
|\widetilde\rho_{\nu,L}(n;x,\theta)|\le C_L n^{-L}.
\]

The coefficient of \(n^0\) is zero. If \(m_\nu>0\), then in every conditional
term we have
\[
r\le\lfloor m_\nu/2\rfloor,
\]
and hence \(m_\nu-r\ge1\). Since \(k_n\asymp n\) and \(K_n\asymp n\),
\[
\left|
k_n^{-m_\nu}K_n^{\underline r}
\right|
\le
C n^{-(m_\nu-r)}
\le
C n^{-1}.
\]
Thus no \(n^0\) term can arise from the conditional moment expansion. If
\(m_\nu=0\), then \(\nu_0\ge1\), and the remaining factor is a central beta
moment
\[
\mathbb E[(B_n-u_\alpha)^{\nu_0+s}],
\]
whose expansion has zero constant coefficient because \(B_n\to u_\alpha\).
Therefore \(\widetilde c_{\nu,0}=0\). Setting
\[
c_{\nu,\ell}:=\widetilde c_{\nu,\ell},
\qquad
\ell=1,\dots,L-1,
\]
and using \eqref{eq:app_mixed_local_to_threshold_local}, proves
\eqref{eq:app_delta_mixed_moment_expansion}.
\end{proof}

We can now assemble the finite-order expansion for the cross-fitted base
estimator. The local reconstruction and Taylor expansion express
\(H(\hat\eta_n)\) as a finite polynomial in \(\Delta_n\), up to a controlled
remainder. The high-moment lemma controls this remainder, while the mixed-moment
lemma expands the expectation of each polynomial term.

\begin{proposition}
\label{prop:cross_fitted_plugin_expansion}
Suppose Assumptions~\ref{assum:gaussian-tail-model},
\ref{assum:bounded_score}, and \ref{assum:nondegenerate_scale} hold. Fix an
integer \(L\ge1\), and let
\[
\delta_n:=\lceil \alpha n\rceil-\alpha n.
\]
Then there exist coefficient vectors
\[
b_\ell(\delta_n;x,\theta),
\qquad
\ell=1,\dots,L-1,
\]
depending on \(n\) only through \(\delta_n\), such that, for all sufficiently
large \(n\),
\begin{equation}
\label{eq:app_cross_fitted_plugin_expansion}
\mathbb E[G_n^{\mathrm{cf}}(\theta;x)]
=
g_\theta(x)
+
\sum_{\ell=1}^{L-1}
b_\ell(\delta_n;x,\theta)n^{-\ell}
+
\mathcal R_L(n;x,\theta),
\end{equation}
where
\begin{equation}
\label{eq:app_cross_fitted_plugin_remainder}
\|\mathcal R_L(n;x,\theta)\|
\le
C_0A_0^L L!\,
\frac{\sqrt{\log N}}{n^L}.
\end{equation}
The constants \(C_0,A_0>0\) depend only on
\((\alpha,G,\sigma_{\min},M_R)\), uniformly over \(x,\theta,n,L\).
\end{proposition}

\begin{proof}
By the cross-fitted identity in
\Cref{lem:cross_fitted_identity_base_variance},
\[
\mathbb E[G_n^{\mathrm{cf}}(\theta;x)]
=
\mathbb E[H(\hat\eta_n)].
\]
Thus it suffices to expand \(\mathbb E[H(\hat\eta_n)]\).

Using \(\mathcal L_n\), decompose
\[
\mathbb E[H(\hat\eta_n)]-H(\eta)
=
\mathbb E\!\left[
\bigl(H(\hat\eta_n)-H(\eta)\bigr)\Ind{\mathcal L_n}
\right]
+
\mathbb E\!\left[
\bigl(H(\hat\eta_n)-H(\eta)\bigr)\Ind{\mathcal L_n^c}
\right].
\]
On \(\mathcal L_n\), \Cref{lem:local_reconstruction_empirical_tail_vector}
gives
\[
H(\hat\eta_n)=\mathcal G(\Delta_n).
\]
Therefore, by \Cref{lem:local_taylor_reconstructed_plugin_map},
\begin{align}
\bigl(H(\hat\eta_n)-H(\eta)\bigr)\Ind{\mathcal L_n}
&=
\left[
\sum_{1\le|\nu|\le 2L-1}
\widetilde{\mathcal H}_\nu(x,\theta)\Delta_n^\nu
+
\mathcal E_{\mathcal G,L}(\Delta_n)
\right]\Ind{\mathcal L_n}.
\label{eq:app_plugin_expansion_on_local_event}
\end{align}

The Taylor remainder is controlled by
\Cref{lem:local_taylor_reconstructed_plugin_map} and
\Cref{lem:local_tail_error_high_moments}:
\[
\begin{aligned}
\mathbb E\!\left[
\|\mathcal E_{\mathcal G,L}(\Delta_n)\|\Ind{\mathcal L_n}
\right]
&\le
C_0A_0^L L!\,\sqrt{\log N}\,
\mathbb E\!\left[
\|\Delta_n\|^{2L}\Ind{\mathcal L_n}
\right]  \\
&\le
C_0A_0^L L!\frac{\sqrt{\log N}}{n^L}.
\end{aligned}
\]

For the polynomial terms, \Cref{lem:local_tail_error_mixed_moment_expansion}
gives, for every \(1\le|\nu|\le 2L-1\),
\[
\mathbb E[\Delta_n^\nu\Ind{\mathcal L_n}]
=
\sum_{\ell=1}^{L-1}
c_{\nu,\ell}(\delta_n;x,\theta)n^{-\ell}
+
r_{\nu,L}(n;x,\theta),
\qquad
|r_{\nu,L}(n;x,\theta)|\le C_Ln^{-L}.
\]
Substituting this into
\eqref{eq:app_plugin_expansion_on_local_event}, and using the coefficient bound
from \Cref{lem:local_taylor_reconstructed_plugin_map}, gives
\begin{align}
\mathbb E\!\left[
\bigl(H(\hat\eta_n)-H(\eta)\bigr)\Ind{\mathcal L_n}
\right]
&=
\sum_{\ell=1}^{L-1}
\left(
\sum_{1\le|\nu|\le 2L-1}
\widetilde{\mathcal H}_\nu(x,\theta)
c_{\nu,\ell}(\delta_n;x,\theta)
\right)n^{-\ell}
+
\mathcal R_{L,\mathrm{loc}}(n;x,\theta),
\end{align}
where
\[
\|\mathcal R_{L,\mathrm{loc}}(n;x,\theta)\|
\le
C_0A_0^L L!\frac{\sqrt{\log N}}{n^L}.
\]

Define
\begin{equation}
\label{eq:app_plugin_expansion_coefficients_def}
b_\ell(\delta_n;x,\theta)
:=
\sum_{1\le|\nu|\le 2L-1}
\widetilde{\mathcal H}_\nu(x,\theta)
c_{\nu,\ell}(\delta_n;x,\theta),
\qquad
\ell=1,\dots,L-1.
\end{equation}
These coefficients depend on \(n\) only through \(\delta_n\).

It remains to bound the complement \(\mathcal L_n^c\). By
\eqref{eq:app_expansion_local_event},
\[
\mathbb P(\mathcal L_n^c)\le Ce^{-cn}.
\]
Moreover, the same moment-envelope argument used in the proof of
\Cref{lem:cross_fitted_identity_base_variance} gives
\[
\mathbb E\|H(\hat\eta_n)-H(\eta)\|^2\le C\log N.
\]
Indeed, this follows from \(H(\eta')=P\varphi_{\eta'}\), the explicit quadratic
growth of \(\widetilde R_{\eta'}\), the bounded score, the clipped scale
\(\hat\sigma_n\ge\varepsilon_\sigma\), and
\Cref{lem:empirical_tail_moment_envelope}. Therefore, by Cauchy--Schwarz,
\[
\left\|
\mathbb E\!\left[
(H(\hat\eta_n)-H(\eta))\Ind{\mathcal L_n^c}
\right]
\right\|
\le
C\sqrt{\log N}\,e^{-cn}.
\]
Since the right-hand side is exponentially small, it is absorbed into
\[
C_0A_0^LL!\sqrt{\log N}\,n^{-L}
\]
for every fixed \(L\).

Combining the local expansion and the complement bound yields
\[
\mathbb E[H(\hat\eta_n)]
=
H(\eta)
+
\sum_{\ell=1}^{L-1}
b_\ell(\delta_n;x,\theta)n^{-\ell}
+
\mathcal R_L(n;x,\theta),
\]
with
\[
\|\mathcal R_L(n;x,\theta)\|
\le
C_0A_0^L L!\frac{\sqrt{\log N}}{n^L}.
\]
Finally,
\[
H(\eta)=P\varphi_\eta=g_\theta(x),
\]
and
\[
\mathbb E[G_n^{\mathrm{cf}}(\theta;x)]
=
\mathbb E[H(\hat\eta_n)].
\]
This proves \eqref{eq:app_cross_fitted_plugin_expansion} and
\eqref{eq:app_cross_fitted_plugin_remainder}.
\end{proof}

\subsection{Proof of \Cref{thm:fixed_order_debiased_gradient}}
\label{app:proof_fixed_order_debiased_gradient}

\begin{proof}[Proof of \Cref{thm:fixed_order_debiased_gradient}]
We first prove split-budget bounds for \(\widetilde g_{n,k,J}\) defined in \eqref{eq:app_split_budget_fixed_order_estimator}
. Write
\[
w_j:=w_j^{(k,J)}(n),
\qquad
G_j:=G_{n_j}^{\mathrm{cf}}(\theta;x).
\]

\paragraph{Variance.}
Let
\[
\xi_j:=G_j-\mathbb E[G_j].
\]
Then
\[
\widetilde g_{n,k,J}(\theta;x)
-
\mathbb E[\widetilde g_{n,k,J}(\theta;x)]
=
\sum_{j=1}^J w_j\xi_j.
\]
By Cauchy--Schwarz,
\[
\left\|
\sum_{j=1}^J w_j\xi_j
\right\|^2
\le
\left(\sum_{j=1}^Jw_j^2\right)
\left(\sum_{j=1}^J\|\xi_j\|^2\right)
=
\|w^{(k,J)}(n)\|_2^2
\sum_{j=1}^J\|\xi_j\|^2.
\]
Taking expectations and applying
\Cref{lem:cross_fitted_identity_base_variance} with \(n=n_j\), we obtain
\[
\mathbb E\!\left[
\left\|
\widetilde g_{n,k,J}(\theta;x)
-
\mathbb E[\widetilde g_{n,k,J}(\theta;x)]
\right\|^2
\right]
\le
C\log N\,
\|w^{(k,J)}(n)\|_2^2
\sum_{j=1}^J\frac1{n_j}.
\]
By \Cref{lem:moment_cancellation_weight_complexity},
\[
\|w^{(k,J)}(n)\|_2^2
\sum_{j=1}^J\frac1{n_j}
\le
C_{k,J}\frac1n.
\]
Thus
\[
\mathbb E\!\left[
\left\|
\widetilde g_{n,k,J}(\theta;x)
-
\mathbb E[\widetilde g_{n,k,J}(\theta;x)]
\right\|^2
\right]
\le
C_{k,J}\frac{\log N}{n}.
\]

\paragraph{Bias.}
Apply \Cref{prop:cross_fitted_plugin_expansion} with \(L=k\) and
\(n=n_j\). For each \(j=1,\dots,J\),
\[
\mathbb E[G_j]
=
g_\theta(x)
+
\sum_{\ell=1}^{k-1}
b_\ell(\delta_{n_j};x,\theta)n_j^{-\ell}
+
\mathcal R_k(n_j;x,\theta),
\]
where, since \(k\) is fixed,
\[
\|\mathcal R_k(n_j;x,\theta)\|
\le
C_{k,J}\frac{\sqrt{\log N}}{n_j^k}.
\]
By \Cref{lem:compatible_prefixes},
\[
\delta_{n_j}=0,
\qquad
j=1,\dots,J.
\]
Therefore
\[
\mathbb E[G_j]
=
g_\theta(x)
+
\sum_{\ell=1}^{k-1}
b_\ell(0;x,\theta)n_j^{-\ell}
+
\mathcal R_k(n_j;x,\theta).
\]
Using the definition of \(\widetilde g_{n,k,J}\),
\begin{align}
\mathbb E[\widetilde g_{n,k,J}(\theta;x)]
-
g_\theta(x)
&=
\left(\sum_{j=1}^Jw_j-1\right)g_\theta(x) \notag\\
&\quad+
\sum_{\ell=1}^{k-1}
b_\ell(0;x,\theta)
\sum_{j=1}^Jw_jn_j^{-\ell}
+
\sum_{j=1}^Jw_j\mathcal R_k(n_j;x,\theta).
\end{align}
The first term is zero because
\[
\sum_{j=1}^Jw_j=1.
\]
For \(\ell=1,\dots,k-1\), the middle terms are zero because
\[
\sum_{j=1}^Jw_jn_j^{-\ell}=0
\]
by \Cref{lem:moment_cancellation_weight_complexity}. Hence
\[
\mathbb E[\widetilde g_{n,k,J}(\theta;x)]
-
g_\theta(x)
=
\sum_{j=1}^Jw_j\mathcal R_k(n_j;x,\theta).
\]
Thus
\begin{align}
\left\|
\mathbb E[\widetilde g_{n,k,J}(\theta;x)]
-
g_\theta(x)
\right\|
&\le
C_{k,J}\sqrt{\log N}
\sum_{j=1}^J |w_j|n_j^{-k} \notag\\
&\le
C_{k,J}\frac{\sqrt{\log N}}{n^k},
\end{align}
where the last step uses
\Cref{lem:moment_cancellation_weight_complexity}.

Finally set \(n=\lfloor m/2\rfloor\). By
\eqref{eq:app_total_budget_fixed_order_estimator},
\[
\widehat g_{m,k,J}^{\mathrm{fo}}(\theta;x)
=
\widetilde g_{\lfloor m/2\rfloor,k,J}(\theta;x).
\]
Since \(\lfloor m/2\rfloor\asymp m\), the split-budget bounds imply, after
enlarging \(C_{k,J}\),
\[
\left\|
\mathbb E[\widehat g_{m,k,J}^{\mathrm{fo}}(\theta;x)]
-
g_\theta(x)
\right\|
\le
C_{k,J}\frac{\sqrt{\log N}}{m^k},
\]
and
\[
\mathbb E\!\left[
\left\|
\widehat g_{m,k,J}^{\mathrm{fo}}(\theta;x)
-
\mathbb E[\widehat g_{m,k,J}^{\mathrm{fo}}(\theta;x)]
\right\|^2
\right]
\le
C_{k,J}\frac{\log N}{m}.
\]
The enlarged \(C_{k,J}\) depends only on
\((\alpha,G,\sigma_{\min},M_R,k,J)\).
This proves the theorem.
\end{proof}

\subsection{Proofs for finite-rollout post-training}
\label{app:proof_finite_rollout_post_training}

We first state the regularity facts of $\hatTTSobj$ needed for the convergence analysis.

\begin{lemma}
\label{lem:post_training_surrogate_regular}
Under Assumptions~\ref{assum:gaussian-tail-model},
\ref{assum:bounded_score}, and \ref{assum:nondegenerate_scale}, the surrogate objective
\(\hatTTSobj\) is \(L_{\mathrm{opt}}\)-smooth and bounded above:
\[
\hatTTSobj(\theta)\le \hatTTSobj^\star<\infty .
\]
Moreover, defining
\[
u_\theta(x)
:=
g_\theta(x)
-
\beta\nabla_\theta
\mathrm{KL}\bigl(
\curpolicy(\cdot\mid x)\|\refpolicy(\cdot\mid x)
\bigr),
\]
we have
\[
\nabla_\theta\hatTTSobj(\theta)
=
\mathbb E_{x\sim P_X}[u_\theta(x)]
\]
and
\begin{equation}
\label{eq:app_prompt_variance_bound}
\mathbb E_{x\sim P_X}
\left[
\left\|
u_\theta(x)-\nabla_\theta\hatTTSobj(\theta)
\right\|^2
\right]
\le
C(\log N+\beta^2).
\end{equation}
The constant \(C\) depends only on \((\alpha,G,\sigma_{\min},M_R)\). Moreover,
\(L_{\mathrm{opt}}\le C(1+\beta+\sqrt{\log N})\) for such a constant \(C\).
\end{lemma}
\begin{proof}
Write
\[
K_\theta(x)
:=
\mathrm{KL}\bigl(
\curpolicy(\cdot\mid x)\|\refpolicy(\cdot\mid x)
\bigr).
\]

By the proof of \Cref{lem:VN_tail_surrogate}, the Gaussian-tail model implies
\[
r_{\theta,\alpha}(x)
=
\mu_\theta(x)+\sigma_\theta(x)z_\alpha,
\]
\[
\mu_{\theta,\alpha}(x)
=
\mu_\theta(x)+\lambda_\alpha\sigma_\theta(x),
\qquad
\sigma_{\theta,\alpha}(x)
=
\sqrt{\delta_\alpha}\,\sigma_\theta(x).
\]
Since \(\alpha\) is fixed, \(z_\alpha,\lambda_\alpha,\delta_\alpha\) are fixed
constants. Hence Assumption~\ref{assum:bounded_score} implies
\begin{equation}
\label{eq:app_tail_map_first_second_bounds}
\left\|
\nabla_\theta \mu_{\theta,\alpha}(x)
\right\|
+
\left\|
\nabla_\theta \sigma_{\theta,\alpha}(x)
\right\|
+
\left\|
\nabla_\theta^2 \mu_{\theta,\alpha}(x)
\right\|_{\mathrm{op}}
+
\left\|
\nabla_\theta^2 \sigma_{\theta,\alpha}(x)
\right\|_{\mathrm{op}}
\le C,
\end{equation}
uniformly over \(x\) and \(\theta\). The same assumption gives
\begin{equation}
\label{eq:app_K_first_second_bounds}
\left\|
\nabla_\theta K_\theta(x)
\right\|
+
\left\|
\nabla_\theta^2 K_\theta(x)
\right\|_{\mathrm{op}}
\le C,
\end{equation}
uniformly over \(x\) and \(\theta\).

By definition of the Gaussian-tail surrogate,
\[
\hatTTSobj(\theta)
=
\mathbb E_{x\sim P_X}
\left[
\mu_{\theta,\alpha}(x)
+
\widetilde c_N\,\sigma_{\theta,\alpha}(x)
-
\beta K_\theta(x)
\right].
\]
Differentiating under the expectation, justified by
\eqref{eq:app_tail_map_first_second_bounds} and
\eqref{eq:app_K_first_second_bounds}, gives
\begin{align}
\nabla_\theta\hatTTSobj(\theta)
&=
\mathbb E_{x\sim P_X}
\left[
\nabla_\theta \mu_{\theta,\alpha}(x)
+
\widetilde c_N\,\nabla_\theta \sigma_{\theta,\alpha}(x)
-
\beta\nabla_\theta K_\theta(x)
\right].
\label{eq:app_hatJ_gradient_explicit}
\end{align}
By the promptwise gradient identity,
\[
g_\theta(x)
=
\nabla_\theta
\left[
\mu_{\theta,\alpha}(x)
+
\widetilde c_N\,\sigma_{\theta,\alpha}(x)
\right]
=
\nabla_\theta \mu_{\theta,\alpha}(x)
+
\widetilde c_N\,\nabla_\theta \sigma_{\theta,\alpha}(x).
\]
Thus
\[
\nabla_\theta\hatTTSobj(\theta)
=
\mathbb E_{x\sim P_X}
\left[
g_\theta(x)
-
\beta\nabla_\theta K_\theta(x)
\right]
=
\mathbb E_{x\sim P_X}[u_\theta(x)].
\]

We next prove smoothness. Differentiating
\eqref{eq:app_hatJ_gradient_explicit} once more gives
\[
\nabla_\theta^2\hatTTSobj(\theta)
=
\mathbb E_{x\sim P_X}
\left[
\nabla_\theta^2 \mu_{\theta,\alpha}(x)
+
\widetilde c_N\,\nabla_\theta^2 \sigma_{\theta,\alpha}(x)
-
\beta\nabla_\theta^2 K_\theta(x)
\right].
\]
Using \eqref{eq:app_tail_map_first_second_bounds} and
\eqref{eq:app_K_first_second_bounds},
\[
\left\|
\nabla_\theta^2\hatTTSobj(\theta)
\right\|_{\mathrm{op}}
\le
C\left(1+|\widetilde c_N|+\beta\right).
\]
The Gaussian maximum coefficient satisfies
\[
|\widetilde c_N|\le C\sqrt{\log N}.
\]
Therefore
\[
\left\|
\nabla_\theta^2\hatTTSobj(\theta)
\right\|_{\mathrm{op}}
\le
C(1+\beta+\sqrt{\log N}).
\]
Hence \(\hatTTSobj\) is \(L_{\mathrm{opt}}\)-smooth with
\[
L_{\mathrm{opt}}
\le
C(1+\beta+\sqrt{\log N}).
\]

We now prove that \(\hatTTSobj\) is bounded above. Since
\[
K_\theta(x)\ge0,
\]
we have
\[
\hatTTSobj(\theta)
\le
\mathbb E_{x\sim P_X}
\left[
\mu_{\theta,\alpha}(x)
+
\widetilde c_N\,\sigma_{\theta,\alpha}(x)
\right].
\]
By \Cref{cor:uniform_tail_bounds}
\[
|\mu_{\theta,\alpha}(x)|+\sigma_{\theta,\alpha}(x)\le C
\]
uniformly over \(x\) and \(\theta\). Together with
\(|\widetilde c_N|\le C\sqrt{\log N}\), this yields
\[
\hatTTSobj(\theta)
\le
C\sqrt{\log N}.
\]
Thus
\[
\hatTTSobj^\star
:=
\sup_\theta \hatTTSobj(\theta)
<\infty .
\]

It remains to prove the prompt-variance bound. From the identity above,
\[
u_\theta(x)
=
g_\theta(x)
-
\beta\nabla_\theta K_\theta(x)
=
\nabla_\theta \mu_{\theta,\alpha}(x)
+
\widetilde c_N\,\nabla_\theta \sigma_{\theta,\alpha}(x)
-
\beta\nabla_\theta K_\theta(x).
\]
Using \eqref{eq:app_tail_map_first_second_bounds} and
\eqref{eq:app_K_first_second_bounds},
\[
\|u_\theta(x)\|
\le
C(1+|\widetilde c_N|+\beta)
\le
C(1+\sqrt{\log N}+\beta).
\]
Therefore
\[
\mathbb E_{x\sim P_X}
\left[
\|u_\theta(x)\|^2
\right]
\le
C(\log N+\beta^2).
\]
Finally, since
\[
\nabla_\theta\hatTTSobj(\theta)
=
\mathbb E_{x\sim P_X}[u_\theta(x)],
\]
we have
\[
\mathbb E_{x\sim P_X}
\left[
\left\|
u_\theta(x)
-
\nabla_\theta\hatTTSobj(\theta)
\right\|^2
\right]
\le
\mathbb E_{x\sim P_X}
\left[
\|u_\theta(x)\|^2
\right]
\le
C(\log N+\beta^2).
\]
This proves \eqref{eq:app_prompt_variance_bound}.
\end{proof}

\subsubsection{Proof of \Cref{prop:finite_rollout_post_training}}
\begin{proof}[Proof of \Cref{prop:finite_rollout_post_training}]
Let
\[
\nabla_t:=\nabla_\theta\hatTTSobj(\theta_t),
\qquad
\bar G_t:=\mathbb E[\widehat G_t\mid \theta_t].
\]
By \Cref{lem:post_training_surrogate_regular}, \(\hatTTSobj\) is
\(L_{\mathrm{opt}}\)-smooth with
\[
L_{\mathrm{opt}}
\le
C(1+\beta+\sqrt{\log N}).
\]
The step-size condition in the proposition ensures
\[
\gamma L_{\mathrm{opt}}\le c_0
\]
for a sufficiently small numerical constant \(c_0>0\).

We first control the bias of the prompt-batch direction. Conditional on
\(\theta_t\), the prompt-batch mean satisfies
\[
\bar G_t
=
\mathbb E_{x\sim P_X}
\left[
\mathbb E[\widehat g(\theta_t;x)]-
\beta\nabla_\theta
\mathrm{KL}\bigl(
\curpolicy(\cdot\mid x)\|\refpolicy(\cdot\mid x)
\bigr)
\right].
\]
Also, by \Cref{lem:post_training_surrogate_regular},
\[
\nabla_t
=
\mathbb E_{x\sim P_X}
\left[
g_{\theta_t}(x)-
\beta\nabla_\theta
\mathrm{KL}\bigl(
\curpolicy(\cdot\mid x)\|\refpolicy(\cdot\mid x)
\bigr)
\right].
\]
Therefore,
\[
\bar G_t-\nabla_t
=
\mathbb E_{x\sim P_X}
\left[
\mathbb E[\widehat g(\theta_t;x)]-g_{\theta_t}(x)
\right].
\]
If \(\widehat g=\widehat g_m^{\mathrm{dir}}\), then
\Cref{thm:direct_plugin_estimator} gives
\[
\|\bar G_t-\nabla_t\|^2
\le
C\frac{\log N}{m^2}.
\]
If \(\widehat g=\widehat g_{m,k,J}^{\mathrm{fo}}\), then
\Cref{thm:fixed_order_debiased_gradient} gives
\[
\|\bar G_t-\nabla_t\|^2
\le
C\frac{\log N}{m^{2k}},
\]
Hence, in both cases,
\begin{equation}
\label{eq:app_post_training_bias_generic}
\|\bar G_t-\nabla_t\|^2
\le
C\mathsf B_m,
\end{equation}
where \(\mathsf B_m\) is the corresponding bias term in
\Cref{prop:finite_rollout_post_training}.

We next control the stochastic fluctuation of \(\widehat G_t\). For a single
prompt \(x\), define
\[
u_{\theta_t}(x)
:=
g_{\theta_t}(x)
-
\beta\nabla_\theta
\mathrm{KL}\bigl(
\curpolicy(\cdot\mid x)\|\refpolicy(\cdot\mid x)
\bigr).
\]
By \Cref{lem:post_training_surrogate_regular},
\[
\mathbb E_{x\sim P_X}
\left[
\left\|
u_{\theta_t}(x)-\nabla_t
\right\|^2
\right]
\le
C(\log N+\beta^2).
\]
The prompt-level variance contributes \(C(\log N+\beta^2)\). The rollout
variance contributes \(C\log N/m\), for either estimator after fixing \(k,J\).
Since \(m\ge1\), this is absorbed into \(C(\log N+\beta^2)\). Combining these
bounds and using independence across the \(P\) prompts, we obtain
\begin{equation}
\label{eq:app_post_training_variance_generic}
\mathbb E\!\left[
\|\widehat G_t-\bar G_t\|^2
\,\middle|\,
\theta_t
\right]
\le
C\frac{\log N+\beta^2}{P}
+
C\frac{\mathsf B_m}{P}.
\end{equation}
The extra \(\mathsf B_m/P\) term comes from the promptwise bias variation; it
will be absorbed into the final \(\mathsf B_m\) term.

By \(L_{\mathrm{opt}}\)-smoothness and the ascent update,
\[
\hatTTSobj(\theta_{t+1})
\ge
\hatTTSobj(\theta_t)
+
\gamma\langle \nabla_t,\widehat G_t\rangle
-
\frac{L_{\mathrm{opt}}\gamma^2}{2}
\|\widehat G_t\|^2.
\]
Taking conditional expectation given \(\theta_t\), we get
\[
\mathbb E\!\left[
\hatTTSobj(\theta_{t+1})-\hatTTSobj(\theta_t)
\,\middle|\,
\theta_t
\right]
\ge
\gamma\langle \nabla_t,\bar G_t\rangle
-
\frac{L_{\mathrm{opt}}\gamma^2}{2}
\mathbb E\!\left[
\|\widehat G_t\|^2
\,\middle|\,
\theta_t
\right].
\]
Let
\[
e_t:=\bar G_t-\nabla_t.
\]
Then
\[
\langle \nabla_t,\bar G_t\rangle
=
\|\nabla_t\|^2+\langle \nabla_t,e_t\rangle
\ge
\frac34\|\nabla_t\|^2-C\|e_t\|^2.
\]
Moreover,
\[
\mathbb E\!\left[
\|\widehat G_t\|^2
\,\middle|\,
\theta_t
\right]
\le
2\|\bar G_t\|^2
+
2\mathbb E\!\left[
\|\widehat G_t-\bar G_t\|^2
\,\middle|\,
\theta_t
\right],
\]
and
\[
\|\bar G_t\|^2
=
\|\nabla_t+e_t\|^2
\le
2\|\nabla_t\|^2+2\|e_t\|^2.
\]
Using
\eqref{eq:app_post_training_bias_generic} and
\eqref{eq:app_post_training_variance_generic}, we obtain
\[
\mathbb E\!\left[
\|\widehat G_t\|^2
\,\middle|\,
\theta_t
\right]
\le
C\|\nabla_t\|^2
+
C\mathsf B_m
+
C\frac{\log N+\beta^2}{P}.
\]
Substituting the last displays into the smoothness inequality and using
\(\gamma L_{\mathrm{opt}}\le c_0\), with \(c_0\) sufficiently small, gives
\[
\mathbb E[
\hatTTSobj(\theta_{t+1})-\hatTTSobj(\theta_t)]
\ge
c\gamma\,
\mathbb E\|\nabla_t\|^2
-
C\gamma\,\mathsf B_m
-
C\gamma^2\frac{\log N+\beta^2}{P}.
\]
Here we used \(P\ge1\) and \(\gamma\le1\) to absorb the harmless
\(\gamma^2\mathsf B_m/P\) term into \(C\gamma\mathsf B_m\).

Summing over \(t=0,\dots,T-1\) gives
\[
c\gamma
\sum_{t=0}^{T-1}
\mathbb E\|\nabla_t\|^2
\le
\mathbb E[\hatTTSobj(\theta_T)]-\hatTTSobj(\theta_0)
+
CT\gamma\,\mathsf B_m
+
CT\gamma^2\frac{\log N+\beta^2}{P}.
\]
Since
\[
\hatTTSobj(\theta_T)\le \hatTTSobj^\star,
\]
we conclude that
\[
\frac1T
\sum_{t=0}^{T-1}
\mathbb E
\left[
\left\|
\nabla_\theta\hatTTSobj(\theta_t)
\right\|^2
\right]
\le
C
\left[
\frac{\hatTTSobj^\star-\hatTTSobj(\theta_0)}{\gamma T}
+
\mathsf B_m
+
\gamma\frac{\log N+\beta^2}{P}
\right].
\]
This proves \Cref{prop:finite_rollout_post_training}.
\end{proof}

\begin{corollary}
\label{cor:value_transfer_to_best_of_n}
Under the same conditions as in
\Cref{prop:finite_rollout_post_training}, let
\(\TTSobj^\star:=\sup_\theta \TTSobj(\theta)\), and let \(I\) be uniformly
distributed on \(\{0,\dots,T-1\}\), independently of the training randomness.
Suppose, in addition, that \(\hatTTSobj\) satisfies the gradient-domination
condition for some \(\mu_{\mathrm{gd}}>0\):
\[
\hatTTSobj^\star-\hatTTSobj(\theta)
\le
\frac1{2\mu_{\mathrm{gd}}}
\left\|
\nabla_\theta\hatTTSobj(\theta)
\right\|^2
\qquad
\text{for all }\theta .
\]
Then, for either the direct plug-in estimator or the fixed-order estimator,
\begin{equation}
\label{eq:best_of_n_value_gap}
\TTSobj^\star-\mathbb E[\TTSobj(\theta_I)]
\le
C_{\mathrm{gd}}
\left[
\frac{\hatTTSobj^\star-\hatTTSobj(\theta_0)}{\gamma T}
+
\mathsf B_m
+
\gamma\frac{\log N+\beta^2}{P}
\right]
+
C(1-2\alpha)^N,
\end{equation}
where \(\mathsf B_m\) is the corresponding finite-rollout bias term from
\Cref{prop:finite_rollout_post_training}.
The constant \(C_{\mathrm{gd}}\) depends only on
\((\alpha,G,\sigma_{\min},M_R,\mu_{\mathrm{gd}})\), and in the fixed-order case
also on \((k,J)\).
\end{corollary}

\subsubsection{Proof of \Cref{cor:value_transfer_to_best_of_n}}
\label{app:proof_value_transfer_to_best_of_n}
\begin{proof}[Proof of \Cref{cor:value_transfer_to_best_of_n}]
By the Gaussian-tail value approximation in \Cref{lem:VN_tail_surrogate},
\[
\sup_\theta
\left|
\hatTTSobj(\theta)-\TTSobj(\theta)
\right|
\le
C(1-2\alpha)^N.
\]
Therefore, for every \(\theta\),
\begin{align}
\TTSobj^\star-\TTSobj(\theta)
&=
\sup_{\theta'}\TTSobj(\theta')-\TTSobj(\theta) \notag\\
&\le
\sup_{\theta'}
\left[
\hatTTSobj(\theta')+C(1-2\alpha)^N
\right]
-
\left[
\hatTTSobj(\theta)-C(1-2\alpha)^N
\right] \notag\\
&\le
\hatTTSobj^\star-\hatTTSobj(\theta)
+
C(1-2\alpha)^N.
\label{eq:app_value_transfer_to_exact_tts}
\end{align}
Let \(I\) be uniformly distributed on \(\{0,\dots,T-1\}\), independently of the
training randomness. By the gradient-domination condition,
\[
\hatTTSobj^\star-\mathbb E[\hatTTSobj(\theta_I)]
\le
\frac1{2\mu_{\mathrm{gd}}}
\mathbb E
\left[
\left\|
\nabla_\theta\hatTTSobj(\theta_I)
\right\|^2
\right].
\]
Since \(I\) is uniform,
\[
\mathbb E
\left[
\left\|
\nabla_\theta\hatTTSobj(\theta_I)
\right\|^2
\right]
=
\frac1T
\sum_{t=0}^{T-1}
\mathbb E
\left[
\left\|
\nabla_\theta\hatTTSobj(\theta_t)
\right\|^2
\right].
\]
Applying \Cref{prop:finite_rollout_post_training} and then
\eqref{eq:app_value_transfer_to_exact_tts}, we get
\[
\TTSobj^\star-\mathbb E[\TTSobj(\theta_I)]
\le
C_{\mathrm{gd}}
\left[
\frac{\hatTTSobj^\star-\hatTTSobj(\theta_0)}{\gamma T}
+
\mathsf B_m
+
\gamma\frac{\log N+\beta^2}{P}
\right]
+
C(1-2\alpha)^N.
\]
This proves \Cref{cor:value_transfer_to_best_of_n}.
\end{proof}

\section{Synthetic diagnostic for Prefix-TEA}
\label{app:synthetic_bias_variance_frontier}

\paragraph{Purpose.}
This section is a diagnostic for the fixed-order debiasing theorem. It is not a
training benchmark. We isolate the finite-sample bias--variance tradeoff in a
one-prompt Gaussian-tail model where the target gradient is known analytically.
The goal is to check that Prefix-TEA cancels systematic bias, while also giving 
guidance on the budget regimes where prefix-TEA is expected to outperform TEA.

\paragraph{Synthetic model.}
We take
\[
R=Z,\qquad Z\sim \mathcal N(0,1),
\]
with bounded two-dimensional score
\[
S(Z)
=
\bigl(
\Ind{Z\ge 1.0}-\bar\Phi(1.0),
\Ind{Z\ge 1.5}-\bar\Phi(1.5)
\bigr).
\]
We set \(\alpha=0.25\) and \(N=128\). The true gradient is computed from exact
Gaussian tail integrals. We compare prefix-TEA to TEA with the same rollout budget \(m\) and prompt batch size \(P\).

\paragraph{Metrics.}
For a prompt batch of size \(P\), we report
\[
\operatorname{MSE}_{P}(m)
=
\|\operatorname{Bias}(m)\|_2^2
+
\frac{\operatorname{Var}(m)}{P}.
\]
For Prefix-TEA, systematic bias is estimated by Rao--Blackwellizing over the
evaluation split; variance is measured using the actual finite-sample
estimator. We report the main Prefix-TEA setting \(k=2,J=4\), and include
\(J=8\) as a stabilizing ablation.

\paragraph{Results.}
Figure~\ref{fig:prefix_tea_synthetic_scaling} gives the bias and variance scaling with \(m\) for \(P=1\). The bias plot
 confirms second-order cancellation for the main \(k=2,J=4\)
Prefix-TEA diagnostic: its bias decreases faster than TEA and follows the
expected \(m^{-2}\) trend in the large-\(m\) regime. The variance plot shows the
cost: Prefix-TEA has a larger variance constant. 

Figure~\ref{fig:prefix_tea_synthetic_frontier} then shows how this
bias--variance tradeoff changes after prompt-batch averaging. For small \(P\),
the larger Prefix-TEA variance dominates, so TEA has lower MSE. As \(P\)
increases, the variance term is suppressed and the smaller Prefix-TEA bias
becomes decisive, producing the blue region where Prefix-TEA outperforms TEA.
Table~\ref{tab:prefix_tea_summary} quantifies the same effect: Prefix-TEA has worse
one-prompt MSE, but crosses over once \(P\) reaches the thousands range, depending on \(m\) and \(J\).

\begin{figure}[t]
    \centering
    \includegraphics[width=\textwidth]{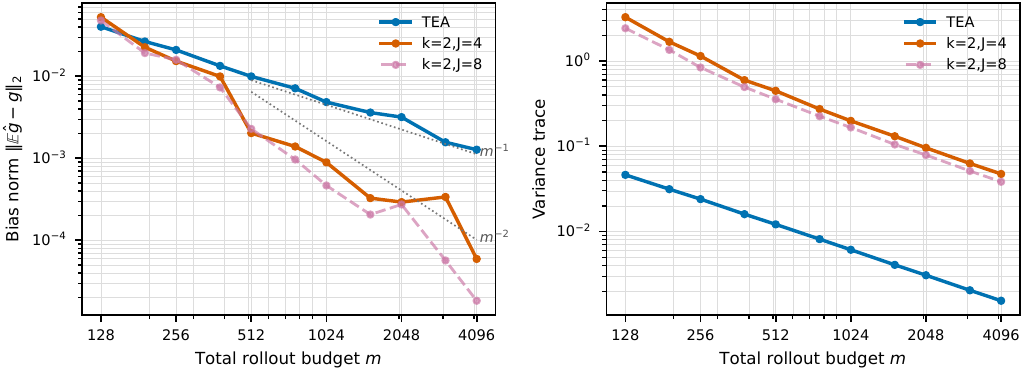}
    \caption{Bias and variance scaling in the one-prompt Gaussian-tail
    diagnostic. Prefix-TEA \(k=2,J=4\) reduces systematic bias faster than TEA,
    but with a larger variance constant.}
    \label{fig:prefix_tea_synthetic_scaling}
\end{figure}

\begin{figure}[t]
    \centering
    \includegraphics[width=\textwidth]{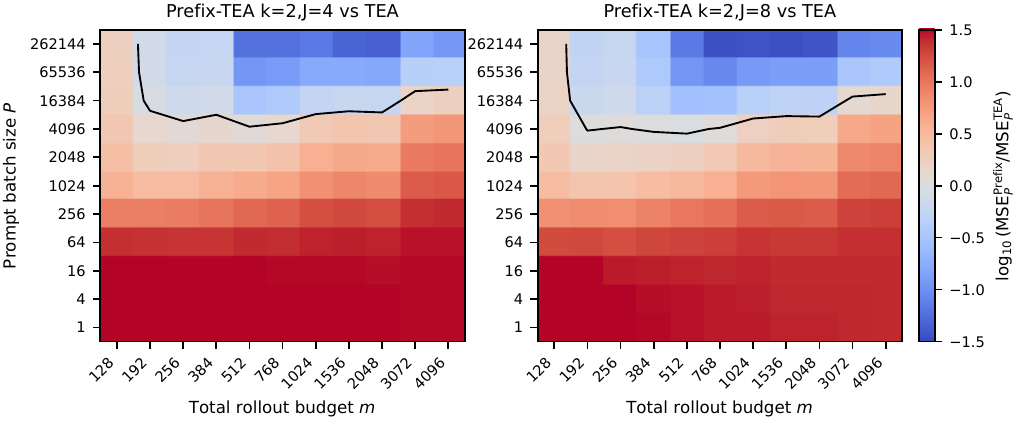}
    \caption{Prompt-batch frontier for Prefix-TEA \(k=2,J=4\), with \(J=8\) as
    a stabilizing ablation. The
    heatmap shows
    \(\log_{10}(\operatorname{MSE}_{P}^{\mathrm{Prefix}}(m)/
    \operatorname{MSE}_{P}^{\mathrm{TEA}}(m))\); values below zero indicate
    lower MSE for Prefix-TEA.}
    \label{fig:prefix_tea_synthetic_frontier}
\end{figure}

\begin{table}[t]
\centering
\small
\caption{Synthetic Prefix-TEA bias--variance summary.}
\label{tab:prefix_tea_summary}
\resizebox{\textwidth}{!}{%
\begin{tabular}{llrrrrr}
\toprule
Estimator & $m$ & Bias norm & Variance & MSE $P=1$ & MSE $P=2048$ & MSE $P=65536$ \\
\midrule
TEA & 256 & 0.021 & 0.024 & 0.024 & 4.520e-04 & 4.406e-04 \\
TEA & 512 & 9.972e-03 & 0.012 & 0.012 & 1.054e-04 & 9.962e-05 \\
TEA & 1024 & 4.848e-03 & 6.114e-03 & 6.138e-03 & 2.649e-05 & 2.360e-05 \\
TEA & 2048 & 3.180e-03 & 3.073e-03 & 3.083e-03 & 1.161e-05 & 1.016e-05 \\
TEA & 4096 & 1.274e-03 & 1.544e-03 & 1.546e-03 & 2.377e-06 & 1.646e-06 \\
Prefix-TEA $k=2,J=4$ & 256 & 0.015 & 1.144 & 1.144 & 7.939e-04 & 2.528e-04 \\
Prefix-TEA $k=2,J=4$ & 512 & 2.030e-03 & 0.447 & 0.447 & 2.225e-04 & 1.095e-05 \\
Prefix-TEA $k=2,J=4$ & 1024 & 8.919e-04 & 0.199 & 0.199 & 9.778e-05 & 3.826e-06 \\
Prefix-TEA $k=2,J=4$ & 2048 & 2.924e-04 & 0.096 & 0.096 & 4.690e-05 & 1.549e-06 \\
Prefix-TEA $k=2,J=4$ & 4096 & 5.924e-05 & 0.047 & 0.047 & 2.314e-05 & 7.264e-07 \\
Prefix-TEA $k=2,J=8$ & 256 & 0.016 & 0.839 & 0.840 & 6.633e-04 & 2.662e-04 \\
Prefix-TEA $k=2,J=8$ & 512 & 2.286e-03 & 0.356 & 0.356 & 1.793e-04 & 1.066e-05 \\
Prefix-TEA $k=2,J=8$ & 1024 & 4.661e-04 & 0.166 & 0.166 & 8.135e-05 & 2.753e-06 \\
Prefix-TEA $k=2,J=8$ & 2048 & 2.736e-04 & 0.079 & 0.079 & 3.864e-05 & 1.280e-06 \\
Prefix-TEA $k=2,J=8$ & 4096 & 1.831e-05 & 0.038 & 0.038 & 1.872e-05 & 5.852e-07 \\
\bottomrule
\end{tabular}%
}
\end{table}

\paragraph{Takeaway.}
The diagnostic explains why Prefix-TEA can underperform TEA in small-budget real
training runs despite its higher-order bias cancellation: the variance and
constant factors matter. Its advantage appears in the large-rollout and
large-prompt-batch regime where the reduced bias floor dominates.

\end{document}